\documentclass[journal]{IEEEtran}

\ifCLASSINFOpdf
\else
\fi

\usepackage{setspace}
\usepackage{cite}
\usepackage{graphicx}
\graphicspath{{../Images/}}

\usepackage[cmex10]{amsmath}

\usepackage{float, epstopdf}
\usepackage{verbatim}
\usepackage{relsize,bbm}
\usepackage{amsfonts}
\usepackage{amsthm}
\usepackage{amssymb}
\usepackage{latexsym}
\usepackage{url}
\usepackage{epstopdf}
\usepackage{enumerate}
\usepackage{mathrsfs}
\usepackage[bottom]{footmisc}
\usepackage{tabulary, hyperref}
\usepackage{color}
\usepackage{balance}
\usepackage{amsmath}
\usepackage{cases}

\usepackage{tikz}
\usepackage{tkz-euclide}

\allowdisplaybreaks[1]
\newcommand{\Emmetttt}[5]{
\draw[#4] (#1)
-- ++(#2,1.0166278585759576*#3)
-- ++(#2,0.9101359472832664*#3)
-- ++(#2,0.6454157992852653*#3)
-- ++(#2,0.36495322004951475*#3)
-- ++(#2,0.12910424140387444*#3)
-- ++(#2,0.006346233897352481*#3)
-- ++(#2,-0.2596894483188231*#3)
-- ++(#2,3.3967800151603167*#3)
-- ++(#2,1.0113539747910816*#3)
-- ++(#2,-1.1293758543669614*#3)
-- ++(#2,-0.22296320661761376*#3)
-- ++(#2,-1.0279679630219853*#3)
-- ++(#2,-0.03724590940212457*#3)
-- ++(#2,-1.285369512035384*#3)
-- ++(#2,-1.1212366096102708*#3)
-- ++(#2,0.032798519812387204*#3)
-- ++(#2,-0.75713209869772*#3)
-- ++(#2,0.3151557041479214*#3)
-- ++(#2,0.8178027580853732*#3)
-- ++(#2,0.4950973269668775*#3)
-- ++(#2,0.3750719569405528*#3)
-- ++(#2,1.5029426879305205*#3)
-- ++(#2,1.15028413873218*#3)
-- ++(#2,-0.24665074659623995*#3)
-- ++(#2,-0.44831834032329954*#3)
-- ++(#2,-0.07057531899975539*#3)
node[above,right] {#5};
}

\newcommand{\Emmettt}[5]{
\draw[#4] (#1)
-- ++(#2,0.3449056335215385*#3)
-- ++(#2,0.7398026578211958*#3)
-- ++(#2,-0.916112736574975*#3)
-- ++(#2,-0.7280702579919234*#3)
-- ++(#2,0.10103138828017721*#3)
-- ++(#2,1.1746583619342272*#3)
-- ++(#2,0.6694380024991095*#3)
-- ++(#2,-0.2776380171251103*#3)
-- ++(#2,2.2773804583006108*#3)
-- ++(#2,0.49127760266658743*#3)
-- ++(#2,0.8973856503151352*#3)
-- ++(#2,-0.26974348637899015*#3)
-- ++(#2,-1.0370845106132867*#3)
-- ++(#2,-0.2936946221312899*#3)
-- ++(#2,0.268588357964518*#3)
-- ++(#2,1.1575965166011128*#3)
-- ++(#2,-0.04862997563755554*#3)
-- ++(#2,0.11893047650736724*#3)
-- ++(#2,-0.8347936050512282*#3)
-- ++(#2,1.0181422147925627*#3)
-- ++(#2,-1.1303170849521254*#3)
-- ++(#2,1.3921694982475885*#3)
-- ++(#2,0.6158990397062992*#3)
-- ++(#2,0.3805001771520416*#3)
-- ++(#2,0.2717243317116084*#3)
-- ++(#2,0.009795826930386532*#3)
-- ++(#2,-0.021313694589365637*#3)
-- ++(#2,0.05423950858368839*#3)
-- ++(#2,-2.130904562297885*#3)
-- ++(#2,-0.073875566791504*#3)
-- ++(#2,-0.39453571211520394*#3)
-- ++(#2,-1.623808801718912*#3)
-- ++(#2,-1.5435719946660016*#3)
-- ++(#2,1.6434123017787372*#3)
-- ++(#2,-0.06576150541180928*#3)
-- ++(#2,1.167114685888141*#3)
-- ++(#2,0.4316496542829812*#3)
-- ++(#2,0.010882410653040518*#3)
-- ++(#2,0.5860701048181445*#3)
-- ++(#2,-0.23725784251300142*#3)
-- ++(#2,0.18581011349681975*#3)
-- ++(#2,0.8737916208367913*#3)
-- ++(#2,1.1619470256643603*#3)
-- ++(#2,1.0901007632010566*#3)
-- ++(#2,-0.8783468370367103*#3)
-- ++(#2,1.5823736465587712*#3)
-- ++(#2,0.39016973834282337*#3)
-- ++(#2,-0.8455399385738679*#3)
-- ++(#2,0.3151663363874288*#3)
-- ++(#2,0.745527625884931*#3)
-- ++(#2,1.377982588836424*#3)
-- ++(#2,-2.2122176221005865*#3)
-- ++(#2,1.922834204203318*#3)
-- ++(#2,0.10858490291840239*#3)
-- ++(#2,1.9737049091140266*#3)
-- ++(#2,0.8011895463710068*#3)
-- ++(#2,1.262269558251995*#3)
-- ++(#2,0.9793138192213999*#3)
-- ++(#2,-0.34177792440928556*#3)
-- ++(#2,-0.1460170632355545*#3)
-- ++(#2,0.4770507260593286*#3)
-- ++(#2,0.16772988919182674*#3)
-- ++(#2,-1.684669328117657*#3)
-- ++(#2,-0.6986662750228173*#3)
-- ++(#2,-1.0450741374123143*#3)
-- ++(#2,-0.9136342743675437*#3)
-- ++(#2,-1.0649506834631823*#3)
-- ++(#2,0.9438339554068944*#3)
-- ++(#2,1.187483110852872*#3)
-- ++(#2,0.5231501218840362*#3)
-- ++(#2,0.3400091446729878*#3)
-- ++(#2,-0.2834847322938794*#3)
-- ++(#2,-0.6784270577745334*#3)
-- ++(#2,-0.9004606727541768*#3)
-- ++(#2,0.08901043126196052*#3)
-- ++(#2,-1.184904192310481*#3)
-- ++(#2,-0.006575634406559893*#3)
-- ++(#2,-1.1053722400238393*#3)
-- ++(#2,-0.8272332040172831*#3)
-- ++(#2,-0.7009312993704205*#3)
-- ++(#2,0.6361168518638886*#3)
-- ++(#2,0.24088813492638678*#3)
-- ++(#2,-0.7423868762792266*#3)
-- ++(#2,0.2225808061790885*#3)
-- ++(#2,2.04866364045871*#3)
-- ++(#2,-0.8259463364231185*#3)
-- ++(#2,-0.32497326633189905*#3)
-- ++(#2,-1.3562669509039462*#3)
-- ++(#2,0.33779657331262464*#3)
-- ++(#2,0.08586459620642885*#3)
-- ++(#2,0.4172502512566366*#3)
-- ++(#2,-1.1887739233143806*#3)
-- ++(#2,-0.5796897515158352*#3)
-- ++(#2,1.0367729903181098*#3)
-- ++(#2,1.1320289820356295*#3)
-- ++(#2,-0.4943611722488632*#3)
-- ++(#2,1.9318255009792915*#3)
-- ++(#2,-0.27539247705938547*#3)
-- ++(#2,-1.0787128649021125*#3)
-- ++(#2,-0.03267405493507573*#3)
-- ++(#2,-1.2467540346428205*#3)
-- ++(#2,0.4347730232422345*#3)
-- ++(#2,2.2777158336092076*#3)
-- ++(#2,-0.567629556459478*#3)
-- ++(#2,-0.13406002136593428*#3)
-- ++(#2,1.7576744749992506*#3)
-- ++(#2,0.037623379131566666*#3)
-- ++(#2,-1.5389692291253714*#3)
-- ++(#2,2.546426040816052*#3)
-- ++(#2,-0.8912021667357914*#3)
-- ++(#2,0.1520807018035584*#3)
-- ++(#2,1.0227771466798226*#3)
-- ++(#2,-1.5445106141341833*#3)
-- ++(#2,1.102962012628561*#3)
-- ++(#2,-0.8375873256453004*#3)
-- ++(#2,0.7184236707445469*#3)
-- ++(#2,1.4123575576875833*#3)
-- ++(#2,0.04755199264945984*#3)
-- ++(#2,0.13748766765710987*#3)
-- ++(#2,1.4161751272959366*#3)
-- ++(#2,1.5412133204985856*#3)
-- ++(#2,-0.2254206921346609*#3)
-- ++(#2,0.9577904204657445*#3)
-- ++(#2,0.46442517614144946*#3)
-- ++(#2,-1.1204306749889217*#3)
-- ++(#2,0.31477101144435116*#3)
-- ++(#2,-1.531193849136144*#3)
-- ++(#2,0.926048027832965*#3)
-- ++(#2,0.35002152861367425*#3)
-- ++(#2,0.7002507470789914*#3)
-- ++(#2,0.6609211616329218*#3)
-- ++(#2,-0.5919767846048613*#3)
-- ++(#2,-0.8869335901564764*#3)
-- ++(#2,0.705309532788491*#3)
-- ++(#2,-1.0499175036898691*#3)
-- ++(#2,1.9491754928575893*#3)
-- ++(#2,-0.18318509578693445*#3)
-- ++(#2,-1.145852774546802*#3)
-- ++(#2,-1.6826282723966812*#3)
-- ++(#2,-1.3502092840050326*#3)
-- ++(#2,0.005703716186119519*#3)
-- ++(#2,0.04059851445668165*#3)
-- ++(#2,-0.9748425215111768*#3)
-- ++(#2,0.33474668181691897*#3)
-- ++(#2,-0.9228852844913327*#3)
-- ++(#2,0.4076658442957865*#3)
-- ++(#2,0.740213092510974*#3)
-- ++(#2,-0.19195474283111094*#3)
-- ++(#2,0.5491795618352212*#3)
-- ++(#2,-0.5114584336539122*#3)
-- ++(#2,-0.8454186861750425*#3)
-- ++(#2,-0.7531936198938354*#3)
-- ++(#2,0.6586294890349088*#3)
-- ++(#2,-1.7594004780852077*#3)
-- ++(#2,-0.5707741224451651*#3)
-- ++(#2,-0.4300754678320325*#3)
-- ++(#2,-0.3457614367398927*#3)
-- ++(#2,-0.7275445701614217*#3)
-- ++(#2,0.2571820638018825*#3)
-- ++(#2,0.24769035268517192*#3)
-- ++(#2,-0.8947897719699436*#3)
-- ++(#2,0.31391007611583555*#3)
-- ++(#2,-0.6128283615906525*#3)
-- ++(#2,-1.221768994128327*#3)
-- ++(#2,-0.3497653162822244*#3)
-- ++(#2,1.556681558035022*#3)
-- ++(#2,-0.11671270943855697*#3)
-- ++(#2,-1.314187122094663*#3)
-- ++(#2,-1.976366741521246*#3)
-- ++(#2,-0.36702540760352625*#3)
-- ++(#2,-0.4131457980436011*#3)
-- ++(#2,0.510006637598822*#3)
-- ++(#2,0.02927764500001159*#3)
-- ++(#2,-0.4610549563385248*#3)
-- ++(#2,1.6095830544013083*#3)
-- ++(#2,0.48907659087010596*#3)
-- ++(#2,-1.5597466209003452*#3)
-- ++(#2,0.3854431885093869*#3)
-- ++(#2,0.5389653953688969*#3)
-- ++(#2,-0.8315382788373957*#3)
-- ++(#2,-1.3700809014451447*#3)
-- ++(#2,0.07290210694954767*#3)
-- ++(#2,-0.2585568122905575*#3)
-- ++(#2,-0.642924483746135*#3)
-- ++(#2,1.0273898343530599*#3)
-- ++(#2,-0.6215967042879802*#3)
-- ++(#2,0.9100805597007047*#3)
-- ++(#2,-2.125442244063577*#3)
-- ++(#2,0.17289306826298406*#3)
-- ++(#2,-0.7630975685083816*#3)
-- ++(#2,0.6129414863102024*#3)
-- ++(#2,0.8253855620392107*#3)
-- ++(#2,-1.4812673774086238*#3)
-- ++(#2,0.9260286426168104*#3)
-- ++(#2,-0.09032557504581051*#3)
-- ++(#2,-0.023645401913980596*#3)
-- ++(#2,-1.437141206710381*#3)
-- ++(#2,-0.28321488029506064*#3)
-- ++(#2,1.0941394477732072*#3)
-- ++(#2,-1.1341586158405863*#3)
-- ++(#2,-1.2582230166362223*#3)
-- ++(#2,-0.4132008723951448*#3)
-- ++(#2,0.34481696209662077*#3)
-- ++(#2,0.6409946092719749*#3)
-- ++(#2,0.3049194745412806*#3)
-- ++(#2,-0.046873689568691264*#3)
-- ++(#2,-1.1224679819111714*#3)
-- ++(#2,-0.4050987285392296*#3)
-- ++(#2,-0.2040123547807524*#3)
-- ++(#2,-0.474209657765763*#3)
-- ++(#2,-1.025882541530904*#3)
-- ++(#2,-0.33208617842506993*#3)
-- ++(#2,0.600813705938282*#3)
-- ++(#2,-1.3937364489343025*#3)
-- ++(#2,0.8135293112411721*#3)
-- ++(#2,-0.4899912644436204*#3)
-- ++(#2,1.803942007112332*#3)
-- ++(#2,0.06324514602453456*#3)
-- ++(#2,2.679389376343012*#3)
-- ++(#2,0.28908992803607164*#3)
-- ++(#2,-0.8850873832371031*#3)
-- ++(#2,1.0319102068827024*#3)
-- ++(#2,0.9386824564780153*#3)
-- ++(#2,-1.5656985907069647*#3)
-- ++(#2,1.170796946073813*#3)
-- ++(#2,0.06050476997019145*#3)
-- ++(#2,-0.3696787906960233*#3)
-- ++(#2,0.12828087763986762*#3)
-- ++(#2,0.2758944989753252*#3)
-- ++(#2,1.194485025279975*#3)
-- ++(#2,0.21944708557042836*#3)
-- ++(#2,-0.3765147123401648*#3)
-- ++(#2,0.7521646822264286*#3)
-- ++(#2,-1.5479558747640523*#3)
-- ++(#2,-0.13645847876741055*#3)
-- ++(#2,-0.5493813720684166*#3)
-- ++(#2,-0.8158801346891295*#3)
-- ++(#2,-0.942673107272523*#3)
-- ++(#2,0.6081053047266466*#3)
-- ++(#2,1.8236115484242947*#3)
-- ++(#2,-2.7245103694518376*#3)
-- ++(#2,0.5701632578826887*#3)
-- ++(#2,0.37711103611830993*#3)
-- ++(#2,-1.022896327221233*#3)
-- ++(#2,0.5554005168101616*#3)
-- ++(#2,-0.21686947685701038*#3)
-- ++(#2,0.5099317832328711*#3)
-- ++(#2,0.43213585248918385*#3)
-- ++(#2,0.23941727717660768*#3)
-- ++(#2,-1.2168455607867628*#3)
-- ++(#2,1.5579353423180724*#3)
-- ++(#2,0.6031449937946921*#3)
-- ++(#2,-1.9958465769892608*#3)
-- ++(#2,-0.23934693751772654*#3)
-- ++(#2,-1.8046239194805167*#3)
-- ++(#2,0.3938288794130791*#3)
-- ++(#2,2.632219962280346*#3)
-- ++(#2,0.5538992581781934*#3)
-- ++(#2,0.6517949111153402*#3)
-- ++(#2,0.5106141975773855*#3)
-- ++(#2,-0.07183230071063822*#3)
-- ++(#2,0.5298332093368895*#3)
-- ++(#2,-0.249614483251254*#3)
-- ++(#2,1.773208051841553*#3)
-- ++(#2,0.48634398847192767*#3)
-- ++(#2,0.5122910177194563*#3)
-- ++(#2,-0.8434671183278081*#3)
-- ++(#2,1.7273200542928102*#3)
-- ++(#2,1.1658378454374783*#3)
-- ++(#2,0.7991731571166161*#3)
-- ++(#2,0.5075723681581606*#3)
-- ++(#2,0.427674991541707*#3)
-- ++(#2,-1.1858838946315733*#3)
-- ++(#2,-0.10254883495729698*#3)
-- ++(#2,1.3615735795618833*#3)
-- ++(#2,-1.6875021417759097*#3)
-- ++(#2,-1.4047427852583572*#3)
-- ++(#2,1.1777160842823047*#3)
-- ++(#2,-0.0018225323108189912*#3)
-- ++(#2,-1.0096467439868413*#3)
-- ++(#2,-1.7278436639712542*#3)
-- ++(#2,-0.02824254878257679*#3)
-- ++(#2,-1.9441458457686662*#3)
-- ++(#2,-0.7274791634779245*#3)
-- ++(#2,0.7278394004442781*#3)
-- ++(#2,-1.3168622060458204*#3)
-- ++(#2,0.44403102480325796*#3)
-- ++(#2,0.7406788691364518*#3)
-- ++(#2,2.3692356889515245*#3)
-- ++(#2,-0.26244116310781485*#3)
-- ++(#2,-0.0863491546673918*#3)
-- ++(#2,1.4205655513935345*#3)
-- ++(#2,-1.050553091444569*#3)
-- ++(#2,0.1266347754275285*#3)
-- ++(#2,0.10181452239497983*#3)
-- ++(#2,0.3575067641429259*#3)
-- ++(#2,-0.5534378707962939*#3)
-- ++(#2,2.089873811003557*#3)
-- ++(#2,-0.5490534776747515*#3)
-- ++(#2,0.24014010784437786*#3)
-- ++(#2,-1.6095897672553232*#3)
-- ++(#2,0.4115975986156625*#3)
-- ++(#2,-0.2613641571108994*#3)
-- ++(#2,1.009121796182624*#3)
-- ++(#2,0.5793279757980518*#3)
-- ++(#2,1.0308836692078056*#3)
-- ++(#2,0.6596359730013667*#3)
-- ++(#2,-0.9798407608863228*#3)
-- ++(#2,0.30708819920349384*#3)
-- ++(#2,-1.3351522109769716*#3)
-- ++(#2,-0.27636084795029225*#3)
-- ++(#2,2.9393078217612096*#3)
-- ++(#2,-0.0277222169378764*#3)
-- ++(#2,1.4850476072374155*#3)
-- ++(#2,0.759542159066531*#3)
-- ++(#2,-0.09497826523305174*#3)
-- ++(#2,1.8947779175616262*#3)
-- ++(#2,1.0744388944795802*#3)
-- ++(#2,-0.33321210288284725*#3)
-- ++(#2,-0.5315791599702846*#3)
-- ++(#2,-0.34428215487273683*#3)
-- ++(#2,-0.6084454404237982*#3)
-- ++(#2,1.9653001172704558*#3)
-- ++(#2,1.694896189258117*#3)
-- ++(#2,-0.2620020354347105*#3)
-- ++(#2,-0.00431395424581888*#3)
-- ++(#2,0.17427844695651548*#3)
-- ++(#2,1.0430879826095383*#3)
-- ++(#2,0.22978528723346106*#3)
-- ++(#2,0.4549725028476907*#3)
-- ++(#2,0.49102945163884093*#3)
-- ++(#2,0.1325005005078553*#3)
-- ++(#2,-0.8373567672232912*#3)
-- ++(#2,-0.48639398095439484*#3)
-- ++(#2,0.28317320485063485*#3)
-- ++(#2,1.8743725298114877*#3)
-- ++(#2,1.7970161061570897*#3)
-- ++(#2,0.2529916597713062*#3)
-- ++(#2,0.11382560766332978*#3)
-- ++(#2,1.0242941919174993*#3)
-- ++(#2,-1.1053739447279007*#3)
-- ++(#2,0.4122138635878645*#3)
-- ++(#2,1.0969812475733527*#3)
-- ++(#2,0.15904735165619904*#3)
-- ++(#2,-1.3516355410045806*#3)
-- ++(#2,-0.13994403771765174*#3)
-- ++(#2,-0.8424554522037775*#3)
-- ++(#2,-0.35130393874122434*#3)
-- ++(#2,-0.7969745293035027*#3)
-- ++(#2,-0.9529824056292071*#3)
-- ++(#2,-1.0246284030658213*#3)
-- ++(#2,-1.0371245116924246*#3)
-- ++(#2,-1.7365045632712193*#3)
-- ++(#2,-0.2055404248218686*#3)
-- ++(#2,-2.878627842357432*#3)
-- ++(#2,0.5233494009032058*#3)
-- ++(#2,-1.112628650687609*#3)
-- ++(#2,1.552862132513126*#3)
-- ++(#2,-1.2694012269597688*#3)
-- ++(#2,0.6106911411854051*#3)
-- ++(#2,3.2345680141749797*#3)
-- ++(#2,0.6569026416221404*#3)
-- ++(#2,1.3179799032652373*#3)
-- ++(#2,-0.272571247007929*#3)
-- ++(#2,-1.1486650692088376*#3)
-- ++(#2,-0.467610410211383*#3)
-- ++(#2,0.707759092193448*#3)
-- ++(#2,1.1532779915613272*#3)
-- ++(#2,-1.6054478618302725*#3)
-- ++(#2,-1.0279785072704222*#3)
-- ++(#2,-0.7248849144290405*#3)
-- ++(#2,-0.48855255904245043*#3)
-- ++(#2,-0.05846231734137145*#3)
-- ++(#2,-0.05897031622130418*#3)
-- ++(#2,0.8698985732144219*#3)
-- ++(#2,0.4052294758282132*#3)
-- ++(#2,-0.30247809829760514*#3)
-- ++(#2,-2.207618106693046*#3)
-- ++(#2,-2.1809848839416524*#3)
-- ++(#2,-0.8646109717653689*#3)
-- ++(#2,-1.5127590393816963*#3)
-- ++(#2,-1.0051612439135995*#3)
-- ++(#2,-2.0188616364243215*#3)
-- ++(#2,1.8160486260828381*#3)
-- ++(#2,-0.5250889950886415*#3)
-- ++(#2,-2.1655273290447963*#3)
-- ++(#2,0.4210049426665618*#3)
-- ++(#2,-0.9394400101950525*#3)
-- ++(#2,-0.41957334046354194*#3)
-- ++(#2,0.21291147353226617*#3)
-- ++(#2,0.12310007745156282*#3)
-- ++(#2,-0.4456796702277672*#3)
-- ++(#2,-2.7605886934682786*#3)
-- ++(#2,-0.9388209081093465*#3)
-- ++(#2,0.86602316010081*#3)
-- ++(#2,-0.9142932106859767*#3)
-- ++(#2,-0.05794675086449222*#3)
-- ++(#2,-1.659483389473278*#3)
-- ++(#2,0.07712996464326002*#3)
-- ++(#2,-0.6582730412947633*#3)
-- ++(#2,-0.2997237899321786*#3)
-- ++(#2,-0.4704674796636908*#3)
-- ++(#2,0.9767455151283504*#3)
-- ++(#2,-0.8475415962317578*#3)
-- ++(#2,-0.4022537558522391*#3)
-- ++(#2,0.9857341698827583*#3)
-- ++(#2,-0.3156492503687313*#3)
-- ++(#2,1.3040625788457327*#3)
-- ++(#2,-1.269467245199967*#3)
-- ++(#2,0.5614629625910763*#3)
-- ++(#2,0.1797340427852219*#3)
-- ++(#2,-0.10709970304632442*#3)
-- ++(#2,-0.5169198104682865*#3)
-- ++(#2,-0.36109872595047354*#3)
-- ++(#2,-1.7249448840458803*#3)
-- ++(#2,2.3643767184199294*#3)
-- ++(#2,-1.4805095991878143*#3)
-- ++(#2,-1.0089672747265748*#3)
-- ++(#2,1.120450167042116*#3)
-- ++(#2,2.2219389979180946*#3)
-- ++(#2,-1.3570548430461047*#3)
-- ++(#2,0.6510788704882343*#3)
-- ++(#2,-0.01225114631446155*#3)
-- ++(#2,1.4161432152217253*#3)
-- ++(#2,-0.4913583623539011*#3)
-- ++(#2,0.020928441933064267*#3)
-- ++(#2,1.3783406379608978*#3)
-- ++(#2,-0.14655056893693966*#3)
-- ++(#2,-0.21988722037189318*#3)
-- ++(#2,1.3444953509863549*#3)
-- ++(#2,-0.10787648337401755*#3)
-- ++(#2,1.272079637362561*#3)
-- ++(#2,1.3729184516515216*#3)
-- ++(#2,0.8017054172273055*#3)
-- ++(#2,-0.04719137902275315*#3)
-- ++(#2,0.44532370430747437*#3)
-- ++(#2,-1.3502820805894624*#3)
-- ++(#2,-0.4124166670148647*#3)
-- ++(#2,1.347489307164453*#3)
-- ++(#2,0.802453022370353*#3)
-- ++(#2,-0.14764100822679221*#3)
-- ++(#2,-0.21944215833640088*#3)
-- ++(#2,-0.8455920242120544*#3)
-- ++(#2,-0.5700840683655094*#3)
-- ++(#2,1.1497474980468574*#3)
-- ++(#2,-0.5577588239215046*#3)
-- ++(#2,-0.5760740666104283*#3)
-- ++(#2,0.8021339254525824*#3)
-- ++(#2,0.06030024154208552*#3)
-- ++(#2,1.1811786256544192*#3)
-- ++(#2,0.7737380578709562*#3)
-- ++(#2,0.48383268948156555*#3)
-- ++(#2,-1.582372123438796*#3)
-- ++(#2,-0.6024699885169178*#3)
-- ++(#2,1.3430573046421372*#3)
-- ++(#2,-0.4302536731926585*#3)
-- ++(#2,-1.2987729318803507*#3)
-- ++(#2,-2.829279161786895*#3)
-- ++(#2,-1.525335705534098*#3)
-- ++(#2,0.18311318163462778*#3)
-- ++(#2,1.7144927983387406*#3)
-- ++(#2,0.4162921839081978*#3)
-- ++(#2,-0.3681140940353205*#3)
-- ++(#2,0.5721762533178347*#3)
-- ++(#2,-0.08708593629397639*#3)
-- ++(#2,0.9530485214893566*#3)
-- ++(#2,0.03276409175246046*#3)
-- ++(#2,-0.25640431066482816*#3)
-- ++(#2,0.7038476770437718*#3)
-- ++(#2,-0.25501671063855347*#3)
-- ++(#2,-2.055292506536319*#3)
-- ++(#2,1.2150361910275222*#3)
-- ++(#2,-0.5032321142517004*#3)
-- ++(#2,0.5313866493338116*#3)
-- ++(#2,-0.20732529366541266*#3)
-- ++(#2,0.3641845579435721*#3)
-- ++(#2,-0.9851073205901165*#3)
-- ++(#2,-0.4701905107885363*#3)
-- ++(#2,-1.5541153507787153*#3)
-- ++(#2,0.2508565497686116*#3)
-- ++(#2,-0.06937103181563054*#3)
-- ++(#2,-0.13493596255832516*#3)
-- ++(#2,-1.1412216536145465*#3)
-- ++(#2,-1.55941280511056*#3)
-- ++(#2,-0.38783210923843237*#3)
-- ++(#2,-1.3426203883863315*#3)
-- ++(#2,0.5480823406048734*#3)
-- ++(#2,-0.32315122159120063*#3)
-- ++(#2,0.2526868734744513*#3)
-- ++(#2,1.1463199955195684*#3)
-- ++(#2,0.3858753404676878*#3)
-- ++(#2,1.108223015656131*#3)
-- ++(#2,-0.47745153508958277*#3)
-- ++(#2,-0.09095022521703548*#3)
-- ++(#2,1.380438166873263*#3)
-- ++(#2,-0.7337487861908415*#3)
-- ++(#2,0.7648024897343026*#3)
-- ++(#2,0.29880206028752804*#3)
-- ++(#2,0.4269283479344177*#3)
-- ++(#2,0.20420454685090889*#3)
-- ++(#2,-1.1629540796544506*#3)
-- ++(#2,0.6118684525156688*#3)
-- ++(#2,-1.205217291600732*#3)
-- ++(#2,0.7706227915609949*#3)
-- ++(#2,-0.4537587909102382*#3)
-- ++(#2,-0.0025317735991971582*#3)
-- ++(#2,-0.015614890316355478*#3)
-- ++(#2,0.6351335054638568*#3)
-- ++(#2,0.5103881667624113*#3)
-- ++(#2,-1.734192249968826*#3)
-- ++(#2,1.0297382198408604*#3)
-- ++(#2,-1.0908538327768877*#3)
-- ++(#2,-1.4561741961929955*#3)
-- ++(#2,0.3186019321971785*#3)
-- ++(#2,0.9499729328081068*#3)
-- ++(#2,0.7682514316798172*#3)
-- ++(#2,0.4378955267438731*#3)
-- ++(#2,-0.7823572121670094*#3)
-- ++(#2,2.2935746436791424*#3)
-- ++(#2,-0.4127782649731224*#3)
-- ++(#2,-0.8852192710945466*#3)
-- ++(#2,0.8657244668158369*#3)
-- ++(#2,-0.1447104164813761*#3)
-- ++(#2,0.6721343275606886*#3)
-- ++(#2,0.8485559071563283*#3)
-- ++(#2,0.8348940198848604*#3)
-- ++(#2,1.1280950886583132*#3)
-- ++(#2,0.4468261630477609*#3)
-- ++(#2,0.07801997113099354*#3)
-- ++(#2,-1.7106942386484958*#3)
-- ++(#2,0.8012332834817133*#3)
-- ++(#2,-0.4520617292772619*#3)
-- ++(#2,-0.5496466003516142*#3)
-- ++(#2,0.8222516668533282*#3)
-- ++(#2,0.6939318353032435*#3)
-- ++(#2,0.4078417664727718*#3)
-- ++(#2,-1.132304071497531*#3)
-- ++(#2,-0.024668274789701225*#3)
-- ++(#2,-0.9884710744033756*#3)
-- ++(#2,0.6109273478177628*#3)
-- ++(#2,0.03146724757275916*#3)
-- ++(#2,-0.3605060002917952*#3)
-- ++(#2,1.2721417003981206*#3)
-- ++(#2,-1.3810019589049523*#3)
-- ++(#2,1.4055759105476668*#3)
-- ++(#2,0.4911868954744994*#3)
-- ++(#2,1.3203473583102103*#3)
-- ++(#2,0.20359431184661614*#3)
-- ++(#2,-0.16171059912627456*#3)
-- ++(#2,0.42738026077288627*#3)
-- ++(#2,-0.1106997788953771*#3)
-- ++(#2,0.5552959699796182*#3)
-- ++(#2,2.091179528763972*#3)
-- ++(#2,-1.497299557975327*#3)
-- ++(#2,1.1099014238327902*#3)
-- ++(#2,0.4668509883806726*#3)
-- ++(#2,0.6790250090989614*#3)
-- ++(#2,2.24752493728416*#3)
-- ++(#2,-0.28038936042836965*#3)
-- ++(#2,-0.3941922071576235*#3)
-- ++(#2,-0.3895727393176059*#3)
-- ++(#2,-0.030089954241603843*#3)
-- ++(#2,-0.2342036584207695*#3)
-- ++(#2,-0.7324072759831188*#3)
-- ++(#2,0.25012539049575383*#3)
-- ++(#2,-0.7667217197002677*#3)
-- ++(#2,0.15906442034866516*#3)
-- ++(#2,-2.4435896802555397*#3)
-- ++(#2,1.2375183830088525*#3)
-- ++(#2,0.8127325570718369*#3)
-- ++(#2,0.8917266132769102*#3)
-- ++(#2,0.7683324552898991*#3)
-- ++(#2,-0.346273896736529*#3)
-- ++(#2,-2.0817598620924715*#3)
-- ++(#2,0.9388434008920936*#3)
-- ++(#2,0.9676292605632831*#3)
-- ++(#2,1.0356158733859608*#3)
-- ++(#2,1.4625679926189832*#3)
-- ++(#2,-0.42128208968556985*#3)
-- ++(#2,-0.616906496800754*#3)
-- ++(#2,-1.0121277872376566*#3)
-- ++(#2,2.842684992368958*#3)
-- ++(#2,0.12223889124839474*#3)
-- ++(#2,1.8650964505858834*#3)
-- ++(#2,1.3183766166570023*#3)
-- ++(#2,-1.09830261927775*#3)
-- ++(#2,0.7179464153538399*#3)
-- ++(#2,-0.2948122369252921*#3)
-- ++(#2,1.21893484978612*#3)
-- ++(#2,-0.06396898079928516*#3)
-- ++(#2,-0.4217470381546676*#3)
-- ++(#2,-0.22421039704484286*#3)
-- ++(#2,-0.11591791252818019*#3)
-- ++(#2,-0.8903389357513859*#3)
-- ++(#2,0.547666018863654*#3)
-- ++(#2,-0.16802431828617634*#3)
-- ++(#2,-1.649456773211496*#3)
-- ++(#2,-1.6933672629129908*#3)
-- ++(#2,1.7519644738645357*#3)
-- ++(#2,0.48624109047714725*#3)
-- ++(#2,-0.060239157191840316*#3)
-- ++(#2,-0.8468580597704425*#3)
-- ++(#2,1.9128319185776927*#3)
-- ++(#2,-0.5671436282057437*#3)
-- ++(#2,-0.12345854150995939*#3)
-- ++(#2,-1.0386815791811266*#3)
-- ++(#2,-0.15633454672997607*#3)
-- ++(#2,-0.561654495892638*#3)
-- ++(#2,-0.20182428283245205*#3)
-- ++(#2,-0.4747429628772831*#3)
-- ++(#2,0.29896285920565097*#3)
-- ++(#2,0.14719658889809722*#3)
-- ++(#2,0.09095087883608137*#3)
-- ++(#2,-0.997435106269009*#3)
-- ++(#2,-1.1632808117757505*#3)
-- ++(#2,-0.06690540969128347*#3)
-- ++(#2,-3.2715046911078285*#3)
-- ++(#2,-0.065536873922069*#3)
-- ++(#2,0.5058876271346499*#3)
-- ++(#2,-0.9693184907559268*#3)
-- ++(#2,1.1933540931175175*#3)
-- ++(#2,0.7209431345718655*#3)
-- ++(#2,-0.8483396966255825*#3)
-- ++(#2,-1.6555003115279228*#3)
-- ++(#2,-1.1825740523280104*#3)
-- ++(#2,1.5675378212503488*#3)
-- ++(#2,-0.5438703612834661*#3)
-- ++(#2,0.11156495003929492*#3)
-- ++(#2,2.443562627680889*#3)
-- ++(#2,-0.9128392761299846*#3)
-- ++(#2,1.496467177008398*#3)
-- ++(#2,1.086370867259672*#3)
-- ++(#2,0.5048557341159385*#3)
-- ++(#2,0.636446412263818*#3)
-- ++(#2,-0.03788050631072836*#3)
-- ++(#2,-1.1024465555308158*#3)
-- ++(#2,0.4060032741161479*#3)
-- ++(#2,-0.2296340252701408*#3)
-- ++(#2,0.5579587801240604*#3)
-- ++(#2,0.963133374075903*#3)
-- ++(#2,0.5149978425204955*#3)
-- ++(#2,1.7140080743846897*#3)
-- ++(#2,-0.24253577149494301*#3)
-- ++(#2,2.4989239030140262*#3)
-- ++(#2,-1.3681674202584757*#3)
-- ++(#2,0.2601171848709939*#3)
-- ++(#2,-0.5896975574049652*#3)
-- ++(#2,0.2083693753206864*#3)
-- ++(#2,0.8014744335307351*#3)
-- ++(#2,0.0025283938716241646*#3)
-- ++(#2,-1.4476456911633797*#3)
-- ++(#2,-0.67809243732817*#3)
-- ++(#2,0.2430996761353203*#3)
-- ++(#2,0.24734559625593217*#3)
-- ++(#2,0.36941304155809307*#3)
-- ++(#2,-0.9357710542794879*#3)
-- ++(#2,1.0091292990509892*#3)
-- ++(#2,-0.8445189431219261*#3)
-- ++(#2,0.4745062355104892*#3)
-- ++(#2,0.8417365865411732*#3)
-- ++(#2,-0.7363251496278173*#3)
-- ++(#2,-0.7885702422840759*#3)
-- ++(#2,-0.31509639098535236*#3)
-- ++(#2,0.32665792885561074*#3)
-- ++(#2,0.9932379637647618*#3)
-- ++(#2,0.8678844866579092*#3)
-- ++(#2,1.4274859343869666*#3)
-- ++(#2,0.1321957075201289*#3)
-- ++(#2,0.5417547136561384*#3)
-- ++(#2,-0.18155894058729277*#3)
-- ++(#2,-1.9881840180875656*#3)
-- ++(#2,1.1179147487537384*#3)
-- ++(#2,-0.4787703665834525*#3)
-- ++(#2,-0.6295909041409146*#3)
-- ++(#2,-1.0505973591143551*#3)
-- ++(#2,-2.6258923223099893*#3)
-- ++(#2,-0.8469190022592032*#3)
-- ++(#2,0.7596176765275361*#3)
-- ++(#2,0.7116404192970169*#3)
-- ++(#2,0.23378981628091436*#3)
-- ++(#2,0.11894404164837669*#3)
-- ++(#2,-0.4474287103344911*#3)
-- ++(#2,1.3368650404415556*#3)
-- ++(#2,3.3522080917166157*#3)
-- ++(#2,0.07412759701195935*#3)
-- ++(#2,0.05215685002904587*#3)
-- ++(#2,-0.43920022882913856*#3)
-- ++(#2,1.6139830102231494*#3)
-- ++(#2,-0.257785873243315*#3)
-- ++(#2,-0.3051389922052778*#3)
-- ++(#2,-0.9458991996394868*#3)
-- ++(#2,-0.22430483768747553*#3)
-- ++(#2,-0.6104449694828917*#3)
-- ++(#2,0.9362117410502374*#3)
-- ++(#2,0.38541946711956615*#3)
-- ++(#2,-1.1623273181783302*#3)
-- ++(#2,0.4027932484559111*#3)
-- ++(#2,-0.5214217838961176*#3)
node[above,right] {#5};
}

\newcommand{\Emmettttt}[5]{
\draw[#4] (#1)
-- ++(#2,-1.1190442307742965*#3)
-- ++(#2,0.1008961128897793*#3)
-- ++(#2,0.708327893864604*#3)
-- ++(#2,0.07348533552056154*#3)
-- ++(#2,0.2422721265233403*#3)
-- ++(#2,0.29375123386573837*#3)
-- ++(#2,-0.3505695422396157*#3)
-- ++(#2,-0.45029271594790704*#3)
-- ++(#2,-0.26980731860332*#3)
-- ++(#2,-0.3423052480576667*#3)
-- ++(#2,0.0944272748541233*#3)
-- ++(#2,-0.34297586132213964*#3)
-- ++(#2,-0.4481579986293799*#3)
-- ++(#2,-1.0559718502489734*#3)
-- ++(#2,1.3256681212340058*#3)
-- ++(#2,-0.13953485499863738*#3)
-- ++(#2,-0.9261315555616028*#3)
-- ++(#2,-0.9154732273040213*#3)
-- ++(#2,-0.22743680811258551*#3)
-- ++(#2,-0.7045624857588833*#3)
-- ++(#2,0.46293994084580314*#3)
-- ++(#2,0.08170289253521872*#3)
-- ++(#2,-0.7514551375788733*#3)
-- ++(#2,-0.3136186105302047*#3)
-- ++(#2,-1.326314960393172*#3)
-- ++(#2,0.8729016763841135*#3)
node[above,right] {#5};
}

\newcommand{\BrownianSimpleIllustration}[5]{
\draw[#4] (#1)
-- ++(#2,1.2921480768805496*#3)
-- ++(#2,-0.7409188366062869*#3)
-- ++(#2,-1.5574970973500888*#3)
-- ++(#2,0.38456057874725047*#3)
-- ++(#2,0.55553638893613*#3)
-- ++(#2,1.0639405632860748*#3)
-- ++(#2,-0.5863446875752336*#3)
-- ++(#2,0.48157796369323175*#3)
-- ++(#2,-0.10467128360035843*#3)
-- ++(#2,-1.1525204398328597*#3)
-- ++(#2,-0.13802563772796142*#3)
-- ++(#2,0.11281600635666696*#3)
-- ++(#2,1.6648916767348063*#3)
-- ++(#2,-1.4710482650024825*#3)
-- ++(#2,0.7634480069903034*#3)
-- ++(#2,-0.8056469648996745*#3)
-- ++(#2,0.7366882392948466*#3)
-- ++(#2,-0.6303999894612854*#3)
-- ++(#2,0.9538825114244409*#3)
-- ++(#2,-0.9967345593096482*#3)
-- ++(#2,0.26671877008959005*#3)
-- ++(#2,0.5524883064710002*#3)
-- ++(#2,0.2432031852486706*#3)
-- ++(#2,-0.24766795742540518*#3)
-- ++(#2,-0.6269664048781385*#3)
-- ++(#2,-0.25673594977809977*#3)
-- ++(#2,0.034876473308214063*#3)
-- ++(#2,0.6194345317103085*#3)
-- ++(#2,-0.9469554645599761*#3)
-- ++(#2,0.7444060268165257*#3)
-- ++(#2,0.8163408066858248*#3)
-- ++(#2,-1.1386985497166615*#3)
-- ++(#2,-2.440685746738396*#3)
-- ++(#2,-1.3617894587025359*#3)
-- ++(#2,-0.6198110753054682*#3)
-- ++(#2,0.013847221316757672*#3)
-- ++(#2,-1.9601967221110166*#3)
-- ++(#2,-0.6270102737852619*#3)
-- ++(#2,-1.2282013040734432*#3)
-- ++(#2,-1.0561466798807626*#3)
-- ++(#2,-0.6642074650134334*#3)
-- ++(#2,-0.6116565043293463*#3)
-- ++(#2,0.0453315279954815*#3)
-- ++(#2,-0.3755610107065398*#3)
-- ++(#2,-0.8751946074005986*#3)
-- ++(#2,1.2751579045493453*#3)
-- ++(#2,0.6692230011832355*#3)
-- ++(#2,-0.051739299775123315*#3)
-- ++(#2,-0.8055810650921591*#3)
-- ++(#2,0.6493050810266684*#3)
-- ++(#2,1.7101072049035093*#3)
-- ++(#2,-0.20053815948308484*#3)
-- ++(#2,-0.02976677374530471*#3)
-- ++(#2,0.4798308837503369*#3)
-- ++(#2,-0.012111834473321915*#3)
-- ++(#2,0.20870988070767152*#3)
-- ++(#2,-0.6711023806147819*#3)
-- ++(#2,1.7684120279283595*#3)
-- ++(#2,0.004134200769607979*#3)
-- ++(#2,-1.039057265992383*#3)
-- ++(#2,0.7365397857822048*#3)
-- ++(#2,-1.3962461783786813*#3)
-- ++(#2,0.1754420153951474*#3)
-- ++(#2,0.7853964552657648*#3)
-- ++(#2,-1.8127160903208355*#3)
-- ++(#2,0.15856922347362232*#3)
-- ++(#2,0.13004371805330364*#3)
-- ++(#2,0.6750281242901801*#3)
-- ++(#2,0.9879408635019694*#3)
-- ++(#2,-0.24060841941794656*#3)
-- ++(#2,0.057580785528114395*#3)
-- ++(#2,-0.3552849515325661*#3)
-- ++(#2,-0.4512968178729068*#3)
-- ++(#2,0.3647536973783716*#3)
-- ++(#2,1.5980463306893267*#3)
-- ++(#2,-0.7587157056274115*#3)
-- ++(#2,-0.9638015815241426*#3)
-- ++(#2,0.14524693094763155*#3)
-- ++(#2,0.40758794749542626*#3)
-- ++(#2,0.6843189357269097*#3)
-- ++(#2,-2.6730933337873055*#3)
-- ++(#2,-0.2814327009718355*#3)
-- ++(#2,0.897794790125806*#3)
-- ++(#2,0.5992299672306803*#3)
-- ++(#2,0.6075572215432794*#3)
-- ++(#2,0.6887654595232517*#3)
-- ++(#2,-0.12420687790446032*#3)
-- ++(#2,0.43029628850585894*#3)
-- ++(#2,2.1890997133111627*#3)
-- ++(#2,-0.029109625348721562*#3)
-- ++(#2,-0.388944868350233*#3)
-- ++(#2,0.18618075622615538*#3)
-- ++(#2,0.33426564648173074*#3)
-- ++(#2,0.042023848523107823*#3)
-- ++(#2,1.9692777622549174*#3)
-- ++(#2,1.601923903889591*#3)
-- ++(#2,0.9516204175406604*#3)
-- ++(#2,1.531459452832764*#3)
-- ++(#2,0.3537191769069364*#3)
-- ++(#2,0.4421975243695229*#3)
-- ++(#2,-0.08464228804626032*#3)
-- ++(#2,-0.6443731318215028*#3)
-- ++(#2,1.6881475446637046*#3)
-- ++(#2,-1.0503781305882873*#3)
-- ++(#2,-1.7032523587334318*#3)
-- ++(#2,1.7036393301930153*#3)
-- ++(#2,0.6405451601134399*#3)
-- ++(#2,2.3354926268093554*#3)
-- ++(#2,0.48365652448008223*#3)
-- ++(#2,0.7674162457269545*#3)
-- ++(#2,-1.4784658828396546*#3)
-- ++(#2,-0.8389833908106654*#3)
-- ++(#2,1.3334150045652207*#3)
-- ++(#2,-0.04331069008562256*#3)
-- ++(#2,-1.2325287073956703*#3)
-- ++(#2,-0.635006938121839*#3)
-- ++(#2,0.909368291975466*#3)
-- ++(#2,1.7237319865962026*#3)
-- ++(#2,1.7749428709798376*#3)
-- ++(#2,0.15660935602722642*#3)
-- ++(#2,0.9624243870477828*#3)
-- ++(#2,0.5266104334174665*#3)
-- ++(#2,0.5598926097719866*#3)
-- ++(#2,1.354996845484812*#3)
-- ++(#2,1.3628362801235467*#3)
-- ++(#2,-1.5422067503230126*#3)
-- ++(#2,-1.078121879757999*#3)
-- ++(#2,0.8495539137041521*#3)
-- ++(#2,0.6024141165302185*#3)
-- ++(#2,-2.1665442778603907*#3)
-- ++(#2,-0.6623989789104672*#3)
-- ++(#2,-1.0326492307446655*#3)
-- ++(#2,-1.499706773087017*#3)
-- ++(#2,0.0380056529399739*#3)
-- ++(#2,1.251266195360183*#3)
-- ++(#2,1.3606233155386414*#3)
-- ++(#2,0.4482410327017523*#3)
-- ++(#2,-0.48856102247347355*#3)
-- ++(#2,-1.8637174449491627*#3)
-- ++(#2,-0.6581222623610149*#3)
-- ++(#2,0.09047723012293755*#3)
-- ++(#2,-1.0178921600398132*#3)
-- ++(#2,-0.2492591252279247*#3)
-- ++(#2,-0.3803526782328749*#3)
-- ++(#2,-1.2961754716043916*#3)
-- ++(#2,-0.6074945512856298*#3)
-- ++(#2,-0.7123366730199877*#3)
-- ++(#2,2.5283723513951113*#3)
-- ++(#2,1.5920202268383474*#3)
-- ++(#2,-0.032796735837843344*#3)
-- ++(#2,1.8173987662174456*#3)
-- ++(#2,-1.6862019732550408*#3)
-- ++(#2,-0.5370017294735443*#3)
-- ++(#2,-1.4021585330218784*#3)
-- ++(#2,0.6275289934061935*#3)
-- ++(#2,-0.2599782719330673*#3)
-- ++(#2,0.17590620470858917*#3)
-- ++(#2,0.9202916368230449*#3)
-- ++(#2,0.5609249904302661*#3)
-- ++(#2,0.5168097186498015*#3)
-- ++(#2,0.7091402732571613*#3)
-- ++(#2,1.2242966687925936*#3)
-- ++(#2,1.5299593927569892*#3)
-- ++(#2,-0.7604510042826704*#3)
-- ++(#2,-0.6539949775551825*#3)
-- ++(#2,-1.1316807863280451*#3)
-- ++(#2,-0.08344141483869712*#3)
-- ++(#2,0.18251120412104108*#3)
-- ++(#2,0.6257569315506815*#3)
-- ++(#2,-1.7595749905365234*#3)
-- ++(#2,-0.37763268967457225*#3)
-- ++(#2,-0.14905781290298964*#3)
-- ++(#2,0.2448459122317343*#3)
-- ++(#2,-0.4352583796771151*#3)
-- ++(#2,1.0390893525998959*#3)
-- ++(#2,-0.17030496269686754*#3)
-- ++(#2,-0.5484343675642898*#3)
-- ++(#2,0.25381541688344456*#3)
-- ++(#2,1.5532460589286823*#3)
-- ++(#2,-1.6845231095064084*#3)
-- ++(#2,-0.5205361077216881*#3)
-- ++(#2,-0.5459227858117603*#3)
-- ++(#2,-0.3979132001299468*#3)
-- ++(#2,-0.03300580409075579*#3)
-- ++(#2,-0.7955576548213725*#3)
-- ++(#2,-0.09031967676800472*#3)
-- ++(#2,-1.157900216698111*#3)
-- ++(#2,-0.7603174548057932*#3)
-- ++(#2,-0.8658068285297116*#3)
-- ++(#2,-0.5719632897195303*#3)
-- ++(#2,1.945368045148999*#3)
-- ++(#2,0.04099620489877427*#3)
-- ++(#2,1.184041752104795*#3)
-- ++(#2,0.04670773593515473*#3)
-- ++(#2,-1.1108338568784069*#3)
-- ++(#2,0.6203333410455972*#3)
-- ++(#2,-1.3544258952503365*#3)
-- ++(#2,1.7373149072753553*#3)
-- ++(#2,-0.588195632046876*#3)
-- ++(#2,-1.117716870425951*#3)
-- ++(#2,0.4942780898191606*#3)
-- ++(#2,-0.3411202450247731*#3)
-- ++(#2,0.59436231937916*#3)
-- ++(#2,-1.1483353613437437*#3)
-- ++(#2,-0.8491152987314229*#3)
-- ++(#2,-0.2864458112888278*#3)
-- ++(#2,-2.1391560856117273*#3)
-- ++(#2,-1.2293549635235048*#3)
-- ++(#2,-1.6951481656444716*#3)
-- ++(#2,-0.6895303189187996*#3)
-- ++(#2,-0.9427804570391286*#3)
-- ++(#2,-0.049934944809216425*#3)
-- ++(#2,0.6697395449679985*#3)
-- ++(#2,-0.0984095671029582*#3)
-- ++(#2,1.2135071081615303*#3)
-- ++(#2,-1.4648950046641553*#3)
-- ++(#2,-1.0936422325413993*#3)
-- ++(#2,-1.1770248239075616*#3)
-- ++(#2,-0.481965232178152*#3)
-- ++(#2,0.3663603313098888*#3)
-- ++(#2,-0.00857385921286949*#3)
-- ++(#2,-0.10052582945817298*#3)
-- ++(#2,-0.07739194193262784*#3)
-- ++(#2,-0.7873150730384058*#3)
-- ++(#2,-1.3838182367000937*#3)
-- ++(#2,-0.9871149527878466*#3)
-- ++(#2,-0.9448931928826433*#3)
-- ++(#2,0.003814515417855425*#3)
-- ++(#2,-0.5281248594789549*#3)
-- ++(#2,0.9087744599515173*#3)
-- ++(#2,-2.254316211544287*#3)
-- ++(#2,0.40810155235196305*#3)
-- ++(#2,0.861584212671331*#3)
-- ++(#2,-0.7158141713881123*#3)
-- ++(#2,-0.5300214433349337*#3)
-- ++(#2,0.3019724738813121*#3)
-- ++(#2,2.122954374806905*#3)
-- ++(#2,0.9059398369629291*#3)
-- ++(#2,0.05600911113743887*#3)
-- ++(#2,-0.08697607601560561*#3)
-- ++(#2,0.07985515967356817*#3)
-- ++(#2,-0.7886434534801885*#3)
-- ++(#2,1.3982907570647753*#3)
-- ++(#2,1.0492441419851564*#3)
-- ++(#2,1.5544246241723314*#3)
-- ++(#2,-1.179019097411362*#3)
-- ++(#2,-1.6512383190098097*#3)
-- ++(#2,0.0017312515281033693*#3)
-- ++(#2,-0.4349179582288256*#3)
-- ++(#2,0.07733177239661572*#3)
-- ++(#2,0.5820996649980769*#3)
-- ++(#2,-1.7604117520384297*#3)
-- ++(#2,-0.9612523637954495*#3)
-- ++(#2,-0.3321554449170514*#3)
-- ++(#2,-0.2778462005797107*#3)
-- ++(#2,0.46060557782011274*#3)
-- ++(#2,1.1280770771910467*#3)
-- ++(#2,0.03927537403814899*#3)
-- ++(#2,-0.8489243311535583*#3)
-- ++(#2,-0.14378395057619447*#3)
-- ++(#2,-0.9918069094331767*#3)
-- ++(#2,-1.0702583232467038*#3)
-- ++(#2,-0.1092291911191107*#3)
-- ++(#2,-2.0373159497509095*#3)
-- ++(#2,-2.236080470873569*#3)
-- ++(#2,0.26475646063303504*#3)
-- ++(#2,1.3547516898648508*#3)
-- ++(#2,2.094792994974014*#3)
-- ++(#2,-2.1145256210996917*#3)
-- ++(#2,0.6739600147355745*#3)
-- ++(#2,0.3338096896359676*#3)
-- ++(#2,-0.8198956546713205*#3)
-- ++(#2,-0.5343120615045995*#3)
-- ++(#2,1.4888626823788342*#3)
-- ++(#2,1.2241145403353682*#3)
-- ++(#2,-0.672282106036639*#3)
-- ++(#2,-0.7909735016107352*#3)
-- ++(#2,-1.4394904969495412*#3)
-- ++(#2,0.3973118915919247*#3)
-- ++(#2,-0.35910057114440513*#3)
-- ++(#2,0.5334399627517437*#3)
-- ++(#2,-0.7243289209636741*#3)
-- ++(#2,0.43501725622634396*#3)
-- ++(#2,1.9561788101215263*#3)
-- ++(#2,-0.22977641186897338*#3)
-- ++(#2,0.08875805545144215*#3)
-- ++(#2,-1.546328526463335*#3)
-- ++(#2,-0.7471736163104647*#3)
-- ++(#2,-0.4661542631687017*#3)
-- ++(#2,0.6904773100779976*#3)
-- ++(#2,-0.6411783242663791*#3)
-- ++(#2,-0.8336412500873087*#3)
-- ++(#2,0.22458205930708122*#3)
-- ++(#2,-0.0848027287524434*#3)
-- ++(#2,0.13547287412432626*#3)
-- ++(#2,0.5507408687429599*#3)
-- ++(#2,0.3107728310356012*#3)
-- ++(#2,-0.8957942650438538*#3)
-- ++(#2,0.7345486819374222*#3)
-- ++(#2,-2.2808947864715066*#3)
-- ++(#2,0.44461242399432704*#3)
-- ++(#2,-0.40125142198370617*#3)
-- ++(#2,-0.20821725238592445*#3)
-- ++(#2,-0.6644326999025151*#3)
-- ++(#2,-1.4377068580369343*#3)
-- ++(#2,-0.3200595658328333*#3)
-- ++(#2,2.005010061523878*#3)
-- ++(#2,0.17573811593110109*#3)
-- ++(#2,-1.6908893130503027*#3)
-- ++(#2,0.09073858987672721*#3)
-- ++(#2,-0.050898073197585714*#3)
-- ++(#2,0.2095075337433358*#3)
-- ++(#2,-0.27820520167728746*#3)
-- ++(#2,-1.8607110500254909*#3)
-- ++(#2,2.2684526764845026*#3)
-- ++(#2,1.1897835347519004*#3)
-- ++(#2,0.6656771821035842*#3)
-- ++(#2,-0.48603031833093*#3)
-- ++(#2,0.6499522833106306*#3)
-- ++(#2,0.5599314609138513*#3)
-- ++(#2,0.5073030005220932*#3)
-- ++(#2,0.48402191313749926*#3)
-- ++(#2,-1.1816789744349645*#3)
-- ++(#2,1.6550467084667357*#3)
-- ++(#2,0.6426029268849429*#3)
-- ++(#2,-1.3785498158440328*#3)
-- ++(#2,-1.7579925870645567*#3)
-- ++(#2,0.11487263446728711*#3)
-- ++(#2,0.1678210657865062*#3)
-- ++(#2,0.03825958612648934*#3)
-- ++(#2,0.5604877106906575*#3)
-- ++(#2,1.395156792515856*#3)
-- ++(#2,0.9458146016677847*#3)
-- ++(#2,0.6482902450251781*#3)
-- ++(#2,-0.270933835377306*#3)
-- ++(#2,0.3378246396048687*#3)
-- ++(#2,0.9325002017985538*#3)
-- ++(#2,-0.04408593739904238*#3)
-- ++(#2,-1.0893445133366648*#3)
-- ++(#2,0.032074577888801276*#3)
-- ++(#2,0.6221576755384597*#3)
-- ++(#2,-0.1812697177850395*#3)
-- ++(#2,-0.1104815640572322*#3)
-- ++(#2,-0.7966759522252197*#3)
-- ++(#2,0.3232477000994598*#3)
-- ++(#2,0.7510432217143492*#3)
-- ++(#2,0.7337105606726634*#3)
-- ++(#2,-0.5967563518196091*#3)
-- ++(#2,-0.8535363323020239*#3)
-- ++(#2,0.4965834667967697*#3)
-- ++(#2,0.721311328834951*#3)
-- ++(#2,-1.9588499807722526*#3)
-- ++(#2,1.1729605174207118*#3)
-- ++(#2,-0.1866856603806199*#3)
-- ++(#2,-0.14612057222206232*#3)
-- ++(#2,-1.5208143566048706*#3)
-- ++(#2,-0.12674558910143668*#3)
-- ++(#2,0.03284126372150488*#3)
-- ++(#2,0.49376702224021973*#3)
-- ++(#2,0.9408628598209601*#3)
-- ++(#2,-0.48464849918260483*#3)
-- ++(#2,-1.0826642787467895*#3)
-- ++(#2,-0.5492786259087735*#3)
-- ++(#2,-1.4464692137973738*#3)
-- ++(#2,0.1945498447645956*#3)
-- ++(#2,-0.2571744776339823*#3)
-- ++(#2,-0.5572876876860753*#3)
-- ++(#2,1.3845468335917674*#3)
-- ++(#2,1.3069381709905314*#3)
-- ++(#2,1.0318321882585508*#3)
-- ++(#2,0.722089374744811*#3)
-- ++(#2,0.23371985082555094*#3)
-- ++(#2,-0.1250388273745063*#3)
-- ++(#2,0.8336951738354083*#3)
-- ++(#2,0.07099401094974858*#3)
-- ++(#2,1.236730232068049*#3)
-- ++(#2,-1.878364715616927*#3)
-- ++(#2,0.8028374019085294*#3)
-- ++(#2,-0.9664023390567179*#3)
-- ++(#2,1.2041782617045762*#3)
-- ++(#2,-0.3331024887742224*#3)
-- ++(#2,0.10017656970661414*#3)
-- ++(#2,-0.28242841683064723*#3)
-- ++(#2,-1.0344551110541202*#3)
-- ++(#2,-0.09864075551010906*#3)
-- ++(#2,-0.46673237743246243*#3)
-- ++(#2,-0.7447081969443416*#3)
-- ++(#2,-0.7685617281965271*#3)
-- ++(#2,0.32067991826237013*#3)
-- ++(#2,1.0216109608492716*#3)
-- ++(#2,1.0101539488658675*#3)
-- ++(#2,0.6623559344377647*#3)
-- ++(#2,-1.13577613549543*#3)
-- ++(#2,-0.14004346846127236*#3)
-- ++(#2,-0.30740563210904565*#3)
-- ++(#2,-1.8030246880051746*#3)
-- ++(#2,1.147136854415572*#3)
-- ++(#2,0.48167982867986786*#3)
-- ++(#2,0.5325396591663276*#3)
-- ++(#2,-1.3732158064274558*#3)
-- ++(#2,0.9780334449078643*#3)
-- ++(#2,-1.758339067598398*#3)
-- ++(#2,0.46005320806505545*#3)
-- ++(#2,0.3978961549684401*#3)
-- ++(#2,2.4708453781791917*#3)
-- ++(#2,-0.5959645354065423*#3)
-- ++(#2,1.0407251739736991*#3)
-- ++(#2,0.46717736013940114*#3)
-- ++(#2,-0.45669477802128866*#3)
-- ++(#2,0.14406190947847333*#3)
-- ++(#2,2.403187814877775*#3)
-- ++(#2,0.6543828194714041*#3)
-- ++(#2,-0.1633703959342626*#3)
-- ++(#2,-0.718720852098346*#3)
-- ++(#2,0.203520401554632*#3)
-- ++(#2,-1.6928217125031093*#3)
-- ++(#2,0.24716859330229696*#3)
-- ++(#2,1.4220798989516823*#3)
-- ++(#2,-0.8379211333866473*#3)
-- ++(#2,0.5532252079601965*#3)
-- ++(#2,7.32209135031439e-05*#3)
-- ++(#2,-2.3383355357703097*#3)
-- ++(#2,-1.336988741893829*#3)
-- ++(#2,1.9031929562707297*#3)
-- ++(#2,0.11859124783438342*#3)
-- ++(#2,-0.9949053482209362*#3)
-- ++(#2,1.73335832240855*#3)
-- ++(#2,1.3926949549275796*#3)
-- ++(#2,0.593876441562065*#3)
-- ++(#2,-1.9386710787484365*#3)
-- ++(#2,-0.7963553978096858*#3)
-- ++(#2,-1.3037492873359564*#3)
-- ++(#2,0.029075033330151715*#3)
-- ++(#2,0.8726156355598937*#3)
-- ++(#2,-0.6674367167424241*#3)
-- ++(#2,-1.453659962240965*#3)
-- ++(#2,0.2728830952158381*#3)
-- ++(#2,1.8107000786898426*#3)
-- ++(#2,0.47712448343927755*#3)
-- ++(#2,-0.8224307263064118*#3)
-- ++(#2,-0.43165268323653067*#3)
-- ++(#2,0.058655149210289116*#3)
-- ++(#2,-0.11227843524908296*#3)
-- ++(#2,0.6394792331935911*#3)
-- ++(#2,0.9327693050997983*#3)
-- ++(#2,-0.5614083682822488*#3)
-- ++(#2,-0.9656046852088518*#3)
-- ++(#2,1.9883656184648504*#3)
-- ++(#2,1.0756748381718808*#3)
-- ++(#2,0.030184828047874097*#3)
-- ++(#2,0.1981103471352144*#3)
-- ++(#2,-0.43061986691757365*#3)
-- ++(#2,-1.2098187495145207*#3)
-- ++(#2,-0.053698416644578056*#3)
-- ++(#2,-0.2730805352897239*#3)
-- ++(#2,-0.20283529787101667*#3)
-- ++(#2,0.5339513518633857*#3)
-- ++(#2,0.17776210541264856*#3)
-- ++(#2,-1.6459265595872965*#3)
-- ++(#2,-1.3422621617226123*#3)
-- ++(#2,-0.11361585879611891*#3)
-- ++(#2,-0.821271637609551*#3)
-- ++(#2,-0.9422337147979347*#3)
-- ++(#2,1.563810898240331*#3)
-- ++(#2,0.9876763566914522*#3)
-- ++(#2,-1.5120176863131105*#3)
-- ++(#2,0.7728648458307845*#3)
-- ++(#2,-0.7659949609674708*#3)
-- ++(#2,0.07780724046958193*#3)
-- ++(#2,-1.1680394308436612*#3)
-- ++(#2,-0.5048202037388365*#3)
-- ++(#2,0.06955312168722573*#3)
-- ++(#2,-0.9825164714796129*#3)
-- ++(#2,0.08809422066521158*#3)
-- ++(#2,-0.4651765164997517*#3)
-- ++(#2,-0.6305616849653464*#3)
-- ++(#2,0.631693197873972*#3)
-- ++(#2,0.29068028292691017*#3)
-- ++(#2,1.6445564583754602*#3)
-- ++(#2,-0.26532213220529305*#3)
-- ++(#2,-0.30961099892913346*#3)
-- ++(#2,0.0808783678763413*#3)
-- ++(#2,0.9500703538999021*#3)
-- ++(#2,-1.0042492278363178*#3)
-- ++(#2,0.25966906645575555*#3)
-- ++(#2,-0.27761070366117657*#3)
-- ++(#2,-0.6374542625676455*#3)
-- ++(#2,1.511555935433511*#3)
-- ++(#2,-0.6294656558795234*#3)
-- ++(#2,-0.309686146858455*#3)
-- ++(#2,-0.826431439675785*#3)
-- ++(#2,1.436143626469873*#3)
-- ++(#2,-1.4134011229046484*#3)
-- ++(#2,1.7355636696968202*#3)
-- ++(#2,1.3030886680681617*#3)
-- ++(#2,-0.4956954867479258*#3)
-- ++(#2,0.43986048760082636*#3)
-- ++(#2,1.8555680642143368*#3)
-- ++(#2,0.2243080410995602*#3)
-- ++(#2,-0.7132967110004695*#3)
-- ++(#2,-0.892146617234264*#3)
-- ++(#2,0.2107574873719904*#3)
-- ++(#2,0.8820690836338018*#3)
-- ++(#2,-0.05896574749534247*#3)
-- ++(#2,-0.1943942285676351*#3)
-- ++(#2,-0.4401323360831701*#3)
-- ++(#2,1.053344358330924*#3)
-- ++(#2,0.8441883232891918*#3)
-- ++(#2,0.4889531062564217*#3)
-- ++(#2,0.6262989733001242*#3)
-- ++(#2,1.594075535908993*#3)
-- ++(#2,0.1790250127312272*#3)
-- ++(#2,-1.6667642565454697*#3)
-- ++(#2,-0.7594788082806018*#3)
-- ++(#2,-0.4797360985580589*#3)
-- ++(#2,-0.44327039729970635*#3)
-- ++(#2,0.8517908847754022*#3)
-- ++(#2,1.4543600104406902*#3)
-- ++(#2,-0.8314309544274088*#3)
-- ++(#2,0.2550318478996842*#3)
-- ++(#2,1.9672317860414195*#3)
-- ++(#2,-0.2150180491861805*#3)
-- ++(#2,0.24878581180731185*#3)
-- ++(#2,0.8243740371338848*#3)
-- ++(#2,-0.2351341423841179*#3)
-- ++(#2,1.3004815090935002*#3)
-- ++(#2,-0.6554748297877541*#3)
-- ++(#2,-2.372351178904624*#3)
-- ++(#2,2.307235865919928*#3)
-- ++(#2,0.6471148561091563*#3)
-- ++(#2,-0.40257191132241676*#3)
-- ++(#2,0.930093948210321*#3)
-- ++(#2,0.2495531196427746*#3)
-- ++(#2,-1.1393781860757073*#3)
-- ++(#2,0.8530754823660269*#3)
-- ++(#2,-0.7570765373161416*#3)
-- ++(#2,-0.5642186412377671*#3)
-- ++(#2,1.5216479704147445*#3)
-- ++(#2,0.8297351294821347*#3)
-- ++(#2,-0.8421352066743846*#3)
-- ++(#2,0.2730114546352103*#3)
-- ++(#2,0.24893261575191045*#3)
-- ++(#2,-0.6793057694365886*#3)
-- ++(#2,-0.7623410897686969*#3)
-- ++(#2,0.44108965370036024*#3)
-- ++(#2,-1.0140353120792704*#3)
-- ++(#2,-0.0983521389136457*#3)
-- ++(#2,1.3690789293539762*#3)
-- ++(#2,-0.15611611620563556*#3)
-- ++(#2,1.4573045366206745*#3)
-- ++(#2,0.4539681188925964*#3)
-- ++(#2,2.466591108411921*#3)
-- ++(#2,-0.021502737219134226*#3)
-- ++(#2,1.235820420934888*#3)
-- ++(#2,-0.5380117059622815*#3)
-- ++(#2,0.08210839370763409*#3)
-- ++(#2,0.9498161542219592*#3)
-- ++(#2,0.7034777454730891*#3)
-- ++(#2,0.5851562677195666*#3)
-- ++(#2,-0.8499707256648564*#3)
-- ++(#2,-0.8899692985141584*#3)
-- ++(#2,-0.22599049868418555*#3)
-- ++(#2,1.2729774358496924*#3)
-- ++(#2,1.3058294178904633*#3)
-- ++(#2,0.6406258501448184*#3)
-- ++(#2,0.937508195747551*#3)
-- ++(#2,1.716660723112716*#3)
-- ++(#2,-2.12228768928222*#3)
-- ++(#2,0.007481563708233615*#3)
-- ++(#2,0.6225830564016436*#3)
-- ++(#2,0.019493431981937204*#3)
-- ++(#2,-0.2296025527558525*#3)
-- ++(#2,0.5077787921952798*#3)
-- ++(#2,-0.7121627776776612*#3)
-- ++(#2,-0.22794897944114417*#3)
-- ++(#2,-0.9008813875711585*#3)
-- ++(#2,0.04271467169271504*#3)
-- ++(#2,-1.248752125278161*#3)
-- ++(#2,-0.11259017813745006*#3)
-- ++(#2,-0.9991698752512583*#3)
-- ++(#2,0.14413540176088074*#3)
-- ++(#2,-0.23076508553226507*#3)
-- ++(#2,1.697205687164822*#3)
-- ++(#2,-0.09413747438049248*#3)
-- ++(#2,0.1521884201223902*#3)
-- ++(#2,0.9502856334549967*#3)
-- ++(#2,2.1272056708191522*#3)
-- ++(#2,-1.203056483232382*#3)
-- ++(#2,-0.22984004459146357*#3)
-- ++(#2,-0.5726524863921816*#3)
-- ++(#2,0.9948433702858864*#3)
-- ++(#2,0.6123126108491557*#3)
-- ++(#2,-0.0003476766678410057*#3)
-- ++(#2,-0.6766961576677754*#3)
-- ++(#2,-1.5576141636408143*#3)
-- ++(#2,0.19728162371757343*#3)
-- ++(#2,1.1378147598298438*#3)
-- ++(#2,-0.7325288132443483*#3)
-- ++(#2,-2.417760380594628*#3)
-- ++(#2,-1.0110848852242318*#3)
-- ++(#2,-1.3097079718730738*#3)
-- ++(#2,0.8455528485076227*#3)
-- ++(#2,-0.33565860677077286*#3)
-- ++(#2,-0.3448958998010851*#3)
-- ++(#2,0.5079494687474971*#3)
-- ++(#2,-0.4017015573603293*#3)
-- ++(#2,0.6640261618045663*#3)
-- ++(#2,0.8911273898117138*#3)
-- ++(#2,0.6260506841846873*#3)
-- ++(#2,-0.7076239290758172*#3)
-- ++(#2,0.8643731200131871*#3)
-- ++(#2,-0.6972270298963349*#3)
-- ++(#2,1.1432060312740329*#3)
-- ++(#2,0.8777194825100931*#3)
-- ++(#2,-0.36042438111231356*#3)
-- ++(#2,-0.7010840138258922*#3)
-- ++(#2,-0.8288769880617723*#3)
-- ++(#2,-0.4801250104146166*#3)
-- ++(#2,0.5880537620381899*#3)
-- ++(#2,1.4802868886942888*#3)
-- ++(#2,-1.514050872936001*#3)
-- ++(#2,0.019660392051271434*#3)
-- ++(#2,0.4445161858251379*#3)
-- ++(#2,-0.25173959612397195*#3)
-- ++(#2,-1.1853975111565431*#3)
-- ++(#2,-1.7448854081691318*#3)
-- ++(#2,0.44468877427477915*#3)
-- ++(#2,0.256719205617879*#3)
-- ++(#2,0.12058654954010334*#3)
-- ++(#2,-0.4186527191793887*#3)
-- ++(#2,1.0267059840508435*#3)
-- ++(#2,1.0805092532252878*#3)
-- ++(#2,-0.2320165650711572*#3)
-- ++(#2,0.28406716307676483*#3)
-- ++(#2,2.085462946248782*#3)
-- ++(#2,-1.4026918056122466*#3)
-- ++(#2,0.19741877978666955*#3)
-- ++(#2,1.9387073019579166*#3)
-- ++(#2,1.2641237442953166*#3)
-- ++(#2,0.05098647505037036*#3)
-- ++(#2,0.05094814180745045*#3)
-- ++(#2,-1.6201757374614407*#3)
-- ++(#2,0.28297809038976035*#3)
-- ++(#2,-1.5225583936112956*#3)
-- ++(#2,-1.8433081008129506*#3)
-- ++(#2,-0.2017007848438185*#3)
-- ++(#2,0.2572848578224259*#3)
-- ++(#2,-0.19129392448786622*#3)
-- ++(#2,1.8076957117598977*#3)
-- ++(#2,0.7038514664823516*#3)
-- ++(#2,-2.0719173095493275*#3)
-- ++(#2,-1.008062781400189*#3)
-- ++(#2,0.8379032357336715*#3)
-- ++(#2,-0.20079895637871034*#3)
-- ++(#2,-0.05025354096739648*#3)
-- ++(#2,1.002499237367182*#3)
-- ++(#2,-0.2069855230029513*#3)
-- ++(#2,0.46848305957833014*#3)
-- ++(#2,0.545500933058521*#3)
-- ++(#2,-0.40202547389501164*#3)
-- ++(#2,0.20412295288267385*#3)
-- ++(#2,-0.14338461177636322*#3)
-- ++(#2,-0.7108427605874968*#3)
-- ++(#2,0.7872152070084891*#3)
-- ++(#2,0.010173742206909864*#3)
-- ++(#2,-1.880534312818036*#3)
-- ++(#2,0.4445513337034989*#3)
-- ++(#2,-0.041390193809471267*#3)
-- ++(#2,0.8184021409671313*#3)
-- ++(#2,-1.2166995262583733*#3)
-- ++(#2,0.48722377213081325*#3)
-- ++(#2,1.1477099302799865*#3)
-- ++(#2,-0.577147136449144*#3)
-- ++(#2,0.810061874291811*#3)
-- ++(#2,0.7471951519324349*#3)
-- ++(#2,-0.021659657403822606*#3)
-- ++(#2,1.5582712638172427*#3)
-- ++(#2,0.8255531671779788*#3)
-- ++(#2,0.4065582254780304*#3)
-- ++(#2,0.730724969842887*#3)
-- ++(#2,-0.6855285666392825*#3)
-- ++(#2,0.08102266644906661*#3)
-- ++(#2,1.0277594828173664*#3)
-- ++(#2,0.38105742320239244*#3)
-- ++(#2,1.4597947825391893*#3)
-- ++(#2,0.20067009920163933*#3)
-- ++(#2,-0.8050514888937973*#3)
-- ++(#2,1.2430671766153356*#3)
-- ++(#2,1.33987234414593*#3)
-- ++(#2,0.6833541786460321*#3)
-- ++(#2,0.12089046472807265*#3)
-- ++(#2,0.2647417830881001*#3)
-- ++(#2,-1.0480698531798258*#3)
-- ++(#2,-0.0015240276672427763*#3)
-- ++(#2,0.3011696138007824*#3)
-- ++(#2,-0.5201746884489986*#3)
-- ++(#2,-1.0538788187646182*#3)
-- ++(#2,-1.3149564266282314*#3)
-- ++(#2,1.3873515093746347*#3)
-- ++(#2,0.42719572096605696*#3)
-- ++(#2,-0.41650086121836766*#3)
-- ++(#2,0.03564579600642637*#3)
-- ++(#2,-0.4644659172711212*#3)
-- ++(#2,0.41496617160344007*#3)
-- ++(#2,0.2654401587798895*#3)
-- ++(#2,-0.21943695955239928*#3)
-- ++(#2,-1.0055265126918147*#3)
-- ++(#2,-0.11149645963815474*#3)
-- ++(#2,0.10106061581443815*#3)
-- ++(#2,-0.06935194400456982*#3)
-- ++(#2,-0.21023170967265054*#3)
-- ++(#2,-0.04784634061973255*#3)
-- ++(#2,-1.1661901281302347*#3)
-- ++(#2,-0.5582555417892962*#3)
-- ++(#2,1.2168991755484957*#3)
-- ++(#2,-0.4408758977713946*#3)
-- ++(#2,0.6500307793036476*#3)
-- ++(#2,-2.548818515710729*#3)
-- ++(#2,1.0502876955732778*#3)
-- ++(#2,-0.47927439615031475*#3)
-- ++(#2,-0.10427262957817388*#3)
-- ++(#2,-0.959745270468569*#3)
-- ++(#2,0.06709162608699215*#3)
-- ++(#2,-1.1417679545024713*#3)
-- ++(#2,0.19088999359745717*#3)
-- ++(#2,0.2341820010676057*#3)
-- ++(#2,0.051537323455575955*#3)
-- ++(#2,-2.198391436695416*#3)
-- ++(#2,-0.5406097674433323*#3)
-- ++(#2,0.5020865889058664*#3)
-- ++(#2,-0.02604832462274743*#3)
-- ++(#2,-0.5988445164746464*#3)
-- ++(#2,-0.3229904730845742*#3)
-- ++(#2,1.5181805454058876*#3)
-- ++(#2,1.3101838554046052*#3)
-- ++(#2,-0.47127282483320293*#3)
-- ++(#2,-1.6158825922295872*#3)
-- ++(#2,-0.20734698777954286*#3)
-- ++(#2,-1.2540459041882952*#3)
-- ++(#2,0.95189519339206*#3)
-- ++(#2,0.7985887407869328*#3)
-- ++(#2,0.7237882538029243*#3)
-- ++(#2,0.10760661231619982*#3)
-- ++(#2,1.024989873911237*#3)
-- ++(#2,0.8133842875066101*#3)
-- ++(#2,-0.24207845800002925*#3)
-- ++(#2,-1.408221808736385*#3)
-- ++(#2,-0.6525764108757117*#3)
-- ++(#2,-0.6207997762618905*#3)
-- ++(#2,0.972340210631979*#3)
node[above,right] {#5};
}

\DeclareMathOperator*{\argmax}{arg\,max}

\newcommand\independent{\protect\mathpalette{\protect\independenT}{\perp}}
\def\independenT#1#2{\mathrel{\rlap{$#1#2$}\mkern2mu{#1#2}}}

\newtheorem{theorem}{Theorem}
\newtheorem{lemma}{Lemma}
\newtheorem{definition}{Definition}

\let\oldref\ref
\renewcommand{\ref}[1]{(\oldref{#1})}

\newcommand{\thmref}[1]{Theorem~\ref{#1}}

\newcommand{\alggref}[1]{Algorithm~\ref{#1}}

\newcommand{\lemref}[1]{Lemma~\ref{#1}}

\newcommand{\rT}{r_T}
\newcommand{\xM}{x_M}

\usepackage{algorithm}
\usepackage{algorithmicx}
\usepackage{algpseudocode}

\usepackage{tikz}
\usepackage{tkz-euclide}

\usepackage{caption}
\usepackage{subcaption}
\usepackage{xcolor}

\allowdisplaybreaks

\begin{document}
%
\title{Tight Regret Bounds for Noisy \\ Optimization of a Brownian Motion}
%
%
%

\author{Zexin Wang, Vincent Y.\ F.\ Tan, \emph{Senior Member, IEEE}, and Jonathan Scarlett, \emph{Member, IEEE}
	\thanks{This work was supported in part by the Singapore National Research Foundation (NRF) under grant numbers R-252-000-A74-281 and R-263-000-D02-281.}%
	\thanks{Z.~Wang is with the  Department of Mathematics, Imperial College London, London, SW7 2AZ, UK (email: zexin.wang19@imperial.ac.uk).}%
	\thanks{V.~Y.~F.~Tan is with the Department of Electrical and Computer Engineering, and with the Department of Mathematics, National University of Singapore, Singapore (email: vtan@nus.edu.sg).} \thanks{J.~Scarlett is with the Department of Computer Science   and with the Department of Mathematics, National University of Singapore, Singapore (email: scarlett@comp.nus.edu.sg).}}

%
%

\markboth{}%
{Shell \MakeLowercase{\textit{et al.}}: Bare Demo of IEEEtran.cls for IEEE Journals}
%



\maketitle

\begin{abstract}
We consider the problem of Bayesian optimization  of a one-dimensional Brownian motion in which the $T$ adaptively chosen observations are corrupted by Gaussian noise. We show that as the smallest possible expected cumulative regret and the smallest possible expected simple regret scale as $\Omega(\sigma\sqrt{T / \log (T)}) \cap \mathcal{O}(\sigma\sqrt{T} \cdot \log T)$ and $\Omega(\sigma / \sqrt{T \log (T)}) \cap \mathcal{O}(\sigma\log T / \sqrt{T})$ respectively, where $\sigma^2$ is the noise variance. Thus, our upper and lower bounds are tight up to a factor of $\mathcal{O}( (\log T)^{1.5} )$.  The upper bound uses an algorithm based on confidence bounds and the Markov property of Brownian motion (among other useful properties), and the lower bound is based on a reduction to binary hypothesis testing.
\end{abstract}


\vspace*{-1ex}
\section{Introduction}

Brownian motion (BM) is a continuous-time stochastic process widely studied in diverse fields such as physics \cite{kuhn1987black}, biology \cite{adler2019conventional} and finance \cite{kijima2016stochastic}. Specifically, BM is used to model random behavior of the movements of random particles in physical and biological systems, as well as the movements of financial asset prices. As a Gauss--Markov process, it inherits properties of a Gaussian process (GP), and also the Markov property.  There have been several studies on methods and algorithms for optimizing a BM \cite{grill2018optimistic,al1996optimal,abdechiri2013gases,calvin2017adaptive}. However, the {\em fundamental limits} (i.e., upper and lower bounds on the regret) for doing so have remained elusive for the most part; this is the main purpose of the present study.
In the broader context of GPs, Bayesian optimization (BO) \cite{mockus89} is a sequential design strategy for global optimization of black-box functions. Here, a Gaussian process prior is assumed on an unknown function $f$. This prior is then  updated to a posterior upon the observation of noisy samples, and further samples are selected based on the updated posterior.  The high-level goal is to maximize $f$ in as few function evaluations as possible. More precisely, following existing works in the literature \cite{srinivas2012information, scarlett2018tight}, we consider the following performance metrics, which are  respectively termed the \textit{simple regret} and \textit{cumulative regret}:
\begin{align}
    \label{SimpleRegretDefinition}
    \rT &= \max_x f(x) - f(x^{(T)}) \\
    \label{CumulativeRegretDefinition}
    R_T &= \sum_{t=1}^T \left(\max_x f(x) - f(x_t)\right).
\end{align}
Here, $x_t$ is the point chosen at time $t$, and $x^{(T)}$ is an additional point returned after the $T$-th time instant.  It is important to note that $\rT$ and $R_T$ are random variables; they depend on the random function $f$, the noise introduced to the samples, and any source of algorithmic randomness in the selection of the points $\{x_t\}$.

The existing literature on BO focuses mainly on functions $f$ that are smooth (i.e., differentiable).  However, in recent applications,  non-smooth functions have become increasingly important \cite{du2019gradient}. Furthermore, gradient (or subgradient) information is often expensive or unavailable, and hence, zeroth-order methods, motivating the study of BO, are particularly attractive.   We use BM as an archetypal example of a non-smooth random  function, and study the fundamental performance limits of BO on a standard BM. To do so, we consider both achievability results (existence results upper bounding the regret) and impossibility results (algorithm-independent lower bounds on the regret).
For the former, we propose an asymptotically near-optimal algorithm based on upper and lower confidence bounds. For the latter,  and for the latter
we reduce the BO problem to a binary hypothesis test (see Sections \ref{sec:related} and \ref{sec:insuff} for related works).

Some potential applications of noisy BM optimization are as follows:
\begin{enumerate} 
    \item Bayesian optimization has been applied to environmental monitoring (e.g., see Marchant and Ramos~\cite{marchant2012bayesian}), which is inherently noisy due to imperfect sensors. In this context, erratic signals may be better modeled by BM compared to using smooth kernels.
    \item Similarly, in hyperparameter tuning problems, erratic behavior may be modeled by a BM.  For instance, in Swersky {\em et al.}~\cite{swersky2014freeze}, the closely-related non-smooth Ornstein--Uhlenback process was adopted as part of the model.
\end{enumerate}
Beyond any specific applications, we believe that this problem is important in the broad context of Bayesian optimization and continuous-armed or continuum-armed bandits. In particular, Scarlett~\cite{scarlett2018tight} advanced the theory of smooth Bayesian optimization, but left a significant gap concerning non-smooth functions that our work partially closes.

\vspace*{-1ex}
 \subsection{Related Work} \label{sec:related}

There have been several optimization algorithms proposed for BM in the literature; the most relevant one is by Grill {\em et al.}~\cite{grill2018optimistic}. The main contribution therein is the proposal of an algorithm---termed Optimistic Optimization of a Brownian (OOB)---for the maximization of a BM in the noiseless setting. The sample complexity was shown to be  $\mathcal{O}(\log^2(1/\varepsilon))$. That is, the minimum number of samples to guarantee that one of the selected samples is $\varepsilon$-close to the global maximum of the BM with probability at least $1-\varepsilon$ is $\mathcal{O}(\log^2(1/\varepsilon))$.   This significantly improves over the earlier results of Al-Mharmah and Calvin \cite{al1996optimal} and Calvin {\em et al.}~\cite{calvin2017adaptive}, in which the upper bound on the sample complexity was a polynomial in $1/\varepsilon$.

There is also a vast literature on BO for {\em smooth} functions. In particular, for a GP with the squared-exponential (SE) kernel or the  Mat\'ern kernel with parameter $\nu>2$, Scarlett~\cite{scarlett2018tight} gave cumulative regret bounds that are tight up to a $\sqrt{\log T}$ factor.  Earlier bounds were also given for the GP-UCB algorithm in the prominent work of Srinivas {\em et al.}~\cite{srinivas2012information}, attaining a near-tight upper bound for the SE kernel but not the Mat\'ern kernel.
In the noiseless case, Gr{\"u}new{\"a}lder {\em et al.}~\cite{grunewalder2010regret} gave both upper and lower cumulative regret bounds assuming the mean and kernel satisfy an $\alpha$-H\"{o}lder continuity condition, and Kawaguchi {\em et al.}~\cite{kawaguchi2015bayesian} proved exponential convergence of the simple regret under certain smoothness assumptions on the kernel.

While our focus is in Bayesian optimization with a stochastic process model, there also exist several related works on the black-box optimization of {\em deterministic} functions with H\"older-type continuity assumptions.  For instance, see Munos~\cite{munos2011optimistic} for the noiseless setting and Shang {\em et al.}~\cite{shang2019general} for the noisy setting, as well as the references therein.  Connections between BM optimization and these works is further discussed in Section~\ref{sec:insuff} below. 

\vspace*{-1ex}
\subsection{Contributions}

Our main results state that the smallest possible expected cumulative regret for optimization of a BM with a fixed noise variance $\sigma^2 > 0$ behaves as $\Omega(\sigma\sqrt{T / \log (T)}) \cap \mathcal{O}(\sigma\sqrt{T} \cdot \log T)$, and the smallest possible expected  simple regret behaves as $\Omega(\sigma / \sqrt{T \log (T)}) \cap \mathcal{O}(\sigma\log T / \sqrt{T})$. In both cases, the gap between the upper and lower bound is only $\mathcal{O}( (\log T)^{1.5} )$.  In more detail, our technical contributions and observations consist of the following:

\begin{enumerate}

\item We develop an upper and lower confidence bound based algorithm that is amenable to the setting of noisy observations; this is in contrast to the OOB algorithm of Grill {\em et al.}~\cite{grill2018optimistic}, which is tailored to the noiseless setting.  We characterize the performance in terms of both the simple regret and cumulative regret.
    
    \item Our lower bound is based on a novel adaptation of the approach of \cite{scarlett2018tight}, requiring several additional technical challenges to move from smooth functions to BM.
    
    \item Our lower bound on the simple regret implies that the optimal sample complexity for the noisy setting is at least polynomial in $1/\varepsilon$ for a fixed  precision $\varepsilon$. This is much larger than the logarithmic upper bound for the noiseless counterpart proposed in \cite{grill2018optimistic}.
    
    \item The similarities between the {\em cumulative}\footnote{In \cite{scarlett2018tight}, simple regret was not considered, and in fact, attaining matching upper and lower simple regret bounds in the smooth setting appears to be more challenging due to subtleties regarding $O\big( \frac{1}{\sqrt T} \big)$ vs.~$O\big( \frac{1}{T} \big)$ dependence \cite{shamir2013complexity}.} regret bounds and the corresponding proof techniques of BO for smooth functions~\cite{scarlett2018tight} and BM indicate that the level of smoothness is not necessarily the key feature that dictates the difficulty of optimization, and that H\"older-type continuity may be similarly beneficial.
    
\end{enumerate}

\vspace*{-1ex}
\subsection{Discussion of Existing Approaches} \label{sec:insuff}

We briefly pause to discuss certain approaches that one might consider adopting based on existing works, and comment on why they appear to be insufficient for the purposes of establishing our main results.

As mentioned above, Srinivas {\em et al.}~\cite{srinivas2012information} analyze the GP-UCB algorithm, and provide an upper bound on the regret for general kernels in terms of a mutual information quantity called the {\em information gain}.  While it is tempting to try to apply this result to Brownian motion, there are two major difficulties in doing so:
\begin{itemize} 
    \item In \cite[Thm.~2]{srinivas2012information}, it is assumed that the derivatives of the function are bounded with high probability, rendering the result inapplicable for nowhere-differentiable processes such as BM consider in this paper.  In fact, the authors go on to conjecture that their result does {\em not} hold for such processes; see the end of \cite[Section V.A]{srinivas2012information} therein.
    \item Even if an analogous result {\em were} to hold for BM, attaining a $\sqrt{T}{\rm poly}(\log T)$ regret bound would require showing that the information gain behaves as ${\rm poly}(\log T)$, which appears to be unlikely given that even the smoother Mat\'ern kernel only has an $\widetilde{\mathcal{O}}(T^c)$   bound (with $c \in (0,1)$ depending on the smoothness parameter $\nu$) on its information gain.
\end{itemize}

The above-mentioned work by Shang {\em et al.}~\cite{shang2019general} provides general bounds for the black-box optimization of deterministic functions with H\"older-type continuity assumptions.  This continuity is captured by variants of the {\em near-optimality dimension} \cite{munos2011optimistic}, which roughly quantifies the volume of $\epsilon$-optimal points that exist in the limit as $\epsilon \to 0$ (cf.~Definition~\ref{def:near_opt}).  It was demonstrated in Grill {\em et al.} \cite{grill2018optimistic} that the BM exhibits a certain {\em modified definition} of the near-optimality dimension that considers the {\em average} number of $\epsilon$-optimal points (since the BM is stochastic).  When the near-optimality dimension according to this modified definition is substituted into the bounds for the noiseless deterministic setting \cite{munos2011optimistic} (despite the mismatch in the underlying definitions), the resulting regret bounds for the noiseless setting \cite{munos2011optimistic,grill2018optimistic} indeed match up to constant factors.  An analogous observation turns out to hold in the noisy setting, with our upper bound coinciding with \cite{shang2019general} in the same way that \cite{grill2018optimistic} coincides with \cite{munos2011optimistic}.

However, while the bounds may end up matching in this sense (upon replacing the near-optimality dimension with its analog concerning the average number of $\epsilon$-optimal points), to the best of our knowledge, existing results for the deterministic setting do not {\em formally} transfer to the stochastic setting.  The algorithm that we use for our upper bound resembles those of \cite{munos2011optimistic,grill2018optimistic,shang2019general}, but requires careful modifications and an analysis that specifically exploits several properties of the BM, such as its Markovity and normal increments; see Section~\ref{sec:aux_lemmas} and Appendix \ref{app:results} for other properties of the BM that we exploit. 

Another approach that one might consider is to simply take the noiseless BM optimization algorithm of Munos~\cite{munos2011optimistic} and repeatedly sample every point sufficiently many times so that we essentially have close-to-noiseless observations.  However, this turns out to yield a suboptimal result.  Roughly speaking, resampling enough times to bring the uncertainty down to $\eta$ requires at least $\Omega\big( \frac{1}{\eta^2} \big)$ repetitions, and the resulting cumulative regret then scales as $\Omega( \frac{1}{\eta^2} + \eta T )$.  This expression is minimized by $\eta = \Theta(T^{-\frac{1}{3}})$, thus yielding suboptimal ${\cal O}(T^{ \frac{2}{3}})$ behavior (for the cumulative regret).  In contrast, we attain a near-optimal bound by carefully considering the tradeoff between the number of resampled points and how many iterations have passed. In particular, our algorithm uses fewer resampled points in the earlier iterations and more resampled points in the later iterations. 

Our lower bound follows similar high-level steps to \cite{scarlett2018tight}, which in turn builds on earlier works such as \cite{raginsky2011information}.  However, the details are very different due to the fact that \cite{scarlett2018tight} crucially exploits the twice-differentiability of the function, whereas the BM is almost surely non-differentiable everywhere.

\vspace*{-1ex}
\subsection{Paper Organization} 

 The rest of the paper is structured as follows. In Section~\ref{sec:setup}, we formally describe the problem setup and state our objectives. In Section~\ref{sec:ach}, we describe the algorithm used, state the achievable regret bounds, and provide the core steps of the proof. In Section~\ref{sec:conv}, we state the impossibility results and provide their core proof steps.  In Section~\ref{sec:num}, we conduct numerical experiments to demonstrate that the expected cumulative regret indeed scales as roughly~$\sqrt{T}$. In Section~\ref{sec:concl}, we conclude our discussion. In Section~\ref{app:proofs}, we prove several technical results used in the earlier sections, with some details further relegated to the appendices (supplementary material).

\vspace*{-1ex}
\section{Problem Setup} \label{sec:setup}

Over a fixed time   horizon  $T$, we seek to sequentially optimize a realization of a  standard BM $W=(W_x)_{x\in D}$ over the one-dimensional domain $D = [0,1]$; any finite interval can be transformed to this choice via shifting and re-scaling.  Note that we envision $W$ as being generated by nature ``in advance'', i.e., prior to any samples being taken.

At time $1$, we select a single point $x_1\in D$ and observe a noisy sample $y_1=W_{x_1}+z_1$ where $z_1\sim \mathcal{N}(0, \sigma^2)$ for some noise variance $\sigma^2 > 0$. At time $t\in\{2,\ldots,T\}$, given the previously sampled points and their noisy function evaluations $\{(x_\tau,y_\tau)\}_{\tau=1}^{t-1}$, we query an additional point $x_t \in D$ and observe a noisy sample $y_t = W_{x_t} + z_t$, where  $(z_t)_{1\le t\le T} \stackrel{\text{i.i.d}}{\sim} \mathcal{N}(0, \sigma^2)$.  We measure the optimization performance using the expected simple regret $\mathbb{E}[\rT]$ and expected cumulative regrets $\mathbb{E}[R_T]$, according to the definitions in \eqref{SimpleRegretDefinition} and \eqref{CumulativeRegretDefinition} with $f(x) = W_x$.

When viewed as a GP, $W$ has zero mean and a non-stationary kernel $\mathbb{E}[ W_{x_1}W_{x_2} ]= \min(x_1, x_2)$.  The BM has many useful properties~\cite{Karatzas88}, notably including the {\em Markov property} and {\em normal increments}: For all $0 \le x_1 \le x_2, (W_{x_2} - W_{x_1}) \independent W_{x_1}$ and $(W_{x_2} - W_{x_1}) \sim \mathcal{N}(0, x_2 - x_1)$.  Some further useful properties are stated in Appendix~\ref{app:results}.

In the following analyses, we use $c_1, c_2, $ etc.\ to denote generic universal constants that may differ from line to line.

\vspace*{-1ex}
\section{Upper Bounds} \label{sec:ach}

In this section, we introduce a confidence-bound based algorithm (see \alggref{NoisyOOBAlgo}) and derive an upper bound on its regret.  The idea is to sequentially discretize the search space, and rule out suboptimal points using confidence bounds, as is commonly done in bandit algorithms (e.g., see \cite[Ch.~22]{lattimore2020bandit}).  The algorithm works in epochs, with each epoch containing fewer remaining points and discretizing at a finer scale; the subsequent analysis seeks to bound the number of points sampled per epoch, and thereby obtain the overall regret bounds.  These bounds are formally stated as follows.


\newcommand{\tilc}{\tilde{c}}

\begin{theorem}\label{CumulativeRegretThm}
For the problem of BM optimization with a noise variance $\sigma^2 > 0$ satisfying $\sigma^2 \ge \frac{\tilc}{T^{1-\zeta}}$ for some positive constants $(\tilc,\zeta)$,\footnote{This is a very mild assumption, since we are primarily interested in the case that $\sigma$ is constant with respect to $T$.}  
there exists an algorithm (\alggref{NoisyOOBAlgo}) achieving the following:
 \begin{align}
\hspace{-.03in}\mathbb{E}[R_T]=\mathcal{O}(\sigma T^{\frac{1}{2}} \log T ), \,\,\,\mbox{and}\,\,\,
\mathbb{E}[\rT] =\mathcal{O}(\sigma T^{-\frac{1}{2}} \log T ).
\end{align}
\end{theorem}
In the rest of the section, we describe the algorithm and present the proof (with many details deferred to Section~\ref{app:proofs} and the supplementary materials).

\subsection{Description of our algorithm}

\begin{algorithm}[ht] 
    \caption{Elimination-based algorithm for Brownian motion optimization \label{alg_ub} }
	\label{NoisyOOBAlgo}
	\begin{algorithmic}[1] 
	\Require Domain $D$, time horizon $T$
	\State Initial set of intervals: $\mathcal{I}_0 = \{ [0,1] \}$
	\State Initialize time index $t = 1$ and epoch number $h = 0$
	\State Sample at $x=1$ for $\displaystyle \lceil \sigma^2 \rceil$ times
	\State Preset $\delta = T^{-1}$ as the error probability
	\While{$t \le T$}
  	\State Find candidate intervals: 
	\hspace{-1mm} \begin{align} \mathcal{L}_{h+1} = \left\{  I \,\Big|\, I \in \mathcal{I}_{h}, \text{UCB}_{h}(I) \ge \max_{I \in \mathcal{I}_{h}} \text{LCB}_{h}(I)\right\}\end{align}
  	\State Slice up the intervals:
	\begin{align} \mathcal{I}_{h+1} = \left\{   \left[a, \frac{a+b}{2}\right], \left[\frac{a+b}{2}, b\right]\, \Big|\, [a,b] \in \mathcal{L}_{h} \right\}\end{align}
  	\State Find the corresponding set of points:
	\begin{align} \mathcal{J}_{h+1} = \bigcup \left\{ \Big\{a, \frac{a+b}{2}, b \Big\} \,\, \, \Big|\, [a,b] \in \mathcal{L}_{h}\right\}\end{align}
 	\State Increment $t$ by $\displaystyle n_h \left\lvert \mathcal{J}_{h+1} \right\rvert$. 
 	\State If $t \le T$, then sample each point in $\mathcal{J}_{h+1}$ for $n_h$ times, where
 	   \begin{align} n_h = \lceil \sigma^22^{h+1} \rceil.\end{align}
    $\quad$ Otherwise, if $t > T$, then only sample until a total of~$T$ samples have been taken.
	\State Increment $h$ by $1$
	\EndWhile
	\State For the simple regret criterion, return $x^{(T)}$ uniformly at random from the set of selected points $\{x_1,\dotsc,x_T\}$.
	\end{algorithmic}
\end{algorithm}

\alggref{NoisyOOBAlgo} works in epochs, with each epoch sampling points restricted to a finer grid than the previous epoch.  This bears some resemblance to OOB by Grill {\em et al.}~\cite{grill2018optimistic}, and similar ideas have also been used for the optimization of smooth functions, e.g., by de Freitas {\em et al.}~\cite{Fre12}, with the main difference being the rule for discarding suboptimal points.

Let $T_h$ be the number of samples taken up to and including the $h$-th epoch, and $t_h$ be the number of samples taken during the $h$-th epoch. Hence, $T_h = \sum_{h^\prime=1}^{h} t_{h^\prime}$. We respectively define the upper and lower confidence bounds\footnote{We preset $\delta$ as the probability of failure for the confidence bounds as a function of $T$, assuming that $T$ is fixed and known.  The case of an unknown time horizon $T$ can be handled using a standard doubling trick, e.g, see \cite[Appendix A]{scarlett2018tight}.} with noisy observations from the first $h$ epochs as
\begin{align}
    \text{UCB}_h([a,b]) &\triangleq \max \big(\bar{y}_a^{(h)}, \bar{y}_b^{(h)}\big) + \beta_\delta(b-a), \label{eq:UCBh} \\
    \text{LCB}_h([a,b]) &\triangleq \min\big(\bar{y}_a^{(h)}, \bar{y}_b^{(h)}\big) - \beta_\delta(b-a), \label{eq:LCBh}
 \end{align}
where $\bar{y}_x^{(h)} = {\rm Average}\left( \left\{ y_s \mid 1 \le s \le T_h, x_s = x \right\} \right)$ is the average of all observations at point $x \in D$ until the $h$-th epoch,
\begin{equation} \label{eqn:eta_alpha}
\eta_\delta(x) \triangleq \sqrt{\frac{5x}{2} \ln\left( \frac{2}{x\delta} \right)} , \,\, \mbox{and} \,\, \alpha_\delta(x) \triangleq \sqrt{6x \ln \left(\frac{1}{x\delta}\right)},
\end{equation}
and $\beta_\delta(x) \triangleq \eta_\delta(x) + \alpha_\delta(x)$.  Here, $\eta_\delta(x)$ is a term that  was used in Grill {\em et al.} \cite{grill2018optimistic}, and represents inherent uncertainty that would exist even in the case of noiseless observations.  The term $\alpha_\delta(x)$ is additionally introduced due to the fact that we require wider confidence bounds in the presence of noise.  With these notations, our algorithm is shown in \alggref{NoisyOOBAlgo}.

At the $h$-th epoch, all sampled points (either at the midpoint or at the ends of the candidate intervals) are sampled at least $n_{h}$ times, to ensure a certain confidence level on $W$ across different points in the same epoch.\footnote{To mitigate the effect of the noise, at each epoch, we resample the chosen point at least $n_h$ times. Re-sampling is convenient for the analyses, whereas in practice one could use a more conventional UCB-type algorithm that samples a different point at each time instant (e.g., Srinivas {\em et al.} \cite{srinivas2012information}).}  The set $\mathcal{L}_h$ can be viewed as the collection of intervals that potentially contain the maximizer; the confidence bounds shrink as more points are observed, and accordingly the total area of $\mathcal{L}_h$ shrinks as $h$ increases.

\vspace*{-1ex}
\subsection{Auxiliary Lemmas}

We now define some high probability events. Firstly, we define the following event, which was introduced in \cite[Definition 1]{grill2018optimistic}:
\begin{align}
    &\mathcal{C} \triangleq \bigcap_{h=0}^\infty \bigcap_{k=0}^{2^h-1} \left\{ \sup_{x \in \left[ \frac{k}{2^h}, \frac{k+1}{2^h} \right]} W_x \le B_{\left[ k/2^h, (k+1)/2^h \right]} \right\}, \label{eqn:eventC}
\end{align} 
where $B_{[a,b]} = \max(W_a, W_b) + \eta_\delta(b-a)$, and $\eta_\delta$ is defined in \eqref{eqn:eta_alpha}. This is the event that the BM $W$ does not exceed a prescribed amount beyond its end points of every dyadic interval of the form $\{0, 1/2^h, 2/2^h, \dots, (2^h-1)/2^h, 1\}$.  In addition, we find it convenient to further define several other events, stated as follows.

\begin{definition} \label{def:Definition2}
    We define the event $\mathcal{M} \triangleq  \mathcal{M}_1 \cap \mathcal{M}_2 \cap \mathcal{M}_3 \cap \mathcal{M}_4 $, where 
    \begin{align}
    &\hspace{0mm}\mathcal{M}_1 \triangleq \bigcap_{h=0}^\infty \bigcap_{k=0}^{2^h-1} \left\{ \left\lvert W_{\frac{k+1}{2^h}} - W_{\frac{k}{2^h}} \right\rvert \le \alpha_\delta\big(2^{-h}\big) \right\},& \\
    &\mathcal{M}_2 \triangleq \bigcap_{h=0}^\infty \bigcap_{k \in \mathcal{J}_h} \left\{ \left\lvert \bar{y}_{\frac{k}{2^h}} - W_{\frac{k}{2^h}} \right\rvert \le \alpha_\delta\big(2^{-h}\big) \right\},\\
    &\hspace{0mm}\mathcal{M}_3 \triangleq \bigcap_{h=0}^\infty \bigcap_{k \in \mathcal{J}_h} \left\{ \max_{x \in I_{h,k}} W_x \le \mbox{\em UCB}_h\left(I_{h,k}\right) \right\}, \\
    & \mathcal{M}_4 \triangleq \bigcap_{h=0}^\infty \bigcap_{k \in \mathcal{J}_h} \left\{ \min_{x \in I_{h,k}} W_x \ge \mbox{\em LCB}_h\left(I_{h,k}\right) \right\}.
    \end{align}
\end{definition}

The event $\mathcal{M}_1$ represents the fact that the BM evaluated at successive points in the set of dyadic rationals $\{0, 1/2^h, 2/2^h, \dots, (2^h-1)/2^h, 1\}$ yields a difference of at most $\alpha_\delta(2^{-h})$; $\mathcal{M}_2$  represents the effect of averaging out the noise in the observations; and $\mathcal{M}_3$ and $\mathcal{M}_4$ are analogous to event $\mathcal{C}$, and characterize a form of H\"older continuity restricted to the dyadic partition.  Following \cite{grill2018optimistic}, we henceforth refer to these as {\em proxy-H\"older} conditions on $W$.  As stated in \alggref{NoisyOOBAlgo}, $\mathcal{J}_h$ is the set of ends of intervals in the set $\mathcal{I}$ at the $h$-th epoch. The events depend on $\delta$ through the evaluations of $\alpha_\delta$ and $\eta_\delta$.

In the following, we present some preliminary results to upper bound the regret. The first lemma states a standard high probability upper bound for a normal distribution.

\begin{lemma} \label{NormalUpperBoundLemma}
	For any $\delta > 0$, a standard Gaussian random variable is upper bounded by $\sqrt{2\ln(1/\delta)}$ with probability at least $1 - \delta$.
\end{lemma}
\begin{proof}
This high-probability upper bound follows from the standard (Chernoff) bound $Q(\alpha) \le e^{-\alpha^2/2}$, where $Q(\alpha)$ is the complementary CDF of an $\mathcal{N}(0,1)$ random variable.
\end{proof}

Next, we provide a formal statement that $\mathcal{M}$ holds with high probability.  The proof is given in Section~\ref{EventM_Proof}, with the main step being to generalize the analysis of the proxy-H\"older event $\mathcal{C}$ defined in \eqref{eqn:eventC} to its noisy variant, $\mathcal{M}_3\cap \mathcal{M}_4$.

\begin{lemma}\label{EventMLemma}
	For any $\delta\in (0, \frac{1}{3})$, we have $\mathbb{P}[\mathcal{M}] \ge 1 - \delta^2$, where $\mathcal{M}$ implicitly depends on $\delta$.
\end{lemma}

We now seek to demonstrate the exponential shrinkage of the upper bounds for regret as the epoch number increases. Lemma~\ref{ICMLEliminateSuboptimalLemma} is a standard result used in the study of algorithms that eliminate sub-optimal points based on confidence bounds, though our confidence bounds are defined on intervals instead of specific points in $D$.

\begin{lemma} \label{ICMLEliminateSuboptimalLemma}
    Fix $\kappa > 0$, and assume that at time $t$, for all intervals $[a,b]$ within some family of intervals $\mathcal{L}$, it holds for all $x\in[a,b]$ that $\text{\em LB}_h([a,b]) \le W_{x} \le \text{\em UB}_h([a,b])$ for some bounds $\text{\em UB}_h$ and $\text{\em LB}_h$ satisfying\footnote{We will only apply this result in the case that the maximum exists, but more generally a supremum could be used.} $\max_{[a,b] \in \mathcal{L}}\lvert \text{\em UB}_h([a,b]) - \text{\em LB}_h([a,b]) \rvert \le 2\kappa $. Then all $ [a,b] \in \mathcal{L}$ containing a $4\kappa$-suboptimal point $x$ (i.e.,  $W_x < \max_{I \in \mathcal{L}} \left( \max_{x^\prime \in I} W_{x^\prime}\right) - 4\kappa$)  must also satisfy the following:
    \begin{align}
        \text{\em UB}_h([a,b]) < \max_{I \in \mathcal{L}}\text{\em LB}_h(I).
    \end{align}
\end{lemma}
\begin{proof} We have 
    \begin{align}
        \text{UB}_t([a,b]) 
            & \le \text{LB}_t([a,b]) + 2\kappa \label{sub_step1} \\
            & \le W_x + 2\kappa \label{sub_step2} \\
            &< \max_{I \in \mathcal{L}} \left( \max_{x^\prime \in I} W_{x^\prime} \right) - 2\kappa \label{sub_step3} \\
            &\le \max_{I \in \mathcal{L}} \text{UB}_t(I) - 2\kappa \label{sub_step4} \\
            & \le \max_{I \in \mathcal{L}} \text{LB}_t(I), \label{sub_step5}
    \end{align}
    where \eqref{sub_step1} and \eqref{sub_step5} use the assumption of $2\kappa$-separation, \eqref{sub_step2} and \eqref{sub_step4} use the assumed validity of the confidence bounds, and \eqref{sub_step3} uses the assumption of $4\kappa$-suboptimality.
\end{proof}

The following lemma expresses the upper bound on the regret of a single point as an exponentially decreasing function of the epoch number $h$, and is proved via~\lemref{ICMLEliminateSuboptimalLemma}. 

\begin{lemma} \label{SimpleRegretUCBMinusLCBLemma}
Conditioned on event $\mathcal{M}$, each point $x_t$ sampled in the $h$-th epoch satisfies $f(x^*) - f(x_t) \le 4\kappa_h$, where $\kappa_{h} = \frac{5}{2}\alpha_\delta(2^{-h}) + \eta_\delta(2^{-h})$.
\end{lemma}
\begin{proof}
For each $h\in\mathbb{N}$ and $ k \in \{0, 1, \dots, 2^h-1\}$, we have
\begin{align}
    &\hspace{1mm}\nonumber \text{UCB}_h\left(\left[\frac{k}{2^h},\frac{k+1}{2^h}\right]\right) - \text{LCB}_h\left(\left[\frac{k}{2^h},\frac{k+1}{2^h}\right]\right)  \\
    &= \left\lvert \bar{y}_{\frac{k}{2^h}}^{(t)} - \bar{y}_{\frac{k+1}{2^h}}^{(t)} \right\rvert +  2\alpha_\delta(2^{-h}) + 2\eta_\delta(2^{-h}) \label{kap_step1} \\
\begin{split}
    &\le \left\lvert \bar{y}_{\frac{k}{2^h}}^{(t)} - W_{\frac{k}{2^h}} \right\rvert + \left\lvert \bar{y}_{\frac{k+1}{2^h}}^{(t)} - W_{\frac{k+1}{2^h}} \right\rvert + \left\lvert W_{\frac{k}{2^h}} - W_{\frac{k+1}{2^h}} \right\rvert \\
    &\qquad+  2\alpha_\delta(2^{-h}) + 2\eta_\delta(2^{-h}) \label{kap_step2} 
\end{split} \\
\label{EventM12}
    &\le 2\left[\frac{5}{2}\alpha_\delta(2^{-h}) + \eta_\delta(2^{-h})\right] = 2\kappa_{h}
\end{align}
where \eqref{kap_step1} uses the definitions of $\text{UCB}_h$ and $\text{LCB}_h$ in \eqref{eq:UCBh}--\eqref{eq:LCBh}, \eqref{kap_step2} follows from the triangle inequality, and \eqref{EventM12} follows from the definitions of $\mathcal{M}_1$ and $\mathcal{M}_2$ in the definition of~$\mathcal{M}$.

Since $\mathcal{M}$ implies that $W_x$ is sandwiched between UCB$_h$ and LCB$_h$, \lemref{ICMLEliminateSuboptimalLemma} implies that the regret for each point $x_t$ sampled in the $h$-th epoch is at most $4\kappa_{h}$.
\end{proof}

We restate a lemma from \cite{grill2018optimistic} regarding the expected number of near-optimal points.  An {\em $\eta$-near-optimal point} $x$ is one with value $W_x$ that is $\eta$-close to $M := \max_{x \in [0,1]} W_x$. 

\begin{definition} \label{def:near_opt}
Let $\mathcal{N}_h(\eta)$ denote the number of $\eta$-near-optimal points among $\{0, 2^{-h}, \dots, 1\}$:
    \begin{align}\mathcal{N}_h(\eta) \triangleq \left\lvert \left\{ k \in \{ 0, 1, \dots, 2^h \} ~:~ W_{k/2^h} \ge M - \eta \right\}\right\rvert.\end{align}
\end{definition}

\begin{lemma}[{\cite[Lemma 3]{grill2018optimistic}}] \label{NearOptimalLemma}
The expected number of $\eta$-near-optimal points $\mathcal{N}_h(\eta)$ in a $2^{-h}$-spaced grid in $[0,1]$ is upper bounded as follows:
\begin{align}\mathbb{E}\left[\mathcal{N}_h(\eta)\right] \le 6\eta^22^h. \label{NearOptBound} \end{align}
\end{lemma}

\vspace*{-1ex}
\subsection{Completion of the Proof of \thmref{CumulativeRegretThm}}

Let $r_{(h)}$ be the maximum instantaneous regret incurred in the $h$-th epoch, and recall that $T_h$ is the number of queries during the $h$-th epoch. Both are random variables, and we first analyze their behavior under event $\mathcal{M}$.
By \lemref{SimpleRegretUCBMinusLCBLemma}, event $\mathcal{M}$ implies that
\begin{equation}
    r_{(h)} \le 4\kappa_{h}, \label{four_kappa}
\end{equation}
and in addition, the definition of $\kappa_h$ therein yields
\begin{align}
\kappa_{h} 
&= \frac{c_1^{\prime\prime}}{\sqrt{2^h}} \left[ \sqrt{\ln\left(\frac{2^{6h}}{2\pi \delta^6}\right)} + \sqrt{\ln\left(\frac{2^h}{\delta}\right)} \right] \label{SimpleRegretUpperBound} \\
&\le c_1 2^{-\frac{h}{2}}\left(2\sqrt{\ln (1/\delta)}+\sqrt{h}\right) \label{SimpleRegretUpperBound3} \\
&= c_1 2^{-\frac{h}{2}}\left(\sqrt{\ln T}+\sqrt{h}\right), \label{SimpleRegretUpperBound4}
\end{align}
where~\eqref{SimpleRegretUpperBound} uses the definitions of $\alpha_\delta$ and $\eta_\delta$ in~\eqref{eqn:eta_alpha}, and \eqref{SimpleRegretUpperBound4} follows from the choice of $\delta = T^{-1}$.  By an analogous argument, we also have the lower bound
\begin{equation}
\kappa_{h} \ge c_1^{\dagger} 2^{-\frac{h}{2}}\left(\sqrt{\ln T}+\sqrt{h}\right) \label{kappa_lower}
\end{equation}
for some $c^{\dagger}_1 > 0$.

Next, we introduce a quantity $h'$ that can be viewed as approximately characterizing the average total number of epochs.  In accordance with \eqref{four_kappa}, it is natural to consider the choice $\eta = 4\kappa_h$ in Lemma \ref{NearOptimalLemma}, which leads to a right-hand side of $96\kappa_{h}^2 2^{h}$ in \eqref{NearOptBound}.  Since each point is sampled at least $\sigma^2 2^{h}$ times in Algorithm \ref{alg_ub}, and the total number of samples in a given epoch can never exceed $T$, we should expect that the epoch index never exceeds the following:
\begin{equation}
h^\prime = \max\left\{ h \,:\, 96\sigma^2\kappa_{h}^22^{2h} < T\right\}. \label{hprime}
\end{equation}
We emphasize that we do not require a formal claim relating the number of epochs to $h'$, but we still find the preceding intuition useful.

We claim that there exist positive constants $c_2^{\dagger}$, $c_2^{\dagger\dagger}$, and $c_2^{\natural}$ such that the following bounds hold:
\begin{align}
    c_2^{\dagger} \frac{T}{\sigma^2 \ln T} &\le 2^{h'} \le c_2^{\dagger\dagger} \frac{T}{\sigma^2 \ln T}, \label{eq:Claim_h1} \\
    h' & \le c_2^{\natural}  \ln T. \label{eq:Claim_h2}
\end{align}
To see this, we consider the choice of $h$ such that the lower bound in \eqref{eq:Claim_h1} holds with equality, and study the left-hand term in \eqref{hprime} as follows:
\begin{align}
    96\sigma^2\kappa_{h}^2 2^{2h} 
        &\le 96 \sigma^2 c_1^2 2^{h}\left(\sqrt{\ln T}+\sqrt{h}\right)^2 \label{eq:h_ana_1} \\
        &= 96 c_1^2 \frac{c_2^{\dagger} T}{\sigma^2 \ln T}\left(\sqrt{\ln T}+\sqrt{h}\right)^2, \label{eq:h_ana_2}
\end{align}
where \eqref{eq:h_ana_1} uses \eqref{SimpleRegretUpperBound4}, and \eqref{eq:h_ana_2} substitutes the left-hand side of \eqref{eq:Claim_h1}.  To eliminate the remaining $\sqrt{h}$ term, we recall the assumption $\sigma^2 \ge \frac{c_{\sigma}}{T^{1-\zeta}}$ in \thmref{CumulativeRegretThm}, and note that taking the logarithm (base 2) of $2^h = c_2^{\dagger} \frac{T}{\sigma^2 \ln T}$ gives $h = O(\ln T + \ln c_2^{\dagger})$.  Hence, \eqref{eq:h_ana_2} gives for sufficiently small $c_2^{\dagger}$ that $96\sigma^2\kappa_{h}^2 2^{2h} < T$, and hence, this choice of $h$ is indeed a lower bound on $h'$, as desired.  The upper bound on $h'$ follows from a near-identical argument, and we notice that \eqref{eq:Claim_h2} was already established as an intermediate step.


Let $\mathbb{E}_\mathcal{M}[\cdot]= \mathbb{E}[\cdot| \mathcal{M}]$ be the conditional expectation given the event $\mathcal{M}$. We now provide an upper bound on the conditional expected cumulative regret as follows, starting with~\eqref{four_kappa} (recall also $n_h$ defined in Algorithm \ref{alg_ub}):
{\allowdisplaybreaks
\begin{align}
&\mathbb{E}_\mathcal{M}[R_T] \nonumber \\
&\le \sum_{h=1}^{h_{\max}} 4\kappa_{h}  \mathbb{E}_\mathcal{M}[t_h] \\
\label{NumberSampleHEpochUpperBound}
&\le \sum_{h=1}^{\infty} 4\kappa_{h} \min\left(n_{h}\mathbb{E}_\mathcal{M}[\lvert \mathcal{J}_{h} \rvert], T\right)\\
\label{SimpleRegretBoundCumulativeUse}
&\le \sum_{h=1}^{\infty} 4\kappa_{h} \min\left(n_{h}\mathbb{E}_{\mathcal{M}}[\mathcal{N}_{h+1}(4\kappa_{h})], T\right) \\
\label{NearOptimalPointsNumber}
&\le \sum_{h=1}^{\infty} 4\kappa_{h} \min\left(\frac{96}{1-\delta^2}\kappa_{h}^2 2^{h+1}(\sigma^22^{h+1}+1), T\right) \\
\label{ApplyHprime}
&\le \sum_{h=1}^{h^\prime} \frac{384}{1-\delta^2}\kappa_{h}^3 2^{h+1}(\sigma^22^{h+1}+1) + \sum_{h = h^\prime+1}^{\infty} 4\kappa_{h} T \\
\label{SimpleRegretUpperBoundUse}
\begin{split}
&\le c_2^{\prime\prime} \sum_{h=1}^{h^\prime} \big( \sigma^2 2^{\frac{h}{2}}  + 2^{-\frac{h}{2}}\big)\big(\sqrt{\ln T}+\sqrt{h}\big)^3 \\
&\qquad+ c_1 \sum_{h=h^\prime+1}^{\infty}  T 2^{-\frac{h}{2}}\big(\sqrt{\ln T}+\sqrt{h}\big)
\end{split} \\
\label{ExponentialDecreasingSummation}
\begin{split}
&\le c_2^\prime \bigg[ \sigma^2\max(\ln T, h^\prime)^{\frac{3}{2}} 2^{\frac{h^\prime}{2}} + (\ln T)^{\frac{3}{2}} \\
&\qquad+ T 2^{-\frac{h^\prime}{2}}\sqrt{\max(\ln T, h^\prime)} \bigg] 
\end{split} \\
\label{HPrimeDefinition}
&\le c_2 \sigma\sqrt{T} \ln T
\end{align}}
where:
\begin{itemize}
    \item \eqref{NumberSampleHEpochUpperBound} follows from $t_h \le T$ and the fact that each point in $\mathcal{J}_{h}$ is sampled $n_h$ times;
    \item \eqref{SimpleRegretBoundCumulativeUse} follows from \lemref{SimpleRegretUCBMinusLCBLemma};
    \item \eqref{NearOptimalPointsNumber} follows from \lemref{NearOptimalLemma} and $n_h = \lceil \sigma^22^{h+1} \rceil$, as well as the fact that for any random variable $A$, $\mathbb{E}_{\mathcal{M}}[A] = \frac{ \mathbb{E}[A \boldsymbol{1}\{\mathcal{M}\}] }{ \mathbb{P}[\mathcal{M}] } \le \frac{ \mathbb{E}[A] }{ 1-\delta^2 }$ (recall from Lemma \ref{EventMLemma} that $\mathbb{P}[\mathcal{M}] \ge 1-\delta^2$);
    \item \eqref{ApplyHprime} follows by upper bounding the minimum by either of its two arguments (choosing which one differently depending on the summation index);
    \item \eqref{SimpleRegretUpperBoundUse} follows from \eqref{SimpleRegretUpperBound4};
    \item \eqref{ExponentialDecreasingSummation} follows from the fact that an exponentially decreasing (resp., increasing) series is bounded above by a constant multiple of its first (resp., last) term -- the first two terms come from splitting $\big( \sigma^2 2^{\frac{h}{2}}  + 2^{-\frac{h}{2}}\big)$ from \eqref{SimpleRegretUpperBoundUse} to form two sums, the first of which is exponentially increasing, and the second of which is exponentially decreasing;
    \item \eqref{HPrimeDefinition} follows by substituting the bounds on $h'$ from \eqref{eq:Claim_h1}--\eqref{eq:Claim_h2}, and also noting that the $(\ln T)^{\frac{3}{2}}$ term is insignificant compared to $\sigma\sqrt{T} \ln T$ due to the assumption $\sigma^2 \ge \frac{c_{\sigma}}{T^{1-\zeta}}$ in \thmref{CumulativeRegretThm}.
\end{itemize}

We can now upper bound the unconditional expectation of $R_T$ via the law of total expectation as follows:
{\allowdisplaybreaks
\begin{align}
&\hspace{-2mm}\mathbb{E}[R_T] = \mathbb{E}_\mathcal{M}[R_T]\mathbb{P}[\mathcal{M}] + \mathbb{E}_{\mathcal{M}^c}[R_T]\mathbb{P}[\mathcal{M}^c]\\
\label{TrivialLinearRegretBound}
&\hspace{8mm}\le c_2 \sigma \sqrt{T} \ln T + \mathcal{O}\big(T \sqrt{\log(1/\delta)} \delta^2\big)\\
\label{DeltaTRelationEq}
&\hspace{8mm} = \mathcal{O}\big(\sigma \sqrt{T} \log T\big),
\end{align}
where \eqref{TrivialLinearRegretBound} is established by showing that $\mathbb{E}[R_T \mid \mathcal{A}] = \mathcal{O}\big( T \sqrt{\log \frac{1}{\mathbb{P}[\mathcal{A}]}} \big)$ for any event $\mathcal{A}$ (see Section~\ref{sec:ProofTrivLinear2} for details) and applying $\mathbb{P}[\mathcal{M}^c] \le \delta^2$ (see Lemma \ref{EventMLemma}), 
and \eqref{DeltaTRelationEq} follows from the fact that $\delta = T^{-1}$ and the assumption $\sigma^2 \ge \frac{c_{\sigma}}{T^{1-\zeta}}$.  This yields the first part of Theorem \ref{CumulativeRegretThm}.}

In addition, by the choice of $x^{(T)}$ in the last line in Algorithm \ref{alg_ub}, the upper bound on the expected simple regret trivially follows:
\begin{align}
    \mathbb{E}[\rT] \le \frac{\mathbb{E}[R_T]}{T} = \mathcal{O}\big(\sigma T^{-\frac{1}{2}} \log T \big), \label{cumul_to_simple}
\end{align}
which yields the second part of \thmref{CumulativeRegretThm}.

\vspace*{-1ex}
\section{Lower Bounds}\label{sec:conv}

In this section, we establish algorithm-independent lower bounds on the regret.  The idea of the proof is to reduce the optimization problem into a binary hypothesis testing problem, while confining attention to ``typical'' realizations of a Brownian motion by conditioning on realizations satisfying suitable high-probability properties (e.g., a proxy-H\"older type condition).  Once the reduction to hypothesis testing is done, a lower bound is deduced via Fano's inequality.  As hinted earlier, this proof strategy builds on that of \cite{scarlett2018tight} (and in turn \cite{raginsky2011information}), but with very different details due to the consideration of BM instead of smooth and stationary functions.

We begin by formally stating our lower bounds.

\newcommand{\tilcp}{\tilde{c}'}
\begin{theorem} \label{LowerBoundCumulativeNoisyTheorem}
For the problem of BM optimization with a noise variance $\sigma^2 > 0$ satisfying $\sigma^2 \le \tilcp T^{1-\zeta'}$ for some positive constants $(\tilcp,\zeta')$,\footnote{This is a very mild assumption, since we are primarily interested in the case that $\sigma$ is constant with respect to $T$.} 
any algorithm must have
\begin{align}
\mathbb{E}[\rT] &= \Omega\big( \sigma (T \log T)^{-\frac{1}{2}} \big), \quad\mbox{and}\quad \\ \mathbb{E}[R_T] &= \Omega\big( \sigma (T / \log T)^{\frac{1}{2}} \big).
\end{align}
\end{theorem}

The proof is given in the remainder of the section, with several details deferred to Section \ref{app:FanoProof}.

\subsection{Reduction to Binary Hypothesis Testing}


We fix $\Delta>0$, and view the BM $W$ on $D=[0,1]$ as being generated by the following procedure:
\begin{enumerate}
	\item Generate a BM $\tilde{W}$ on the larger domain $[-\Delta, 1 + \Delta]$, with $\tilde{W}_{-\Delta} = 0$ (i.e., $\tilde{W}$ is shifted left by $\Delta$ compared to a standard BM).
	\item Randomly draw $V \in \{+,-\}$ with probability $\frac{1}{2}$ each, and perform one of the following steps to generate $\tilde{W}^\prime$: 
	\begin{itemize}
		\item If $V$ is `$+$', then shift $\tilde{W}$ left along the $x$-axis by $\Delta$, and add a constant term of $-\tilde{W}_\Delta$.
		\item If $V$ is `$-$', then shift $\tilde{W}$ right along the $x$-axis by~$\Delta$.
	\end{itemize}
    The two corresponding functions are written as
	\begin{align}
		\label{DefW+}
		W_x^+ &= \tilde{W}_{x+\Delta} - \tilde{W}_\Delta, \\*
		\label{DefW-}
		W_x^- &= \tilde{W}_{x-\Delta},
	\end{align}
    and we observe that $W_0^+ = W_0^- = 0$, and by the Markov property, both $W_0^+$ and $W_0^-$ are standard BM processes when restricted to the domain $[0,1]$. We allow $\Delta$ to vary with $T$, and will in fact eventually set $\Delta = \mathcal{O}(1/T)$.
	\item Let $W_x = \tilde{W}'_x$ for $x \in [0,1]$.  Since $W_x^+$ and $W_x^-$ are both standard BM, we have that $W_x$ is a standard BM conditioned on either value of $V$, and thus, it is also a standard BM unconditionally, as desired.
\end{enumerate}

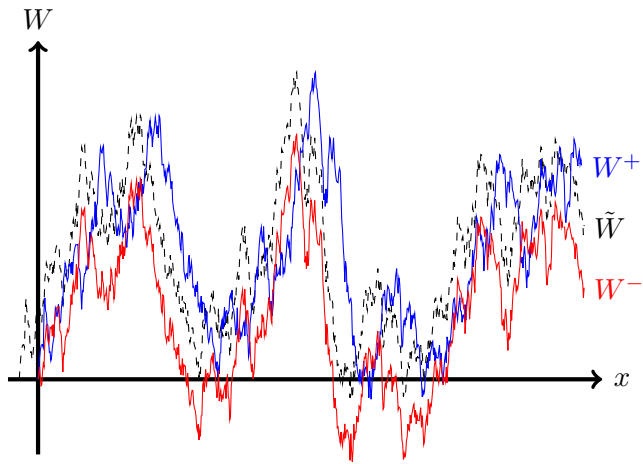
\begin{figure}
\centering
\begin{tikzpicture}
\draw[->,ultra thick] (-0.4,0)--(7.5,0) node[right]{$x$};
\draw[->,ultra thick] (0,-1.0)--(0,4.5) node[above]{$W$};
\Emmetttt{-0.125,0}{0.01}{0.15}{black,dashed}{}
\Emmettt{0.125,0.8325}{0.01}{0.15}{black,dashed}{}
\Emmettttt{7.125,2.865}{0.01}{0.15}{black,dashed}{$\tilde{W}$}
\Emmetttt{0,0}{0.01}{0.15}{blue}{}
\Emmettt{0.25,0.8325}{0.01}{0.15}{blue}{$W^-$}
\Emmettt{0,0}{0.01}{0.15}{red}{}
\Emmettttt{7.00,2.0325}{0.01}{0.15}{red}{$W^+$}
\end{tikzpicture}
\caption{Illustration of $\tilde{W}, W^+$ and $W^-$}
\label{fig:shifted_BM}
\end{figure}

We consider a ``genie-aided'' argument in which $\tilde{W}$ is revealed to the algorithm but the direction of the shift (i.e., the value of $V$) is unknown. Clearly this additional information provided by the genie can only help   the algorithm, so any lower bound still remains valid for the original setting. Stated differently, the algorithm knows that $W$ is either $W^+$ or $W^-$.
This argument allows us to reduce the BO problem to a binary hypothesis test with adaptive sampling. The hypothesis, indexed by $v \in \{+, -\}$, is that the underlying function is $W^v$.

We define the maximizer of $\tilde{W}$ as
    \begin{align} \xM = \argmax_{x \in [-\Delta, 1+\Delta]} \tilde{W}_x \end{align}
(which is almost surely unique), and let
\begin{equation}
    \label{x+*x-*}
	\xM^+ = \xM - \Delta, \quad\mbox{and}\quad \xM^- = \xM + \Delta.
\end{equation}
Then, the maxima  of the processes $W_x^-$ and $W_x^+$  are respectively defined as
\begin{align}
	M^- &:=  W^-_{\xM^-} = \tilde{W}_{\xM} = M, \\
	M^+ &:= W^+_{\xM^+} = M - \tilde{W}_\Delta, \label{M+}
\end{align}
and we define the following functions that, up to a caveat discussed below, represent the (simple and cumulative) regret with respect to $W^+$ and $W^-$:
\begin{alignat}{3}
	r^+(x)& = M^+ - W^+_x,  \qquad &r^-(x) &= M^- - W^-_x, \\
	R_T^+ &= \sum_{t=1}^T r^+(x_t), \qquad  & R_T^-& = \sum_{t=1}^T r^-(x_t).
\end{alignat}
It is important to note that $\xM^+$ and $\xM^-$ in \eqref{x+*x-*} could, in principle, lie outside the domain $[0,1]$ (namely, when $x_M < \Delta$ or $x_M > 1-\Delta$), in which case it may hold that the simple regret $r_T$ satisfies $r_T < r_T^v$ for some $v \in \{+,-\}$.  However, in our analysis, we will condition on a high-probability event (see Definition \ref{def:setT}) that ensures $\xM^+,\xM^-\in[0,1]$, and conditioned on this event we have $r_T = r_T^V$.  Similar observations apply for the cumulative regret.

\vspace*{-1ex}
\subsection{Auxiliary Lemmas} \label{sec:aux_lemmas}

We first state some useful properties of a BM. The {\em Brownian meander}~\cite{durrett1977weak} plays an important role in our analysis, as it characterizes the distribution of function values to the left and right of the maximum of $W$ (see \lemref{BrownianMeanderIdenticalDistributionLemma} in Appendix~\ref{app:results}). Formally, given a standard BM $(W_x)_{x \in [0,1]}$, we define a Brownian meander as the process  $(W_x \mid \min_{x' \in [0,1]} W_{x'} \ge 0)$ for $x \in [0,1]$.  Note that although the event $\{\min_{x' \in [0,1]} W_{x'} \ge 0\}$ has probability zero, this conditional distribution is known to remain well-defined \cite{durrett1977weak}. 
    

The following lemmas characterize the distribution of the running maximum or minimum of a Brownian meander; the proofs are given in Sections \ref{ProofRunMax} and \ref{ProofRunMin}.

\begin{lemma}\label{ConditionalRunningMaxDistributionComplementVariableLengthLemma}
For a standard BM $W$, for any $0 < s < t$ and $ 0 \le x < \frac{\sqrt{s}}{2}$, it holds that
	\begin{align}
	&\mathbb{P}\left[\max_{0 \le z \le s} W_z \ge x \,\Big|\,  \min_{0 < z \le t} W_z > 0, W_0 = 0 \right] \nonumber\\
	&\quad\ge\left\{\begin{array}{cc}
	1 - \frac{1}{2} \left( \frac{x}{\sqrt{s}}\right)^2 &\mbox{ \em if } t > 2s\\
	1 -  \frac{x\sqrt{2}}{\sqrt{s}} & \mbox{ \em if } t \le 2s.
	\end{array}\right.
	\end{align}
\end{lemma}

\begin{lemma} \label{RunningMinimumBrownianMeanderLemma}
For a BM $W$ with initial value $W_0 = u$, we have for any $0 < \varepsilon < u$ that
\begin{align}
\mathbb{P}\left[ \min_{0\le z \le t} W_z > \varepsilon \,\Big|\,  \min_{0\le z \le t} W_z > 0, W_0 = u \right] \ge \frac{u - \varepsilon}{u} .
\end{align}
\end{lemma}

We define a high-probability event to restrict the position of the maximum to the interval $(2\Delta,1-2\Delta)$, and to restrict the two regret functions $r^+$ and $r^-$ to be simultaneously lower than a certain function of $\Delta$.  Note that the constant $\delta > 0$ in the following is not related to that appearing in Section \ref{sec:ach}.

\begin{definition} \label{def:setT}
Fix $\delta > 0$, and let $\mathcal{T}\triangleq \mathcal{T}_1 \cap \mathcal{T}_2 \cap \mathcal{T}_3$, where 
\begin{align}
\mathcal{T}_1 & \triangleq \left\{ 2\Delta < \xM < 1-2\Delta \right\} ,\\
\mathcal{T}_2 &\triangleq \big\{ \forall x \in D, \max\big( r^+(x),r^-(x)\big) \ge c_3\delta^2\sqrt{\Delta} \big\},\\
\mathcal{T}_3 &\triangleq \left\{ \forall x \in D, \left\lvert r^+(x) - r^-(x) \right\rvert \le c_4\sqrt{\Delta \ln (1/\Delta)} \right\}.
\end{align}
\end{definition}

As the following lemma is crucial, we outline the proof here, and provide the full details in Section \ref{EventT_Proof}.

\begin{lemma} \label{EventTLemma}
For $0 < \delta < 1$, and any $0 < \eta < \frac{1}{2}$ and sufficiently small $\Delta$, we have
\begin{align}
	\mathbb{P}[\mathcal{T}] \ge 1 - 3\Delta^\eta - \delta - \Delta,
\end{align}
where $\mathcal{T}$ implicitly depends on $\delta$.
\end{lemma}

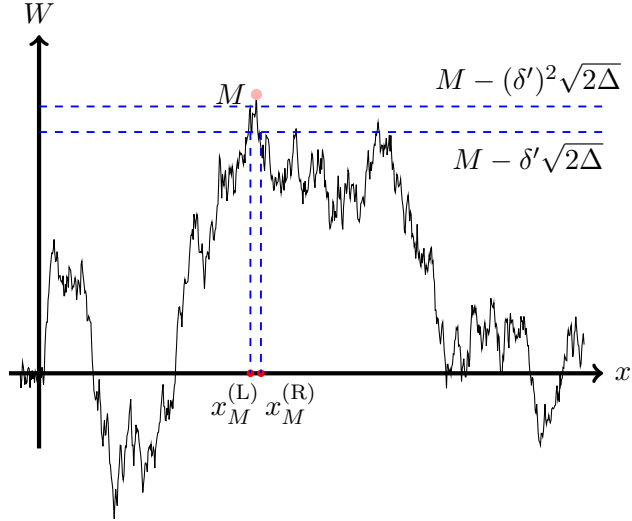
\begin{figure}
\centering
\begin{tikzpicture}
\draw[->,ultra thick] (-0.4,0)--(7.5,0) node[right]{$x$};
\draw[->,ultra thick] (0,-1.0)--(0,4.5) node[above]{$W$};
\BrownianSimpleIllustration{-0.25,0}{0.01}{-0.15}{black}{}
\tkzDefPoint(2.89,3.71){M}
\tkzDefPoint(0,3.55){MSlightDownLeft}
\tkzDefPoint(7.5,3.55){MSlightDownRight}
\tkzDefPoint(0,3.21){MDownLeft}
\tkzDefPoint(7.5,3.21){MDownRight}
\tkzDefPoint(2.955,3.21){MDownSlightRight}
\tkzDefPoint(2.955,0){MDownSlightRightBottom}
\tkzDefPoint(2.815,3.21){MDownSlightLeft}
\tkzDefPoint(2.815,0){MDownSlightLeftBottom}
\tkzLabelPoint[left](M){$M$}
\tkzLabelPoint[below, ,xshift=4mm](MDownSlightRightBottom){$\xM^{({\rm R})}$}
\tkzLabelPoint[below, ,xshift=-2mm](MDownSlightLeftBottom){$\xM^{(\mathrm{L})}$}
\tkzLabelPoint[above,xshift=-10mm](MSlightDownRight){$M - (\delta^\prime)^2 \sqrt{2\Delta}$}
\tkzLabelPoint[below,xshift=-10mm](MDownRight){$M - \delta^\prime \sqrt{2\Delta}$}
\node at (M)[circle,fill,inner sep=1.5pt, color=red!30]{};
\node at (MDownSlightRightBottom)[circle,fill,inner sep=1.0pt, color=red]{};
\node at (MDownSlightLeftBottom)[circle,fill,inner sep=1.0pt, color=red]{};
\draw[line width=0.25 mm, color=blue, dashed] (MDownLeft) -- (MDownRight);
\draw[line width=0.25 mm, color=blue, dashed] (MSlightDownLeft) -- (MSlightDownRight);
\draw[line width=0.25 mm, color=blue, dashed] (MDownSlightRight) -- (MDownSlightRightBottom);
\draw[line width=0.25 mm, color=blue, dashed] (MDownSlightLeft) -- (MDownSlightLeftBottom);
\end{tikzpicture}
\caption{Running maxima of the Brownian motion outside the neighborhood near its maximum}
\label{fig:meander}
\end{figure}

\begin{proof}[Proof Sketch of Lemma \ref{EventTLemma}]
    
The probability of $\mathcal{T}_1$ is characterized via a direct analysis of the probability of the maximum lying in the intervals $[0,2\Delta]$ and $[1-2\Delta,1]$.  In the limit of small $\Delta$, we show that the latter probability decays to zero as roughly $\tilde{O}(\Delta)$ (up to logarithmic factors), and hence, when $\Delta$ is sufficiently small, $\mathcal{T}_1$ holds with probability at least $1 - \Delta^{\eta}$.

For events $\mathcal{T}_2$ and $\mathcal{T}_3$, we consider the right-most time $\xM^{(\mathrm{L})}$ left of the maximizer $\xM$ to hit the value $M - \delta^\prime \sqrt{2\Delta}$, where $\delta^\prime$ is a constant multiple of $\delta$. With probability at least $1 - (\delta^\prime)^2$, $\xM^{(\mathrm{L})}$ falls within the $\Delta$-neighborhood of $\xM$ by \lemref{ConditionalRunningMaxDistributionComplementVariableLengthLemma}, a result concerning the running maximum of a Brownian meander.
%
In addition, we can lower bound the probability of the running maximum from $0$ to $\xM^{(\mathrm{L})}$ exceeding $M - (\delta^\prime)^2 \sqrt{2\Delta}$ by $1 - \delta^\prime$ by \lemref{RunningMinimumBrownianMeanderLemma}, a result concerning the running minimum of a Brownian meander starting from a non-zero value to reach certain lower value.

As illustrated in  Figure \ref{fig:meander}, we define two more ``mirror events'' of these two on the right of $\xM$. When $\mathcal{T}_1$ and these four events simultaneously hold, we have that all values larger than $M - (\delta^\prime)^2 \sqrt{2\Delta}$ must be in the $\Delta$-neighborhood of $\xM$ for either of $W^+$ or $W^-$. Furthermore, with the horizontal shift of $2\Delta$, the $\Delta$-neighborhoods near the maximum of the two functions do not overlap, and the shift would not result in a maximum outside $[0,1]$. By taking a union bound on all of these events, $\mathcal{T}_1 \cap \mathcal{T}_2$ holds with probability at least $1 - 3\Delta^\eta - \delta$ when we let $c_3\delta^2\sqrt{\Delta} = (\delta^\prime)^2 \sqrt{2\Delta}$, where $c_3 = 0.01\sqrt{2}$.

The analysis of event $\mathcal{T}_3$ uses similar ideas to the analysis of event $\mathcal{M}$ in Lemma \ref{EventMLemma}.  Recall that $r^+$ and $r^-$ are shifted versions of each other, and the shift along the horizontal axis is $2\Delta$.  A standard proxy-H\"older continuity argument (e.g., see~\cite{grill2018optimistic}) can be used to establish that all points separated by $2\Delta$ have corresponding $\tilde{W}_x$ values differing by $\mathcal{O}\big( \sqrt{\Delta \ln(1/\Delta)} \big)$ (uniformly on the domain $D$) with high probability; a more quantitative version of this argument yields $\mathbb{P}[\mathcal{T}_3] \ge 1-\Delta$.
\end{proof}

To establish lower bounds for the expected cumulative regret and simple regret, we make use of Fano's inequality \cite{scarlett2019fano}, which naturally introduces the mutual information between the hypothesis $V$ and the selected points and observations $(\mathbf{x},\mathbf{y})$, where $\mathbf{x}=(x_1,\ldots, x_T)$ and $\mathbf{y}=(y_1,\ldots , y_T)$ with $\mathbf{x}\in [0,1]^T$ and $\mathbf{y}\in \mathbb{R}^T$. In the following, we will also condition on $\{\tilde{W} = \tilde{w}\}$, where $\tilde{w}$ is a specific realization of $\tilde{W}$ satisfying the conditions in $\mathcal{T}$ (Definition~\ref{def:setT}). For brevity, this conditioning is indicated as a subscript $\tilde{w}$. For example, we write $I_{\tilde{w}}(V;\mathbf{x},\mathbf{y})$ as a shorthand for the conditional mutual information $I(V;\mathbf{x},\mathbf{y} | \tilde{W} = \tilde{w})$.  Recall that $V$ is assumed to be equiprobable on $\{+,-\}$.

\begin{lemma} \label{FanoBasedLemma}
	Under the preceding setup with $\tilde{w}$ satisfying the conditions in $\mathcal{T}$, we have
	\begin{align}
		\label{IntermediateInequality}
		\mathbb{E}_{\tilde{w}}[\rT ] \ge c_3 \delta^2 \sqrt{\Delta} H_2^{-1}\left( \ln 2 - I_{\tilde{w}}(V; \mathbf{x}, \mathbf{y} ) \right)
	\end{align}
	where $H_2^{-1}: [0, \log 2] \to [0, 1/2]$ is the functional inverse of the binary entropy function $H_2(\alpha) = \alpha \ln (1 / \alpha) + (1-\alpha) \ln(1/ (1-\alpha))$ in nats.
\end{lemma}
\begin{proof}
    The proof  mostly follows Raginsky and Rakhlin~\cite{raginsky2011information} and  Scarlett~\cite{scarlett2018tight} (and related earlier works on statistical estimation), and can be found in Section \ref{app:FanoProof}.
\end{proof}

\vspace*{-1ex}
\subsection{Completion of the Proof of \thmref{LowerBoundCumulativeNoisyTheorem}}

We upper bound mutual information $I_{\tilde{w}}(V; \mathbf{x}, \mathbf{y})$ conditioned on $\tilde{W} = \tilde{w}$ satisfying the conditions defining event~$\mathcal{T}$:
{\allowdisplaybreaks
\begin{align}
\label{Tensorization}
I_{\tilde{w}}(V; \mathbf{x}, \mathbf{y}) &\le \sum_{t=1}^T I_{\tilde{w}}(V; y_t | x_t) \\
\label{KLDivergence}
&\le \sum_{t=1}^T \max_{x \in [0,1]} \frac{\left( r^+(x) - r^-(x) \right)^2 }{2\sigma^2} \\
\label{RegretDeviation2}
&\le \frac{c_4^2 (\Delta \ln\frac{1}{\Delta}) T }{2\sigma^2},
\end{align}
where \eqref{Tensorization} follows from the tensorization property of mutual information \cite[Lemma 3]{scarlett2019fano}, \eqref{KLDivergence} follows from a standard calculation of relative entropy between Gaussian random variables (and replacing the average over $x_t$ by a maximum over $x \in [0,1]$), and \eqref{RegretDeviation2} follows from event $\mathcal{T}_3$ in Definition \ref{def:setT}.}

To ensure that $H_2^{-1}\left( \ln 2 - I_{\tilde{w}}(V; \mathbf{x}, \mathbf{y}) \right)$ is lower bounded by a positive constant, we choose $\Delta = c'_4 \frac{\sigma^2}{T \ln T}$ for some $c'_4 > 0$. The assumption $\sigma^2 \le \tilcp T^{1-\zeta'}$ in \thmref{LowerBoundCumulativeNoisyTheorem} ensures that $\ln\frac{1}{\Delta} = \Theta(\log T)$; hence, the $\ln\frac{1}{\Delta}$ and $\ln T$ terms cancel upon substitution into \eqref{RegretDeviation2}, and we obtain $I_{\tilde{w}}(V; \mathbf{x}, \mathbf{y}) \le \frac{1}{2}$ (say) for sufficiently small $c'_4$.
The resultant inequality of \eqref{IntermediateInequality} therefore gives
\begin{align}
	\label{UpperBoundConditionalExpectedSimpleRegretEq}
	&\mathbb{E}_{\tilde{w}}[\rT] \ge c_5 \sigma \delta^2 \sqrt{\frac{1}{T \ln T}}.
\end{align}
Upon averaging over all BM realizations, this implies
\begin{align}
	\label{BayesianFormulaEquation}
    \mathbb{E}[r_T]&= \mathbb{E}[r_T | \tilde{W}\in \Pi_T ] \Pr[ \tilde{W}\in \Pi_T ]  \nonumber\\*
    &\qquad+ \mathbb{E}[r_T | \tilde{W}\in \Pi_T^c ] \Pr[ \tilde{W}\in \Pi_T^c ] \\
    &\hspace{0mm}\ge \mathbb{E}[r_T | \tilde{W}\in \Pi_T ] \Pr[ \tilde{W}\in \Pi_T ] \\
	\label{TProbabilityLowerBound}
	&\hspace{0mm}\ge c_5 \sigma \delta^2 \sqrt{\frac{1}{T \ln T}} \mathbb{P}[\mathcal{T}] \\
	&\hspace{0mm}= \Omega\bigg( \sigma \sqrt{\frac{1}{T \log T}} \bigg), \label{Tfinal}
\end{align}
where in \eqref{BayesianFormulaEquation} we define $\Pi_{\mathcal{T}}$ to be the set of sample paths of the BM $\tilde{W}$ such that the induced $x_M, r^+(\cdot)$ and $r^-(\cdot)$  satisfy the events defining $\mathcal{T}$  in Definition~\ref{def:setT}, \eqref{TProbabilityLowerBound} follows from \eqref{UpperBoundConditionalExpectedSimpleRegretEq}, and~\eqref{Tfinal} follows from \lemref{EventTLemma} with $\delta$ chosen to a constant value (e.g., $\delta = 0.5$).  This yields the desired lower bound for the simple regret.

As shown in \eqref{cumul_to_simple}, for achievability results it is trivial to obtain $\mathbb{E}[\rT] \le \mathbb{E}[R_T] / T$.  The contrapositive statement is that for converse results, a universal lower bound of on $\mathbb{E}[\rT]$ implies the same universal lower bound for $\mathbb{E}[R_T] / T$, and it follows that $\mathbb{E}[R_T] = \Omega\big( \sigma \sqrt{ \frac{T}{ \log T} } \big)$,  as desired.

\vspace*{-1ex}
\section{Experiments}\label{sec:num}

 Theorems \ref{CumulativeRegretThm} and \ref{LowerBoundCumulativeNoisyTheorem} state that the expected cumulative regret achieved by \alggref{NoisyOOBAlgo} is at most $\mathcal{O}(\sigma T^{\frac{1}{2}} \log T)$ and at least $\Omega(\sigma T^{\frac{1}{2}} / \sqrt{\log T}  )$, so the growth rate is roughly $\sqrt{T}$. We corroborate these theoretical findings using numerical experiments. We implement and run  Algorithm~\ref{NoisyOOBAlgo}, 
 varying the time horizon $T$ from $10^5$ and $1.8\times 10^7$. The noise variance is set to $\sigma^2=0.25$. We generate $50$ independent realizations of a BM on $[0,1]$. For each realization, we run the algorithm $100$ times, each time corresponding to different realizations of the noise.   We calculate the mean and standard deviation of the cumulative regret (over all $5000$ runs) at each $T$, and plot the average $R_T / \sqrt{T}$ against $T$. These are shown in Figure~\ref{RTOverTAgainstT}; error bars indicate one standard deviation from the mean. 

\begin{figure}[t]
\centering
\includegraphics[width=0.95\columnwidth]{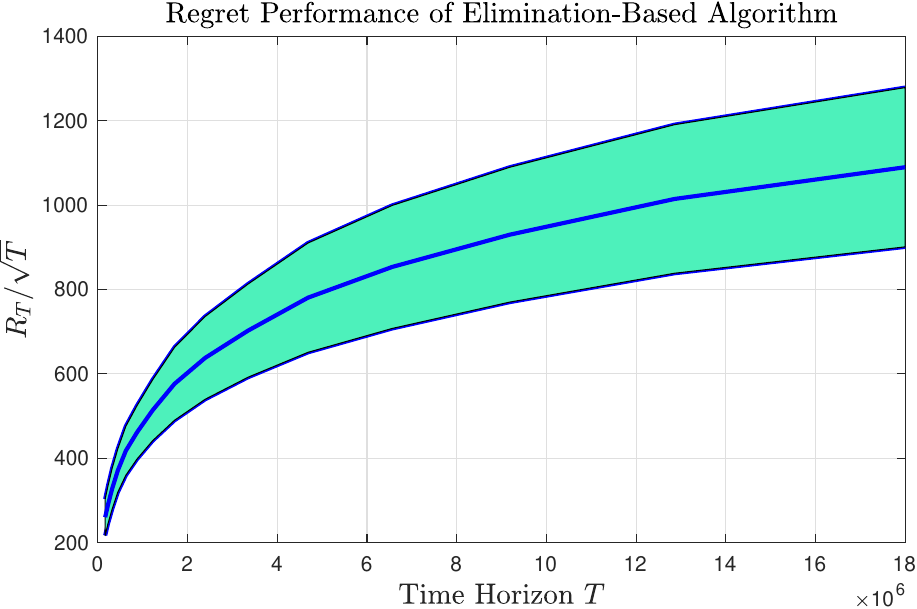}
\caption{Plot of ${R_T}/{\sqrt{T}}$ against $T$.} 
\label{RTOverTAgainstT}
\end{figure}

Figure~\ref{RTOverTAgainstT} shows that as $T\to\infty$, $R_T/\sqrt{T}$ appears to be bounded, or at least growing very slowly (e.g., due to logarithmic factors), as $T$ increases. This indicates that $R_T$ scales roughly as $\sqrt{T}$, corroborating our theoretical findings. 



 \vspace*{-1ex}
\section{Conclusion} \label{sec:concl}
 
 We have established upper and lower bounds on the smallest possible simple regret and cumulative regret for the noisy optimization of a Brownian motion. These bounds are tight up to a logarithmic factor in $T$. Our results complement the existing bounds on Bayesian optimization with smooth functions, revealing that in fact the cumulative regret enjoys similar scaling laws in the smooth and non-smooth scenarios.  
 
 In future work, it would be interesting to determine whether similar results hold for other non-smooth processes, such as the Ornstein--Uhlenbeck (OU) process or multi-dimensional variants of BM.  We expect that some of these (e.g. OU process with {\em known} parameters) should be possible with similar techniques to ours, by leveraging the solution of its accompanying stochastic differential equation. However, we also expect other stochastic processes is more challenging (e.g. processes with {\em unknown} parameters involving stochastic filtering) and possibly require new confidence bounds, H\"older-type continuity results, and so on. We leave such investigations to future work.

\vspace*{-1ex}
\section{Proofs of Auxiliary Lemmas} \label{app:proofs}

\subsection{Proof for \lemref{EventMLemma} (Probability of Event $\mathcal{M}$)} \label{EventM_Proof}

We start with the proxy-Lipschitz condition $\mathcal{C}$ of Grill {\em et al.}~\cite{grill2018optimistic}, in which we use $\delta$ as the argument for the probability of ``failure'', recall that the event $\mathcal{C}$ is defined as follows:

\begin{equation}
    \mathcal{C} \triangleq \bigcap_{h=0}^\infty \bigcap_{k=0}^{2^h-1} \left\{ \max_{x \in [\frac{k}{2^h}, \frac{k+1}{2^h}]} W_x \le B\left(\left[\frac{k}{2^h}, \frac{k+1}{2^h}\right]\right) \right\}, \label{setC}
\end{equation}
where $B([a,b]) = \max(W_a, W_b) + \eta_\delta(b-a)$, and $\eta_{\delta}(\cdot)$ is defined in \eqref{eqn:eta_alpha}.

According to \cite[Lemma 2]{grill2018optimistic}, $\mathcal{C}$ occurs with probability at least $1 - \delta^5$. We can also define an analogous event $\mathcal{C}^\prime$ (involving lower bounds instead of upper bounds), which similarly holds with probability at least $1 - \delta^5$, and is defined as follows:
\begin{align}
\mathcal{C}^\prime \triangleq \bigcap_{h=0}^\infty \bigcap_{k=0}^{2^h-1} \left\{ \min_{x \in [\frac{k}{2^h}, \frac{k+1}{2^h}]} W_x \ge B^\prime\left(\left[\frac{k}{2^h}, \frac{k+1}{2^h}\right]\right) \right\}, 
\end{align}
where $B^\prime([a,b]) = \min(W_a, W_b) - \eta_\delta(b-a)$.

We then define two sets of sub-events to prove that $\mathcal{M}_1$ and $\mathcal{M}_2$ hold with high probability: For each $h \ge 1$ and $k \in \{ 0, 1, \dots, 2^h-1\}$, define
\begin{align}
\mathcal{G}_{h,k} &\triangleq \left\{ \lvert W_\frac{k+1}{2^h} - W_\frac{k}{2^h} \rvert > \alpha_\delta(2^{-h})  \right\},
\end{align}
and for each $h \ge 1$ and $k \in \mathcal{J}_h$, define
\begin{align}
\mathcal{H}_{h,k} &\triangleq \Big\{ \big\lvert \bar{y}_\frac{k+1}{2^h}^{(h)} - W_\frac{k+1}{2^h} \big\rvert > \alpha_\delta(2^{-h}) \Big\},
\end{align}
where $\alpha_{\delta}(\cdot)$ is defined in \eqref{eqn:eta_alpha}, and we recall that $\bar{y}_x^{(h)} = {\rm Average}\left( \left\{ y_s \mid 1 \le s \le T_h, x_s = x \right\} \right)$.

Since for any $k \in \{ 0, 1, \dots, 2^h-1\}$, the differences $W_\frac{k+1}{2^h} - W_\frac{k}{2^h}$ are normally distributed with variance $2^{-h}$, we can upper bound them by 
\begin{align}
\alpha_\delta(2^{-h}) = \sqrt{2^{-h}\ln\left( \frac{1}{2\pi(2^{-h}\delta)^6}\right)}
\end{align}
with probability at least $1 - 2^{-2h}\delta^3$ according to \lemref{NormalUpperBoundLemma}.  This implies that each $\mathcal{G}_{h,k}$ holds with probability at most $2^{-2h}\delta^3$. Since at least $n_h = \sigma^2 2^h$ samples have been taken at each point in $\mathcal{J}_h$ (see Algorithm \ref{alg_ub}), the variance of the difference between the actual value and the average over noisy samples at these points is reduced to at most $2^{-h}$, and each $\mathcal{H}_{h,k}$ holds with probability at most $2^{-2h}\delta^3$ similarly to the argument for $\mathcal{G}_{h,k}$.

By the union bound on all the sub-events, and summing the probabilities with a geometric series, we conclude that $\mathbb{P}[\mathcal{M}_1] \ge 1 - \delta^3$ and $\mathbb{P}[\mathcal{M}_2] \ge 1 - \delta^3$. Note that $\mathcal{M}_2\cap\mathcal{C}\cap\mathcal{C}^\prime$ implies $\mathcal{M}_3$ and $\mathcal{M}_4$. Hence, the intersection of $\mathcal{M}_1$, $\mathcal{M}_2$, $\mathcal{C}$ and $\mathcal{C}^\prime$ implies $\mathcal{M}$, and by the union bound, $\mathbb{P}[\mathcal{M}] \ge 1 - 2\delta^3 - 2\delta^5 \ge 1 - \delta^2$.

\vspace*{-1ex}
\subsection{Proof of Equation \eqref{TrivialLinearRegretBound} (Expected Regret Given $\mathcal{M}^c$) \label{sec:ProofTrivLinear2}}

Here we show that $\mathbb{E}[R_T | \mathcal{A}] = \mathcal{O}\big( T \sqrt{\log \frac{1}{\mathbb{P}[\mathcal{A}]}} \big)$ for any event $\mathcal{A}$.  For brevity, we define $\delta_0 = \mathbb{P}[\mathcal{A}]$.

We first note from \lemref{RunningMaxDistributionLemma} in Appendix \ref{app:results} (in the supplementary material) that the maximum (and similarly, the minimum) of a Brownian motion has the same distribution as the absolute value of an $\mathcal{N}(0,1)$ random variable.  By \lemref{NormalUpperBoundLemma}, it follows that (unconditionally) $\max_{x \in [0,1]} W_x - \min_{x \in [0,1]} W_x \le c \sqrt{\log\frac{1}{\delta}}$ with probability at least $1-\delta^3$, where $\delta$ is arbitrary and $c$ is an absolute constant.  Note that the cumulative regret up to time $T$ is trivially upper bounded by $T$ times ${\rm Gap} := \max_{x \in [0,1]} W_x - \min_{x \in [0,1]} W_x$.

To move to the case with conditioning on $\mathcal{A}$, we write
\begin{align}
    &\mathbb{E}[{\rm Gap} \mid \mathcal{A}]  \nonumber\\*
    &= \mathbb{E}\bigg[{\rm Gap}\, \boldsymbol{1}\Big\{ {\rm Gap} \le c \sqrt{\log\frac{1}{\delta_0}} \Big\} \,\Big|\, \mathcal{A}\bigg] \nonumber\\*
    &\qquad + \mathbb{E}\bigg[{\rm Gap}\, \boldsymbol{1}\Big\{ {\rm Gap} > c \sqrt{\log\frac{1}{\delta_0}} \Big\} \,\Big|\, \mathcal{A}\bigg]\\
    &\le c \sqrt{\log\frac{1}{\delta_0}} + \frac{1}{\delta_0}\mathbb{E}\bigg[{\rm Gap}\, \boldsymbol{1}\Big\{ {\rm Gap} > c \sqrt{\log\frac{1}{\delta_0}} \Big\} \bigg], \label{GapTerms}
\end{align}
where the second term of \eqref{GapTerms} uses $\mathbb{E}[Z|\mathcal{A}] = \frac{\mathbb{E}[Z\boldsymbol{1}\{\mathcal{A}\}]}{\mathbb{P}[\mathcal{A}]} \le \frac{\mathbb{E}[Z]}{\mathbb{P}[\mathcal{A}]}$ for any non-negative random variable $Z$ and event $\mathcal{A}$.

Using the formula $\mathbb{E}[Z] = \int_{0}^{\infty} \mathbb{P}[Z \ge z] \,{\rm d}z$ for a non-negative random variable $Z$, and using the above arguments for (unconditionally) bounding ${\rm Gap}$ with high probability, it is straightforward to establish that 
\begin{align}
\mathbb{E}\bigg[{\rm Gap}\, \boldsymbol{1}\Big\{ {\rm Gap} > c \sqrt{\log\frac{1}{\delta_0}} \Big\} \bigg] = \mathcal{O}\bigg( \delta_0^3 \log \frac{1}{\delta_0} \bigg).
\end{align}
As a result, for any $\delta_0$ bounded away from one, the first term has the dominant scaling behavior in \eqref{GapTerms}, and we obtain $\mathbb{E}[{\rm Gap} \mid \mathcal{A}] = \mathcal{O}\big( \sqrt{\log\frac{1}{\delta_0}} \big)$ and hence $\mathbb{E}[R_T \mid \mathcal{A}] = \mathcal{O}\big( T \sqrt{\log\frac{1}{\delta_0}} \big)$.


\vspace*{-1ex}
\subsection{Proof of \lemref{ConditionalRunningMaxDistributionComplementVariableLengthLemma} (Running Maximum Lower Bound for a Brownian Meander)} \label{ProofRunMax}

	According to \cite{iafrate2019some}, the distribution function of the running maximum conditioned on the running minimum being positive can be expressed as follows:
	{\small\begin{align}
	&\nonumber \mathbb{P}\left[\max_{0 \le z \le s} W_z < x \,\Big|\, \min_{0 < z \le t} W_z > 0, W_0 = 0\right] \\
	&= \int_0^x \!\! H(y) \!\frac{\int_0^\infty \big[ \exp\left( - \frac{(w\! -\! y)^2}{2(t\!-\!s)} \right) \!-\! \exp\left( - \frac{(w\!+\!y)^2}{2(t\!-\!s)} \right) \frac{{\rm d}w}{\sqrt{2\pi(t\!-\!s)}}\big]}{\int_0^\infty \frac{w}{t\sqrt{2\pi t}}\exp\left( -\frac{w^2}{2t}\right) \,{\rm d}w} \,{\rm d}y,\!
	\end{align}}
	where  
\begin{equation}
	H(y)\triangleq\sum_{k=-\infty}^\infty  \frac{y-2kx}{s\sqrt{2\pi s}} \exp\left( -\frac{(y-2kx)^2}{2s} \right) .
	\end{equation}

	We first upper bound $H(y)$ in terms of $a =  \max\big( 0, \big \lfloor{\frac{\sqrt{s}-x}{2x}}\big\rfloor\big)$ with $\upsilon(z) \triangleq z\exp(-\frac{z^2}{2})$:\footnote{Here we adopt the convention $\sum_{k=1}^0 (\cdot) = 0$.}
	\begin{align}
	H(y) &= \sum_{k=-\infty}^\infty \left[ \frac{y-2kx}{s\sqrt{2\pi s}} \exp\left( -\frac{(y-2kx)^2}{2s} \right) \right] \\
	&= \frac{y}{s\sqrt{2\pi s}} \exp\left( -\frac{y^2}{2s} \right)  \nonumber\\*
	&\quad + \frac{1}{s\sqrt{2\pi}}\sum_{k=1}^\infty \left[ \upsilon\left(\frac{2kx+y}{\sqrt{s}}\right) \!-\! \upsilon\left(\frac{2kx-y}{\sqrt{s}}\right) \right]\\
	\label{Monotonicity}
	&\le \frac{y}{s\sqrt{2\pi s}} \exp\left( -\frac{y^2}{2s} \right) \nonumber\\*
	&\quad + \frac{1}{s\sqrt{2\pi}}\sum_{k=1}^a \left[ \upsilon\left(\frac{2kx+y}{\sqrt{s}}\right) \!-\! \upsilon\left(\frac{2kx-y}{\sqrt{s}}\right) \right]\\
	\label{SummationBound}
	&\le \frac{y}{s\sqrt{2\pi s}} \exp\left( -\frac{y^2}{2s} \right) + \frac{1}{s\sqrt{2\pi}}\sum_{k=1}^a \frac{2y}{\sqrt{s}} \\
	&\le \frac{y}{s\sqrt{2\pi s}} \exp\left( -\frac{y^2}{2s} \right) + \frac{1}{s\sqrt{2\pi}} \left(\frac{\sqrt{s}}{2x}\right) \frac{2y}{\sqrt{s}} \\ 
	&\le \frac{y}{s\sqrt{2\pi s}} \exp\left( -\frac{y^2}{2s} \right) + \frac{y}{xs\sqrt{2\pi}}, \label{A_Bound}
	\end{align}
	where \eqref{Monotonicity} follows from the monotonically decreasing property of the function $\upsilon(z) = z\exp(-\frac{z^2}{2})$ for $z > 1$ (since $x \ge y$, we have $\frac{2kx+y}{\sqrt{s}} > \frac{2kx-y}{\sqrt{s}} \ge \frac{\sqrt{s}+x -y}{\sqrt{s}} \ge 1$ for $k \ge a+1 =  \lfloor{\frac{\sqrt{s}-x}{2x}} \rfloor + 1 \ge  \lceil{\frac{\sqrt{s}+x}{2x}} \rceil$), \eqref{SummationBound} follows from the fact that the function $\upsilon(z) = z\exp\left(-\frac{z^2}{2} \right)$ has Lipschitz constant~$1$ on the domain $(0,1)$, and \eqref{A_Bound} follows from $a \le \frac{\sqrt s}{2x}$.

	The distribution function can then be upper bounded as follows:
	{\small \begin{align}
	&\nonumber \mathbb{P}\left[\max_{0 \le z \le s} W_z < x \,\Big|\, \min_{0 < z \le t} W_z > 0, W_0 = 0\right] \\
	& = \int_0^x \! H(y) \frac{\int_0^\infty \left[ \exp\left( - \frac{(w-y)^2}{2(t-s)} \right) \!-\! \exp\left( - \frac{(w+y)^2}{2(t-s)} \right) \right] \frac{{\rm d}w}{\sqrt{2\pi(t-s)}}}{\int_0^\infty \frac{w}{t\sqrt{2\pi t}}\exp\left( -\frac{w^2}{2t}\right) \,{\rm d}w} \,{\rm d}y \label{InitialDistr} \\
	& = \int_0^x\! H(y) \frac{\int_{-\frac{y}{\sqrt{t-s}}}^\infty \exp\left( - \frac{z^2}{2} \right) \frac{{\rm d}z}{\sqrt{2\pi}} \!-\! \int_{\frac{y}{\sqrt{t-s}}}^\infty \exp\left( - \frac{z^2}{2} \right) \frac{{\rm d}z}{\sqrt{2\pi}}}{\int_0^\infty \frac{w}{t\sqrt{2\pi t}}\exp\left( -\frac{w^2}{2t}\right) \,{\rm d}w} \,{\rm d}y  \label{ChangeOfVariable} \\
	&= \int_0^x H(y) \frac{\int_{-\frac{y}{\sqrt{t-s}}}^{\frac{y}{\sqrt{t-s}}} \exp\left( - \frac{z^2}{2} \right) \frac{{\rm d}z}{\sqrt{2\pi}}}{\int_0^\infty \frac{z}{\sqrt{2\pi t}}\exp\left( -\frac{z^2}{2}\right) \,{\rm d}w} \,{\rm d}y	\label{PreSummationBoundCall} \\
	\label{SummationBoundCall}
	&\le \int_0^x \!\left[ \frac{y}{s\sqrt{2\pi s}} \exp\left( -\frac{y^2}{2s} \right) \!+\! \frac{y}{xs\sqrt{2\pi}} \right]\! \frac{\int_{-\frac{y}{\sqrt{t-s}}}^{\frac{y}{\sqrt{t-s}}}\! \exp\left( \!-\! \frac{z^2}{2} \right) \frac{{\rm d}z}{\sqrt{2\pi}}}{ \frac{1}{\sqrt{2 \pi t}} } \,{\rm d}y,
	\end{align}}
	where \eqref{ChangeOfVariable} follows from the change of variables $z = w-y, z = w+y$, and \eqref{SummationBoundCall} follows from \eqref{A_Bound} and a direct evaluation of the integral in the denominator.
	
	We now consider two cases separately. First, if $t > 2s$:
	{\allowdisplaybreaks
	\begin{align}
	&\nonumber \mathbb{P}\left[\max_{0 \le z \le s} W_z < x \,\Big| \, \min_{0 < z \le t} W_z > 0, W_0 = 0\right] \\
	\label{UpperBoundExpInIntegral}
	&\le \int_0^x \left[ \frac{y}{s\sqrt{2\pi s}} \exp\left( -\frac{y^2}{2s} \right) + \frac{y}{xs\sqrt{2\pi}} \right] \frac{2\frac{1}{\sqrt{2\pi}}\frac{y}{\sqrt{t-s}}}{\frac{1}{\sqrt{2\pi t}}} \,{\rm d}y \\
	&= \frac{2\sqrt{t}}{\sqrt{2\pi s(t-s)}} \int_0^x \left( \frac{y^2}{s} \exp\left( -\frac{y^2}{2s} \right) + \frac{y^2}{x\sqrt{s}}\right) \,{\rm d}y \\
	\label{IntegrationByParts}
	&= \frac{2\sqrt{t}}{\sqrt{2\pi s(t-s)}} \bigg[ \sqrt{s} \int_{0}^{\frac{x}{\sqrt{s}}} \exp\left( -\frac{z^2}{2} \right) \,{\rm d}z  \nonumber\\* 
	&\qquad- x \exp\left( -\frac{x^2}{2s} \right) + \frac{x^2}{3\sqrt{s}} \bigg] \\
	\label{UpperBoundExpFunc}
	&\le \frac{2x\sqrt{t}}{\sqrt{2\pi s(t-s)}} \left[ 1 - \exp\left( -\frac{x^2}{2s} \right) + \frac{x}{3\sqrt{s}} \right] \\
	&\le \frac{1}{\sqrt{2\pi}}\sqrt{\frac{t}{t-s}} \left( \frac{x}{\sqrt{s}} \right)^3 + \sqrt{\frac{t}{t-s}}\frac{x^2\sqrt{2}}{3\sqrt{\pi}s} \label{TS_Simplified} \\
	\label{Tge2s}
	&< \frac{1}{\sqrt{\pi}} \left( \frac{x}{\sqrt{s}} \right)^3 	+ \frac{2}{3\sqrt{\pi}} \left( \frac{x}{\sqrt{s}}\right)^2 \\
	&\le \frac{5}{6\sqrt{\pi}} \left( \frac{x}{\sqrt{s}}\right)^2 \\
	&\le \frac{1}{2} \left( \frac{x}{\sqrt{s}}\right)^2,
	\end{align} 
	where \eqref{UpperBoundExpInIntegral} follows by upper bounding the exponential function $\exp(-\frac{z^2}{2})$ by $1$ in the inner integral of \eqref{SummationBoundCall}, \eqref{IntegrationByParts} follows from integration by parts, \eqref{UpperBoundExpFunc} follows by upper bounding the exponential function $\exp(-\frac{z^2}{2})$ by $1$, \eqref{TS_Simplified} follows from the fact that $1 - \frac{x^2}{2s} \le \exp(-\frac{x^2}{2s})$ as $x^2 < \frac{1}{4}s < 2s$, and \eqref{Tge2s} follows from $\frac{t}{t-s} < 2$ (since $t > 2s$) and $x < \sqrt{s}/2$ (assumed in the lemma).}
	
	{\allowdisplaybreaks
	As for the other case, if $t \le 2s$, then:
	\begin{align}
	\label{ProbabilityUpperBoundBy1}
	&\mathbb{P}\left[\max_{0 < z \le s} W_z < x \,\Big| \, \min_{0 \le z \le t} W_z > 0, W_0 = 0\right] \nonumber \\
	&\le \int_0^x \left[ \frac{y}{s\sqrt{2\pi s}} \exp\left( -\frac{y^2}{2s} \right) + \frac{y}{xs\sqrt{2\pi}} \right] \frac{1}{\frac{1}{\sqrt{2\pi t}}} \,{\rm d}y \\
	&= \sqrt{t} \int_0^x \left[\frac{y}{s\sqrt{s}} \exp\left( -\frac{y^2}{2s} \right) + \frac{y}{xs} \right] \,{\rm d}y \\
	& = \sqrt{t} \int_0^{x/\sqrt{s}} \frac{z}{\sqrt{s}} \exp\left( -\frac{z^2}{2} \right) \,{\rm d}z + \sqrt{t} \left[ \frac{y^2}{2xs} \right]_{y=0}^{y=x} \label{ChgVar2} \\
	&= \frac{\sqrt{t}}{\sqrt{s}} \left[ -\exp\left( -\frac{z^2}{2} \right) \right]_{z=0}^{z=\frac{x}{\sqrt{s}}} + \frac{x\sqrt{t}}{2s} \\
	&= \frac{\sqrt{t}}{\sqrt{s}} \left[ 1 - \exp\left( -\frac{x^2}{2s} \right) \right] + \frac{x\sqrt{t}}{2s}\\
	& \le \frac{x^2 \sqrt{t}}{2s\sqrt{s}} + \frac{x\sqrt{t}}{2s} 	\label{EulerExpansionExponential2} \\
	\label{Tle2s}
	&\le \frac{1}{\sqrt{2}}\left( \frac{x}{\sqrt{s}} \right)^2 +\frac{x}{\sqrt{2s}}\\
	&\le \sqrt{2}\frac{x}{\sqrt{s}},
	\end{align}
	 where \eqref{ProbabilityUpperBoundBy1} follows from the fact that the numerator in \eqref{SummationBoundCall} $\int_{-\frac{y}{\sqrt{t-s}}}^{\frac{y}{\sqrt{t-s}}} \exp(-\frac{z^2}{2}) \frac{{\rm d}z}{\sqrt{2\pi}} = \mathbb{P}[ -\frac{y}{\sqrt{t-s}} \le Z \le \frac{y}{\sqrt{t-s}}] \le 1$ for $Z \sim \mathcal{N}(0,1)$, \eqref{ChgVar2} applies the change of variable $z = \frac{y}{\sqrt s}$, \eqref{EulerExpansionExponential2} follows from the fact that $\exp(-x) > 1 - x$ for all $x > 0$, and~\eqref{Tle2s} follows from $t \le 2s$ and $x < \frac{1}{2}\sqrt{s} \le \sqrt{s}$ (assumed in the lemma). Hence, the probability of the complement event is lower bounded as follows:}
	\begin{align}
	&\mathbb{P}\left[ \max_{0 < z \le s} W_z \ge x \,\Big|\, \min_{0 \le z \le t} W_z > 0, W_0 = 0 \right] \nonumber\\*
	&\quad \ge
	\left\{ \begin{array}{lcl}
	 1 - \frac{1}{2} \left( \frac{x}{\sqrt{s}}\right)^2 &\text{ if } t > 2s\\[5mm]
	 1 -  \sqrt{2}\frac{x}{\sqrt{s}} & \text{ if } t \le 2s.
	\end{array} \right.
	\end{align}
	
\vspace*{-1ex}
\subsection{Proof of \lemref{RunningMinimumBrownianMeanderLemma} (Running Minimum Lower Bound for a Brownian Meander)} \label{ProofRunMin}

We have
\begin{align}
&\nonumber \mathbb{P}\left[ \min_{0\le z \le t} W_z > \varepsilon \,\Big| \, \min_{0\le z \le t} W_z > 0, W_0 = u \right] \\
&= \frac{\mathbb{P}\left[ \min_{0\le z \le t} W_z > \varepsilon, \min_{0\le z \le t} W_z > 0 \mid W_0 = u \right]}{\mathbb{P}\left[ \min_{0\le z \le t} W_z > 0 \mid W_0 = u \right]} \\
&= \frac{\mathbb{P}\left[ \min_{0\le z \le t} W_z > \varepsilon \mid W_0 = u \right]}{\mathbb{P}\left[ \min_{0\le z \le t} W_z > 0 \mid W_0 = u \right]} \\
\label{RunningMinimumFormulaCitation}
&= \frac{\int_\varepsilon^\infty \left[ \exp\left( - \frac{(y-u)^2}{2t} \right) - \exp\left( - \frac{(y+u-2\varepsilon)^2}{2t} \right) \right] \text{d}y}{\int_0^\infty \left[ \exp\left( - \frac{(y-u)^2}{2t} \right) - \exp\left( - \frac{(y+u)^2}{2t} \right) \right] \text{d}y} \\
&= \frac{\frac{1}{\sqrt{2\pi}}\int_{0}^{\frac{u-\varepsilon}{\sqrt{t}}} \exp\left( - \frac{1}{2}z^2 \right) \text{d}z}{\frac{1}{\sqrt{2\pi}}\int_{0}^{\frac{u}{\sqrt{t}}} \exp\left( - \frac{1}{2}z^2 \right) \text{d}z} \label{IntegralManipulations} \\
&\ge \frac{\frac{u-\varepsilon}{\sqrt{t}}}{\frac{u}{\sqrt{t}}} = \frac{u - \varepsilon}{u}, \label{NormalPdfDecreasing}
\end{align}
where \eqref{RunningMinimumFormulaCitation} follows from the joint distribution of Brownian meander (with initial value $u$) and its running minimum as stated in \lemref{BrownianMeanderDistributionFunctionLemma} in Appendix \ref{app:results}, \eqref{IntegralManipulations} follows similar steps to \eqref{InitialDistr}--\eqref{PreSummationBoundCall} (recall also that we assumed $0 < \epsilon < u$), and \eqref{NormalPdfDecreasing} follows from the fact that $e^{-z^2/2}$ is monotonically decreasing on the positive real line.

\vspace*{-1ex}
\subsection{Proof of \lemref{EventTLemma} (Probability of Event $\mathcal{T}$)} \label{EventT_Proof}

We lower bound the probability of each $\mathcal{T}_i$ separately for $i=1,2,3$.

{\bf Bounding $\mathbb{P}[\mathcal{T}_1]$.} Let $\delta_1 = \Delta^\eta$ denote the target error probability of event $\mathcal{T}_1$. We first define three running maxima of three regions as follows: 
\begin{align}
M_1 &= \max_{x \in [-\Delta, 2\Delta]} \tilde{W}_x,\quad  M_2 = \max_{x \in [2\Delta, 1-2\Delta]} \tilde{W}_x,  \nonumber\\* 
M_3&= \max_{x \in [1-2\Delta, 1+\Delta]} \tilde{W}_x.
\end{align}
We consider two separate events:
\begin{equation} \label{E_Defs}
    E_1 = \{M_1 < M_2\}, ~~\text{and}~~ E_2 = \{\max(M_1, M_2) > M_3\}, 
\end{equation}
whose intersection directly implies $\mathcal{T}_1$. To simplify notation, we define $Z$ to be a standard normal random variable independent of the other defined random variables.

To bound the probability of $E_1$, we first establish a high-probability upper bound $\displaystyle \hat{M}_1 = \sqrt{2\Delta \ln \frac{8}{\delta_1\sqrt{2\pi}}}$ for $M_1$:
\begin{align}
	\label{RunningMaxDisribution}
	\mathbb{P}[M_1 \le \hat{M}_1] &= \mathbb{P}\left[ \lvert \tilde{W}_{2\Delta} \rvert \le \sqrt{2\Delta \ln \frac{8}{\delta_1\sqrt{2\pi}}} \right] \\
	&= \mathbb{P}\left[\lvert Z \rvert \le \sqrt{\ln \frac{8}{\delta_1\sqrt{2\pi}}} \right] \\
	\label{NormalUpperBound}
	&= 1 - 2Q\left(\sqrt{\ln \frac{8}{\delta_1\sqrt{2\pi}}}\right) \\ &\ge 1 - \frac{\delta_1}{4},
\end{align}
where \eqref{RunningMaxDisribution} follows from the distribution function of the running maximum of a BM (see \lemref{RunningMaxDistributionLemma} in Appendix \ref{app:results}), and \eqref{NormalUpperBound} follows from \lemref{NormalUpperBoundLemma}.

The lower bound on $M_2$ requires a two-part argument. First, we lower bound the value of $\tilde{W}_{2\Delta}$ by $\hat{w}_{2\Delta} = -\sqrt{2\Delta \ln \frac{8}{\delta_1\sqrt{2\pi}}}$:
\begin{align}
	\mathbb{P}\left[\tilde{W}_{2\Delta} \ge \hat{w}_{2\Delta}\right]  
	&= \mathbb{P}\left[Z \ge -\sqrt{\ln \frac{8}{\delta_1\sqrt{2\pi}}} \right] \\
	&= 1 - Q\left(\sqrt{\ln \frac{8}{\delta_1\sqrt{2\pi}}}\right) \\
	&\ge 1 - \frac{\delta_1}{8}. \label{Delta8_Bound}
\end{align}
Second, we lower bound the value of $M_2 - \tilde{W}_{2\Delta}$ by 
\begin{align}
\hat{S}_2 = \frac{\delta_1\sqrt{\pi}}{8\sqrt{2}}\sqrt{1-4\Delta}.
\end{align}
To see this, we similarly use the distribution of running maximum of Brownian motion (\lemref{RunningMaxDistributionLemma}) to obtain
\begin{align}
&    \mathbb{P}\left[M_2 - \tilde{W}_{2\Delta} \ge \frac{\delta_1\sqrt{\pi}}{8\sqrt{2}}\sqrt{1-4\Delta} \right]   \nonumber\\*
    &= \mathbb{P}\left[\lvert Z \rvert \ge \frac{\delta_1\sqrt{\pi}}{8\sqrt{2}} \right] = 1 - \mathbb{P}\left[\lvert Z \rvert < \frac{\delta_1\sqrt{\pi}}{8\sqrt{2}} \right] \\
	&=  1 - 2\int_0^{\frac{\delta_1\sqrt{\pi}}{8\sqrt{2}}} \frac{1}{\sqrt{2\pi}}\exp\left(-\frac{z^2}{2}\right) \text{d}z\\
	& \ge 1 - 2\frac{\delta_1\sqrt{\pi}}{8\sqrt{2}} \frac{1}{\sqrt{2\pi}} = 1 - \frac{\delta_1}{8}.
\end{align}
Hence, defining $\hat{M}_2 = \hat{w}_{2\Delta} + \hat{S}_2$ and applying the union bound, we obtain
\begin{align}
\mathbb{P}\left[ M_2 \ge \hat{M}_2 \right] &= \mathbb{P}\left[ M_2 \ge \hat{w}_{2\Delta} + \hat{S}_2\right] \\
\label{M2LowerBound}
&\ge 1 - \left[ 1 - \mathbb{P}\left[ \tilde{W}_{2\Delta} \ge \hat{w}_{2\Delta} \right] \right]  \nonumber\\* 
&\qquad- \left[ 1 - \mathbb{P}\left[ M_2 - \tilde{W}_{2\Delta} \ge \hat{S}_2 \right] \right] \\
&\ge 1 - \frac{\delta_1}{4}.
\end{align}
As we take $\delta_1 = \Delta^\eta$ for some $\eta < \frac{1}{2}$, we have $\hat{M}_2 \ge \hat{M}_1$ for sufficiently small $\Delta$, since $\hat{M}_1 = \Theta( \sqrt{\Delta \log(1/\Delta)} )$ and $\hat{M}_2 = \Theta( \Delta^{\eta} - \sqrt{\Delta \log(1/\Delta)} ) = \Theta( \Delta^{\eta} )$ as $\Delta \to 0$.
Therefore,
\begin{align}
&\mathbb{P}\left[ M_2 \ge M_1 \right] \nonumber\\*
& \ge 1 - \left[1 - \mathbb{P}\left[ M_2 \ge \hat{M}_2 \right]\right] - \left[1 - \mathbb{P}\left[ M_1 \le \hat{M}_1 \right]\right] \\
&\ge 1 - \frac{\delta_1}{2},
\end{align}
following from \eqref{NormalUpperBound} and \eqref{M2LowerBound} along with the union bound.

To bound the probability of $E_2$ (see \eqref{E_Defs}), we let $w = \tilde{W}_{1 - 2\Delta}$, and establish an upper bound on $M_3$ in terms of $w$ that holds with probability at least $1 - \delta_1 / 4$, namely $\hat{M}_3 = w + \sqrt{2\Delta \ln \frac{8}{\delta_1\sqrt{2\pi}}}$. This is proved similarly to \eqref{NormalUpperBound}, so the details are omitted to avoid repetition. Similarly to \eqref{Delta8_Bound}, we define $\hat{w}_{1-2\Delta} = \sqrt{(1-\Delta)\ln\frac{8}{\delta_1\sqrt{2\pi}}}$ to be the high-probability upper bound of $\tilde{W}_{1-2\Delta}$, namely, $\mathbb{P}\big[ \tilde{W}_{1-2\Delta} \le \hat{w}_{1-2\Delta} \big] \ge 1-\frac{\delta_1}{8}$.

Using the density function of $\tilde{W}_{1-2\Delta}$, denoted by $f_{\tilde{W}_{1-2\Delta}}(w)$, we lower bound the probability for $\max(M_1, M_2)$ to exceed $\hat{M}_3$ in \eqref{eqn:Pmax}--\eqref{TaylorExpansionExp} on the top of the next page,
\begin{figure*}
\begin{align}
	&\mathbb{P}\left[ \max(M_1, M_2) \ge \hat{M}_3 \right] = \mathbb{P}\left[ \max_{x \in [-\Delta, 1 - 2\Delta]} \tilde{W}_x \ge \hat{M}_3 \right] \label{eqn:Pmax}\\
	&= \int_{-\infty}^\infty \mathbb{P}\left[  \max_{x \in [-\Delta, 1 - 2\Delta]} \tilde{W}_x \ge \hat{M}_3 \,\Big|\, \tilde{W}_{1-2\Delta} = w \right] f_{\tilde{W}_{1-2\Delta}}(w) \text{d}w\\
	\label{BrownianBridgeProbability}
	&=\int_{-\infty}^\infty \min\left(1,  \exp\left[ -\frac{2\sqrt{2\Delta \ln \frac{8}{\delta_1\sqrt{2\pi}}}\left(w + \sqrt{2\Delta \ln \frac{8}{\delta_1\sqrt{2\pi}}}\right)}{1 - \Delta} \right] \right) f_{\tilde{W}_{1-2\Delta}}(w) \text{d}w\\
	&\ge \int_{-\infty}^{\hat{w}_{1-2\Delta}} f_{\tilde{W}_{1-2\Delta}}(w) \text{d}w \exp\left[ -\frac{2\sqrt{2\Delta \ln \frac{8}{\delta_1\sqrt{2\pi}}}\left(\hat{w}_{1-2\Delta} + \sqrt{2\Delta \ln \frac{8}{\delta_1\sqrt{2\pi}}}\right)}{1 - \Delta} \right] \label{SubWhat}\\
	&=\mathbb{P}\left[ \tilde{W}_{1-2\Delta} \le \hat{w}_{1-2\Delta} \right] \exp\left[ -\frac{2\sqrt{2\Delta \ln \frac{8}{\delta_1\sqrt{2\pi}}}\left(\hat{w}_{1-2\Delta} + \sqrt{2\Delta \ln \frac{8}{\delta_1\sqrt{2\pi}}}\right)}{1 - \Delta} \right]\\
	\label{UnionBound}
	&\ge \exp\left[ -\frac{2\sqrt{2\Delta \ln \frac{8}{\delta_1\sqrt{2\pi}}}\left(\sqrt{(1-\Delta)\ln\frac{8}{\delta_1\sqrt{2\pi}}} + \sqrt{2\Delta \ln \frac{8}{\delta_1\sqrt{2\pi}}}\right)}{1 - \Delta} \right] - \frac{\delta_1}{8} \\
	\label{SquareRootInequality}
	&\ge \exp\left( -4\sqrt{2\Delta}\ln\frac{8}{\delta_1\sqrt{2\pi}}\right) - \frac{\delta_1}{8} \\
	\label{TaylorExpansionExp}
	&> 1 - 4\sqrt{2\Delta}\ln\frac{8}{\delta_1\sqrt{2\pi}} - \frac{\delta_1}{8} \ge 1 - \frac{\delta_1}{4}, 
\end{align}\hrulefill
\end{figure*}
where
\begin{itemize}
    \item \eqref{BrownianBridgeProbability} follows from the distribution function of the running maximum of Brownian bridge as stated (\lemref{BrownianBridgeLemma} in the supplementary material,  with $a = -\Delta$, $b = 1 - 2\Delta$, $w_a = 0$, $w_b = w$, and $x = \hat{M}_3$) and the fact that the underlying probability is trivially one when the $\exp(-(\dotsc))$ term is greater than one (since the running maximum is always at least as high as the two endpoints);
    \item \eqref{SubWhat} follows from the fact that 
    $\exp\left( -aw + b \right)$ is decreasing in $w$ for any $a > 0$ and $b$;
    \item \eqref{UnionBound} follows from the above-established fact that $\tilde{W}_{1-2\Delta} \le \hat{w}_{1-2\Delta}$ with probability at least $1 - \frac{\delta_1}{8}$, along with $\mathbb{P}[A]\exp(-\alpha) = (1 - \mathbb{P}[A^c])\exp(-\alpha) \ge \exp(-\alpha) - \mathbb{P}[A^c]$ for $\alpha \ge 0$ (and hence $\exp(-\alpha) \le 1$);
    \item \eqref{SquareRootInequality} follows from the fact that $\sqrt{1 - \Delta} + \sqrt{2\Delta} < 2$ for $\Delta \in (0, \frac{1}{2})$;
    \item the first inequality in~\eqref{TaylorExpansionExp} follows from the fact that $\exp(-x) > 1$ for all $x > 0$, along with the choice $\delta_1 = \Delta^\eta$ (when $\Delta$ is sufficiently small).
\end{itemize}
Hence, $E_2$ holds with probability at least $1 - \delta_1 / 2$.
As the intersection of $E_1$ and $E_2$ implies $\mathcal{T}_1$, the maximum lies in between $2\Delta$ and $1-2\Delta$ with probability at least $1 - \delta_1$ by the union bound on the two events.

{\bf Bounding $\mathbb{P}[\mathcal{T}_2]$.}
    Recall from \eqref{x+*x-*}--\eqref{M+} that $M^+$ and $x_M^+$ are respectively the maximum and maximizer of $W^+$. It will be useful to additionally define the following two points, where $\delta^\prime \triangleq 0.1\delta$:
	\begin{align}
	&\xM^{+,\mathrm{L}} \triangleq \sup\left\{ x \, \big|\, x < \xM^+, W^+_x = M^+ - \delta^\prime\sqrt{2\Delta} \right\} \\
	&\xM^{+,\mathrm{R}} \triangleq \inf\left\{ x \, \big|\, x > \xM^+, W^+_x = M^+ - \delta^\prime\sqrt{2\Delta} \right\}.
	\end{align}
	We then consider the following auxiliary events:
	\begin{align}
	&\mathcal{T}_2^{(\mathrm{A})} \triangleq \left\{ \min_{x \in [\xM^+-\Delta, \xM^+]} W^+_x \le M^+ - (\delta^\prime)^2\sqrt{2\Delta} \right\} \\
	&\mathcal{T}_2^{(\mathrm{B})} \triangleq \left\{ \min_{x \in [\xM^+, \xM^++\Delta]} W^+_x \le M^+ - (\delta^\prime)^2\sqrt{2\Delta} \right\} \\
	&\mathcal{T}_2^{(\mathrm{C})} \triangleq \left\{ \max_{x \in [0, \xM^{+,\mathrm{L}}]} W^+_x \le M^+ - (\delta^\prime)^2\sqrt{2\Delta} \right\} \\
	&\mathcal{T}_2^{(\mathrm{D})} \triangleq \left\{ \max_{x \in [\xM^{+,\mathrm{R}}, 1]} W^+_x \le M^+ - (\delta^\prime)^2\sqrt{2\Delta} \right\}.
	\end{align}
	By definition, $\mathcal{T}_2^{(\mathrm{A})}$ is equivalent to $\{ \xM^{+,\mathrm{L}} \ge \xM^+-\Delta \}$, and $\mathcal{T}_2^{(\mathrm{B})}$ is equivalent to $\{ \xM^{+,\mathrm{R}} \le \xM^++\Delta \}$. The intersection of all four events and $\mathcal{T}_1$ implies $\mathcal{T}_2$, since if events  $\mathcal{T}_2^{(\mathrm{A})}$ and $\mathcal{T}_2^{(\mathrm{B})}$ hold, then there is a point in the $\Delta$-neighborhood of $\xM^+$ with value $W^{+}_{x} = M - \delta^\prime\sqrt{2\Delta}$; moreover, if events $\mathcal{T}_2^{(\mathrm{C})}$ and $\mathcal{T}_2^{(\mathrm{D})}$ simultaneously hold, $W^+$ does not fall below $M - (\delta^\prime)^2\sqrt{2\Delta}$ outside this neighborhood. As long as event $\mathcal{T}_1$ holds, both maxima $\xM^+$ and $\xM^-$ are within $[0, 1]$, and it is safe to conclude that if $r^+(x)$ is smaller than $(\delta^\prime)^2 \sqrt{2\Delta}$, $x$ must be in the $\Delta$-neighborhood near $\xM^+$, and we can conclude that $r^-(x) = M - W^{-}_{x} = M^+ - \tilde{W}_{x-2\Delta} = r^+(x-2\Delta)$ is larger than $(\delta^\prime)^2\sqrt{2\Delta}$.

We now define two processes, representing left and right Brownian meanders (i.e., BM conditioned on being positive):
\begin{align}
 B_2^{\mathrm{L}}(x) &\triangleq \frac{M^+ - W^+_{\xM^+ - x \cdot \xM^+}}{\sqrt{\xM^+}},\quad\mbox{and}  \label{DefinitionBrownianMeandersEqL}\\ 
 B_2^{\mathrm{R}}(x)& \triangleq \frac{M^+ - W^{+}_{\xM^+ + x \cdot (1 - \xM^+)}}{\sqrt{1 - \xM^+}}.  \label{DefinitionBrownianMeandersEqR}
\end{align}
Specifically, \lemref{BrownianMeanderIdenticalDistributionLemma} in Appendix \ref{app:results} formally states that $B_2^{\mathrm{L}}$ and $B_2^{{\rm R}}$ are identically distributed Brownian meanders, and furthermore, that they are independent from each other and from $\xM^+$.  See Figure \ref{fig:test} for an example.

 \begin{figure}[t]
\centering
\begin{tabular}{c}
     \includegraphics[width=.8\columnwidth]{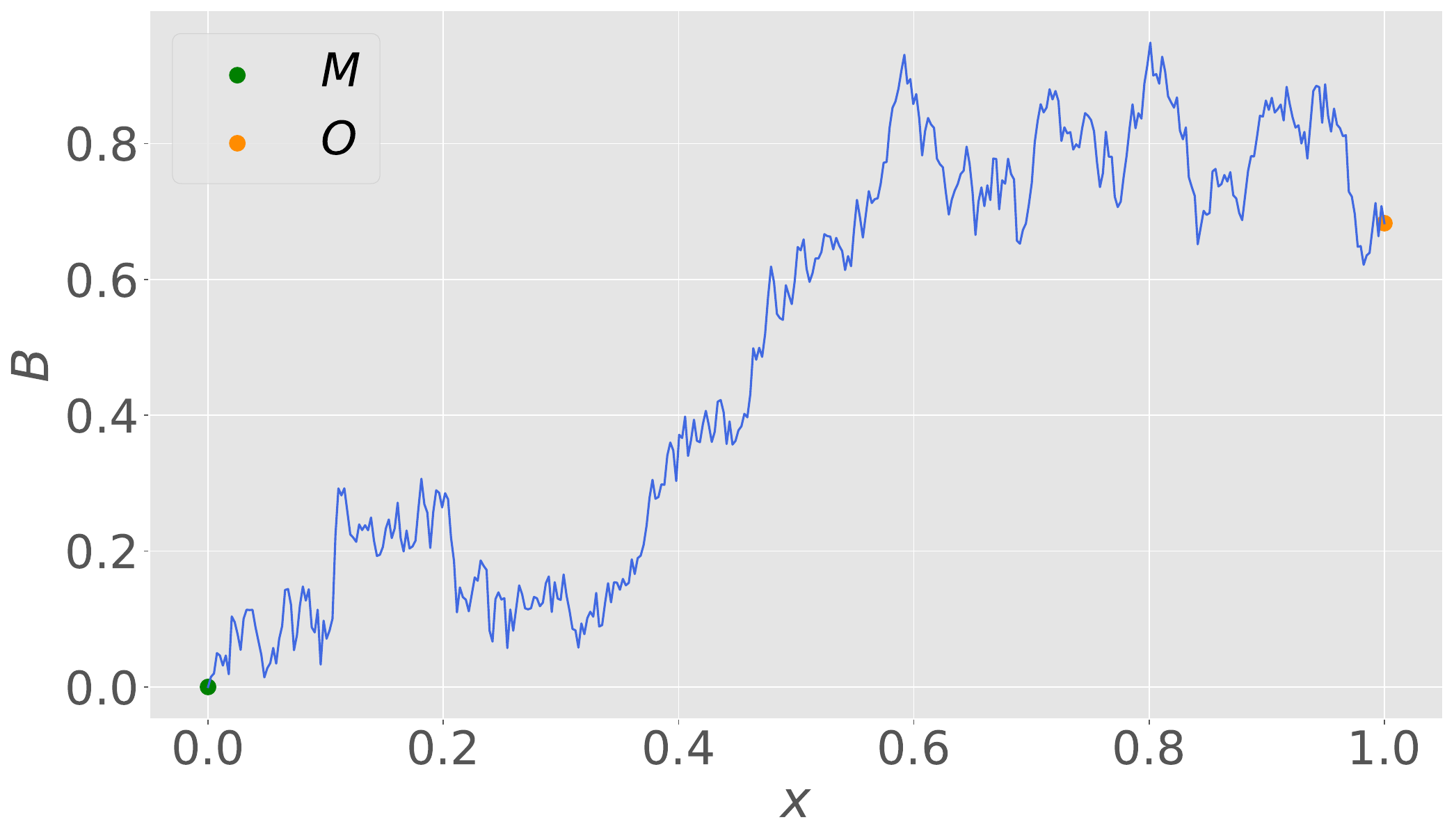}\\
  \includegraphics[width=.8\columnwidth]{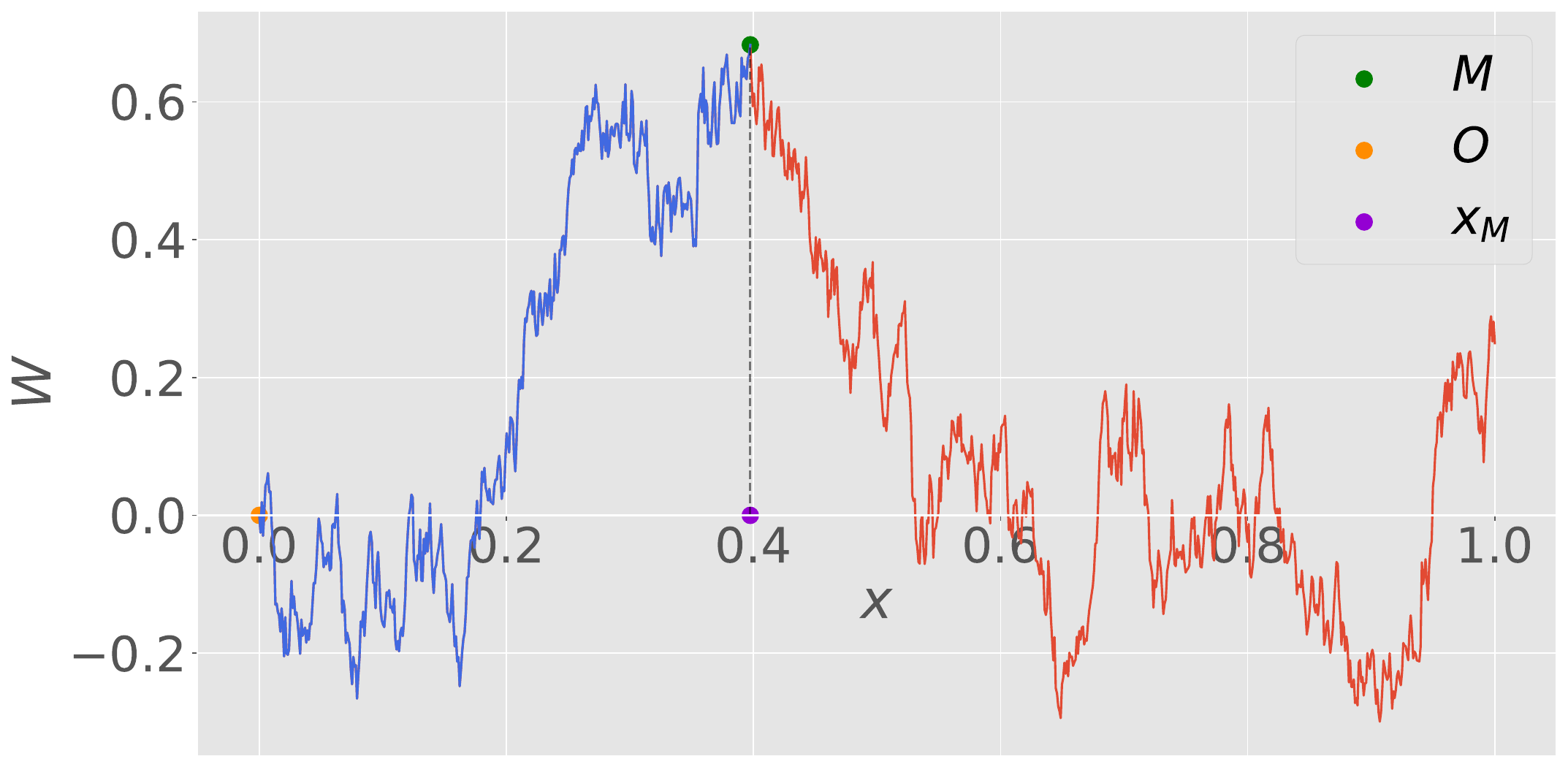} 
\end{tabular}
\caption{Left Brownian meander example (top), and the original Brownian motion (bottom).  Observe the correspondences between the points labeled M (maximum) and O (origin).}
\label{fig:test}
\end{figure}

For any fixed $\hat{x}_M^+$ in the range $[\Delta, 1 - 3\Delta]$, we define an intermediate quantity $\xi \triangleq \frac{\delta^\prime\sqrt{2\Delta}}{\sqrt{\hat{x}_M^+}}$ and note that the following event holds for $B_2^{\mathrm{L}}$ by construction:
\begin{align}
	\mathcal{B}_2^{\mathrm{L}} \triangleq \left\{ \min_{x \in [0, 1]} B_2^{\mathrm{L}}(x) > 0, B_2^{\mathrm{L}}(0) = 0, B_2^{\mathrm{L}}\Big(1 - \frac{\xM^{+, {\rm L}}}{\hat{x}_M^+}\Big) = \xi \right\}. \label{eq:SetB2}
\end{align}
Since the pair $(B_2^{\mathrm{L}},B_2^{\mathrm{R}})$ is independent of $\xM^+$, we can condition on $\xM^+ = \hat{x}_M^+$ in accordance with event $\mathcal{T}_1$ as follows:
\begin{align}
	&\nonumber \mathbb{P}[\mathcal{T}_2^{(\mathrm{A})} \mid \xM^+ = \hat{x}_M^+] \\
	\label{BrownianMeanderConversion}
	& \ge \mathbb{P}\bigg[ \max_{x \in [0, \Delta / \hat{x}_M^+]} B_2^{\mathrm{L}}(x) \ge  \xi\delta^\prime \Big| \mathcal{B}_2^{\mathrm{L}} \bigg] \\
	\label{ConditionalRunningMaxDistributionComplementVariableLength}
	& \ge 1 - \max\left(\frac{1}{2} \left( \frac{ \xi\delta^\prime}{ \xi/\delta^\prime} \right)^2, 2(\delta^\prime)^2\right) \\
	&= 1 - \max\left(\frac{1}{2}(\delta^\prime)^4, 2(\delta^\prime)^2\right) = 1 - 2(\delta^\prime)^2
\end{align}
where \eqref{BrownianMeanderConversion} follows from the definition of $B_2^{\mathrm{L}}$, and \eqref{ConditionalRunningMaxDistributionComplementVariableLength} follows from the minimum of two cases in  \lemref{ConditionalRunningMaxDistributionComplementVariableLengthLemma} (with $x = (\delta^\prime)^2\sqrt{2\Delta} / \sqrt{\xM^+}$, $s = \Delta / \xM^+$, and $t = 1$).\footnote{Note that the Brownian meander in  \lemref{ConditionalRunningMaxDistributionComplementVariableLengthLemma} does not contain any counterpart to the third condition in \eqref{eq:SetB2}, but that condition is not playing a role here.  It is simply an almost-sure event that holds by the definition of $\xM^{+, {\rm L}}$, so it makes no difference whether it is included in the conditioning or not.  It is included in \eqref{eq:SetB2} for later convenience following \eqref{TildeB2Ldefinition}.}

Averaging over $x_M^+$, we have
\begin{align}
    \mathbb{P}[\mathcal{T}_2^{(\mathrm{A})}] &= \int_{-\infty}^\infty f_{x_M^+}(\hat{x}_M^+) \mathbb{P}[\mathcal{T}_2^{(\mathrm{A})} \mid \xM^+ = \hat{x}_M^+] \mathrm{d} \hat{x}_M^+ \\
    &= \int_{\Delta}^{1-3\Delta}\! f_{x_M^+}(\hat{x}_M^+) \mathbb{P}[\mathcal{T}_2^{(\mathrm{A})} \mid \xM^+ = \hat{x}_M^+] \mathrm{d} \hat{x}_M^+ \\
    \label{ConditionalOnXMPlusLowerBound}
    &\ge \int_{\Delta}^{1-3\Delta} f_{x_M^+}(\hat{x}_M^+) \left(1 - \delta^\prime\sqrt{2}\right) \mathrm{d} \hat{x}_M^+ \\
    &= \big(1 - 2(\delta^\prime)^2 \big) \mathbb{P}\left[x_M^+ \in [\Delta, 1-3\Delta]\right] \\
    \label{T1BoundUsageEquation}
    &\ge \big(1 - 2(\delta^\prime)^2 \big) (1 - \Delta^\eta) \\
    &\ge 1 - \Delta^\eta - 2(\delta^\prime)^2,
\end{align}
where \eqref{ConditionalOnXMPlusLowerBound} follows from \eqref{ConditionalRunningMaxDistributionComplementVariableLength}, and \eqref{T1BoundUsageEquation} follows from the fact that $\mathbb{P}[\mathcal{T}_1] \ge 1 - \Delta^\eta$. Similarly, as $B_2^{{\rm L}}$ is equal in distribution to $B_2^{{\rm R}}$, $\mathbb{P}[\mathcal{T}_2^{(\mathrm{B})} \mid x_M^- = \hat{x}_M^-] \ge 1 - 2(\delta^\prime)^2$ (following from \lemref{ConditionalRunningMaxDistributionComplementVariableLengthLemma} for any $x_M^-$ restricted in the range $[3\Delta, 1-\Delta]$ in accordance with $\mathcal{T}_1$), and thus $\mathbb{P}[\mathcal{T}_2^{(\mathrm{B})}] \ge \left(1 - 2(\delta^\prime)^2\right)\mathbb{P}[x_M^- \in [3\Delta, 1-\Delta]] \ge \left(1 - 2(\delta^\prime)^2\right) (1 - \Delta^\eta) \ge 1 - \Delta^\eta - 2(\delta^\prime)^2$.

Before considering $\mathcal{T}_2^{(\mathrm{C})}$, we use the Markov property to define a {\em non-standard} Brownian meander $\widetilde{B}_2^{\mathrm{L}}$ as follows: For all $0 \le x \le \frac{\xM^{+,\mathrm{L}}}{\hat{x}_M^+}$,
\begin{align} \label{TildeB2Ldefinition}
	\widetilde{B}_2^{\mathrm{L}}(x) \triangleq B_2^{\mathrm{L}}\Big(x + 1 - \frac{\xM^{+,\mathrm{L}}}{\hat{x}_M^+}\Big) 
\end{align}
The distribution of $\widetilde{B}_2^{\mathrm{L}}$ follows from $B_2^{\mathrm{L}}$ in the domain $[1 - \xM^{+,\mathrm{L}}/\hat{x}_M^+, 1]$ to be a Brownian meander with starting value $\xi$, due to the Markov property of BM. The counterpart $\widetilde{\mathcal{B}}_2^{\mathrm{L}}$ to the event $\mathcal{B}_2^{\mathrm{L}}$ in \eqref{eq:SetB2} is as follows:
\begin{align}
	\widetilde{\mathcal{B}}_2^{\mathrm{L}} \triangleq \left\{ \min_{x \in [0, \xM^{+, {\rm L}}/\hat{x}_M^+]} \widetilde{B}_2^{\mathrm{L}}(x) > 0,  \widetilde{B}_2^{\mathrm{L}}(0) =  \xi \right\}
\end{align}
Next, we lower bound $\mathbb{P}[\mathcal{T}_2^{(\mathrm{C})}]$ by conditioning on $\xM^+ = \hat{x}_M^+$ as follows: 
\begin{align}
	&\nonumber \mathbb{P}[\mathcal{T}_2^{(\mathrm{C})} \mid \xM^+ = \hat{x}_M^+] \\
	\label{BrownianMeanderConversion2}
	& = \mathbb{P}\left[ \min_{x \in [1 - \xM^{+, {\rm L}}/\hat{x}_M^+, 1]} B_2^{\mathrm{L}}(x) \ge  \xi\delta^\prime \,\Big|\, \mathcal{B}_2^{\mathrm{L}} \right] \\
	\label{RedefineB2}
	& = \mathbb{P}\left[ \min_{x \in [0, \xM^{+, {\rm L}}/\hat{x}_M^+]} \widetilde{B}_2^{\mathrm{L}}(x) \ge  \xi\delta^\prime \,\Big|\, \widetilde{\mathcal{B}}_2^{\mathrm{L}} \right] \\
	\label{RunningMinimumBrownianMeander}
	& \ge 1 - \frac{ \xi\delta^\prime}{ \xi}= 1 - \delta^\prime, 
\end{align}
	where:
	\begin{itemize}
	    \item \eqref{BrownianMeanderConversion2} follows from the definition of $B_2^{\mathrm{L}}$ in \eqref{DefinitionBrownianMeandersEqL}, which implies that $B_2^{{\rm L}}$ is non-negative and has $B_2^{\mathrm{L}}\big(1 - \frac{\xM^{+, {\rm L}}}{\hat{x}_M^+}\big) =  \xi$ (recall from the definition of $\xM^{+, {\rm L}}$ that $W_{\xM^{+, {\rm L}}} = M^+ - \delta'\sqrt{2\pi}$).
	    \item \eqref{RedefineB2} follows by \eqref{TildeB2Ldefinition}. Note also that the condition $\min_{x \in [1 - \xM^{+,\mathrm{L}}/\hat{x}_M^+, 1]} B_2^{\mathrm{L}}(x) > 0$ is directly inherited from $\min_{x \in [0, 1]} B_2^{\mathrm{L}}(x) > 0$;
	    \item \eqref{RunningMinimumBrownianMeander} follows from \lemref{RunningMinimumBrownianMeanderLemma} (with $u =  \xi$ and $\varepsilon = \delta^\prime \xi$).
	\end{itemize}
	Since $\mathbb{P}[\mathcal{T}_2^{(\mathrm{C})} \mid \xM^+ = \hat{x}_M^+] \ge 1 - \delta^\prime$ for all $\hat{x}_M^+$, we have $\mathbb{P}[\mathcal{T}_2^{(\mathrm{C})}] \ge 1 - \delta^\prime$. Similarly, as $B_2^{\mathrm{L}}$ is equal in distribution to $B_2^{\mathrm{R}}$, $\mathbb{P}[\mathcal{T}_2^{(\mathrm{D})}] \ge 1 - \delta^\prime$.
	
    When all four events and $\mathcal{T}_1$ hold, if $r^+(x)$ is less or equal to $c_6\delta^2\sqrt{\Delta}$, then $r^-(x)$ is greater or equal to $c_6\delta^2\sqrt{\Delta}$, where $c_6\delta^2\sqrt{\Delta}=(\delta^\prime)^2\sqrt{2\Delta}$ for $c_6 = 0.01\sqrt{2}$ (and vice versa). By the union bound and choice of $\delta^\prime = 0.1\delta$, 
\begin{align}
	\mathbb{P}[\mathcal{T}_2] & \ge \mathbb{P}[\mathcal{T}_2^{(\mathrm{A})} \cap \mathcal{T}_2^{(\mathrm{B})} \cap \mathcal{T}_2^{(\mathrm{C})} \cap \mathcal{T}_2^{(\mathrm{D})}]\\
	& \ge 1 - \Delta^\eta - 2(\delta^\prime)^2 - \Delta^\eta - 2(\delta^\prime)^2 - \delta^\prime - \delta^\prime \\
	&\ge 1 - 2\Delta^\eta - \delta.
\end{align}    

{\bf Bounding $\mathbb{P}[\mathcal{T}_3]$.} Recall the definitions of $\eta_{\delta}(\cdot)$ and $\alpha_{\delta}(\cdot)$ in \eqref{eqn:eta_alpha}. Divide the time horizon $[-\Delta, 1+\Delta]$ into a time grid of size $\Delta$, with grid points $\tau_k \triangleq \min(k\Delta, 1+\Delta)$ for all $ -1 \le k \le \lceil 1/\Delta \rceil+1$ and intervals $I_k \triangleq [\tau_k, \tau_{k+1}] $ for all  $-1 \le k \le \lceil 1/\Delta \rceil$. For each interval, we define the running maximum and running minimum as follows:
\begin{align}
	\tilde{W}_{I_k}^{\max} \triangleq \max_{x \in I_k} \tilde{W}_x, \quad \tilde{W}_{I_k}^{\min} \triangleq \min_{x \in I_k} \tilde{W}_x.
\end{align}
We define the event $\mathcal{S} \triangleq \mathcal{S}_1 \cap \mathcal{S}_2 \cap \mathcal{S}_3$, where $\mathcal{S}_i$ for $i = 1,2,3$ are defined as follows:
\begin{align}
	\mathcal{S}_1 &\triangleq \bigcap_{k=-1}^{\lceil 1/\Delta \rceil} \left\{ \tilde{W}_{I_k}^{\max} \le \max\left(\tilde{W}_{\tau_k}, \tilde{W}_{\tau_{k+1}}\right) \!+\! \eta_\Delta\left(\Delta\right) \right\} \label{def:S1}\\
	\mathcal{S}_2 &\triangleq \bigcap_{k=-1}^{\lceil 1/\Delta \rceil} \left\{ \tilde{W}_{I_k}^{\min} \ge \min\left(\tilde{W}_{\tau_k}, \tilde{W}_{\tau_{k+1}} \right) - \eta_\Delta\left(\Delta\right) \right\} \label{def:S2}\\
	\mathcal{S}_3 &\triangleq \bigcap_{k=-1}^{\lceil 1/\Delta \rceil} \left\{ \left\lvert \tilde{W}_{\tau_k} - \tilde{W}_{\tau_{k+1}}\right\rvert \le \alpha_\Delta\left(\Delta\right) \right\}.\label{def:S3}
\end{align}
    We first show that $\mathcal{S}$ implies $\mathcal{T}_3$ (see Definition \ref{def:setT}). Suppose that $\mathcal{S}$ holds and there exists an integer $-1 \le k \le \lceil 1 / \Delta \rceil-2$ such that $x-\Delta$ falls into the $k$-th interval $I_k$ while $x+\Delta$ falls into the $(k+2)$-nd interval $I_{k+2}$. Then, the difference $\lvert \tilde{W}_{x-\Delta} - \tilde{W}_{x+\Delta} \rvert$ can be upper bounded via the triangle inequality:
	\begin{align}
	    \nonumber &\left\lvert \tilde{W}_{x-\Delta} - \tilde{W}_{x+\Delta} \right\rvert \\
	    &= \max\left( \tilde{W}_{x-\Delta} - \tilde{W}_{x+\Delta}, \tilde{W}_{x+\Delta} - \tilde{W}_{x-\Delta} \right) \\
	    \label{DefinitionRunningMaxMin}
	    &\le  \max \left( \tilde{W}_{I_k}^{\max} - \tilde{W}_{I_{k+2}}^{\min}, \tilde{W}_{I_{k+2}}^{\max} - \tilde{W}_{I_k}^{\min} \right) \\
	    \label{EventS1S2Eq}
	    &\le  \max \bigg[ \max\left( \tilde{W}_{\tau_k}, \tilde{W}_{\tau_{k+1}} \right)  \nonumber \\
	    &\qquad- \min\left( \tilde{W}_{\tau_{k+2}}, \tilde{W}_{\tau_{k+3}} \right) + 2\eta_\Delta(\Delta),  \nonumber\\
	    &\hspace{10.5mm}   \max\left( \tilde{W}_{\tau_{k+2}}, \tilde{W}_{\tau_{k+3}} \right)\nonumber\\*
	    &\qquad - \min\left( \tilde{W}_{\tau_k}, \tilde{W}_{\tau_{k+1}} \right) + 2\eta_\Delta(\Delta) \bigg]
	    \\
	    \label{EventS3Eq}
	    &\le \, 3\alpha_\Delta(\Delta) + 2\eta_\Delta(\Delta) \\ &\le  c_4 \sqrt{\Delta \ln (1 / \Delta)}, \label{DeltaLogDelta}
	\end{align}
	where \eqref{DefinitionRunningMaxMin} follows from the definitions of running maximum and minimum, \eqref{EventS1S2Eq} follows from events $\mathcal{S}_1$ and $\mathcal{S}_2$, \eqref{EventS3Eq} follows from event $\mathcal{S}_3$, and \eqref{DeltaLogDelta} uses the definitions of $\eta_{\delta}(\cdot)$ and $\alpha_{\delta}(\cdot)$ in \eqref{eqn:eta_alpha}. 
	
	For any $0 \le x \le 1$, the absolute difference between the two regret functions can be expressed as the difference between the $\tilde{W}$ values corresponding to two $2\Delta$-separated points:
	\begin{align}
		\left\lvert r^+(x) - r^-(x) \right\rvert &= \left\lvert (W_{\xM^+}^+ - W_x^+) - (W_{\xM^-}^- - W_x^-) \right\rvert \\
		\label{Definitionx+*x-*}
		&= \lvert M^+ - M^- - W^+_x + W_x^-\rvert \\
		\label{DefinitionW+W-}
		&= \left\lvert \tilde{W}_{x-\Delta} - \tilde{W}_{x+\Delta} \right\rvert \\
		\label{PreviousResult}
		&\le c_4 \sqrt{\Delta \ln (1 / \Delta)}
	\end{align}
	where \eqref{Definitionx+*x-*} follows from the definitions of $\xM^+$ and $\xM^-$ in \eqref{x+*x-*}, \eqref{DefinitionW+W-} follows from definitions of $W^+$ and $W^-$ in \eqref{DefW+} and \eqref{DefW-}, and \eqref{PreviousResult} follows from \eqref{DeltaLogDelta} (which holds uniformly in $x$). Hence, the absolute difference between the two regret functions is upper bounded by $\mathcal{O}(\sqrt{\Delta \ln (1 / \Delta)})$ everywhere.
    
    It remains to show that $\mathcal{S}$ holds with probability at least $1 - \Delta$. From \lemref{BrownianRunningMaxOOBLemma} in Appendix~\ref{app:results}, each sub-event of $\mathcal{S}_1$ and $\mathcal{S}_2$ holds with probability at least $1 - (\Delta\cdot\Delta)^{5}$. By the union bound over the $\lceil 1 / \Delta \rceil$ many sub-events of $\mathcal{S}_1$, we have $\mathbb{P}[ \mathcal{S}_1 ] \ge 1 - \Delta^9$. Similarly, we have $\mathbb{P}[\mathcal{S}_2] \ge 1 - \Delta^9$.
    
    Similar to the argument for lower bounding the probability of $\mathcal{M}_1$ in the proof of \lemref{EventMLemma}, each sub-event of $\mathcal{S}_3$ holds with probability at least $1 - (\Delta\cdot\Delta)^3$ by \lemref{NormalUpperBoundLemma}. Hence, we can lower bound $\mathbb{P}[\mathcal{S}_3]$ by $1 - \Delta^5$, and by the union bound on $\mathcal{S}_1$, $\mathcal{S}_2$ and $\mathcal{S}_3$, we have 
\begin{equation}
    \mathbb{P}[\mathcal{T}_3] \ge \mathbb{P}[\mathcal{S}] \ge 1 - \Delta
    \end{equation}    
    as event $\mathcal{S}$ implies $\mathcal{T}_3$. 
    
{\bf Lower bounding $\mathbb{P}[\mathcal{T}]$.}  Recall from Section~\ref{EventT_Proof} that $\mathbb{P} [ \mathcal{T}_1 ]\ge 1-\delta_1 = 1- \Delta^\eta$. Combining this with the lower bounds on $\mathbb{P}[\mathcal{T}_2]$ and $\mathbb{P}[\mathcal{T}_3]$,  and using the union bound on the events $\mathcal{T}_1$, $\mathcal{T}_2$ and $\mathcal{T}_3$, we have 
\begin{equation}
  \mathbb{P}[\mathcal{T}] \ge 1 - 3\Delta^\eta - \delta - \Delta
  \end{equation}  
 as desired.


\subsection{Proof of Lemma \ref{FanoBasedLemma} (Regret Bound Using Fano's Inequality)}\label{app:FanoProof}

By Markov's inequality, we have
\begin{align}
	\mathbb{E}_{\tilde{w}}[\rT] \ge c_3 \delta^2 \sqrt{\Delta} \cdot \mathbb{P}_{\tilde{w}}\left[ \rT \ge c_3 \delta^2 \sqrt{\Delta} \right].
\end{align}
By Definition \ref{def:setT}, conditioned on $\tilde{W} = \tilde{w}$ satisfying $\mathcal{T}$, $r_T$ is smaller than $c_3 \delta^2 \sqrt{\Delta}$ for at most one of the functions $W^+$ and $W^-$. Hence, if we let $\hat{V}$ denote the index in $\{+,-\}$ corresponding to the smaller regret, we find that if the regret associated with $v$ is smaller than $c_3 \delta^2 \sqrt{\Delta}$, we must have $\hat{V} = v$. Therefore, 
\begin{align}
\mathbb{P}_{\tilde{w}}^v\left[ \rT \ge c_3 \delta^2 \sqrt{\Delta} \right] \ge \mathbb{P}_{\tilde{w}}^v\left[ \hat{V} = v \right],
\end{align}
where the superscript indicates conditioning on $V = v$. Hence, we can lower bound the above probability as follows, using the fact that $V$ is equiprobable on $\{+,-\}$:
\begin{align}
	\mathbb{P}_{\tilde{w}}\left[ \rT \ge c_3 \delta^2 \sqrt{\Delta} \right] 
	&\hspace{0mm} = \frac{1}{2} \sum_{v \in \{+,-\}} \mathbb{P}_{\tilde{w}}^v\left[ \rT \!\ge\! c_3 \delta^2 \sqrt{\Delta} \right] \\
	\label{InfoArgument}
	&\hspace{0mm}\ge \frac{1}{2} \sum_{v \in \{+,-\}} \mathbb{P}_{\tilde{w}}^v\big[ \hat{V} = v \big] \\
	\label{Fano}
	&\hspace{0mm}\ge H_2^{-1}\left( \log 2 - I_{\tilde{w}}(V; \mathbf{x}, \mathbf{y}) \right),
\end{align}
where \eqref{InfoArgument} 
follows from events $\mathcal{T}_1$ and $\mathcal{T}_2$ in Definition \ref{def:setT}, and \eqref{Fano} follows from the binary version of Fano's inequality, e.g., as stated in \cite[Remark~1]{scarlett2019fano}. 

\medskip
{\bf Supplementary material (appendix).} The supplementary material is uploaded as a separate document, containing Appendix \ref{app:results} stating known properties of Brownian motion.

\medskip
{\bf Acknowledgment.} The authors gratefully acknowledge Prof.~Rongfeng Sun for helpful discussions regarding properties of Brownian motion. 

\bibliographystyle{IEEEtran}
\bibliography{reference}

\begin{IEEEbiography}[{\includegraphics[width=1in,height=1.25in,clip,keepaspectratio]{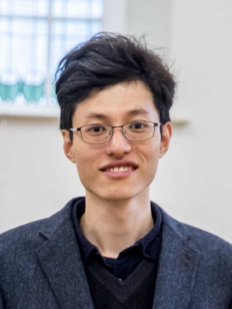}}]{Zexin Wang}
received the Bachelor of Science degree in Quantitative Finance from National University of Singapore, Singapore, in 2019. He is currently a Ph.D.\ student with the Mathematical Finance section at Department of Mathematics, Imperial College London, London, U.K. From May 2019 to July 2019, he was a research assistant with the Department of Computer Science, National University of Singapore. His research interests are in the area of mathematical finance, mainly market microstructure, derivative pricing, limit order zoning and liquidation costs.
\end{IEEEbiography}

\begin{IEEEbiography}[{\includegraphics[width=1in,height=1.25in,clip,keepaspectratio]{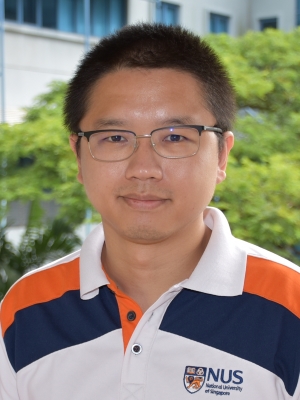}}]{Vincent Y.~F.~Tan}
    (S'07--M'11--SM'15)    received the B.A.\
and M.Eng.\ degrees in electrical and information
sciences from Cambridge University, Cambridge,
U.K., in 2005, and the Ph.D.\ degree in electrical
engineering and computer science (EECS) from the
Massachusetts Institute of Technology (MIT), Cambridge,
MA, USA, in 2011. He is currently an Associate Professor with the Department of
Electrical and Computer Engineering and the Department
of Mathematics, National University of Singapore.
His research interests include information theory, machine learning, and
statistical signal processing. He is currently an Associate Editor for the IEEE
TRANSACTIONS ON SIGNAL PROCESSING and an Associate Editor of machine
learning for the IEEE TRANSACTIONS ON INFORMATION THEORY.
\end{IEEEbiography}

\begin{IEEEbiography}[{\includegraphics[width=1in,height=1.25in,clip,keepaspectratio]{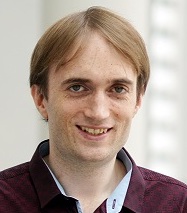}}]{Jonathan Scarlett}
    (S'14--M'15) received 
    the B.Eng.~degree in electrical engineering and the B.Sci.~degree in 
    computer science from the University of Melbourne, Australia. 
    From October 2011 to August 2014, he
    was a Ph.D. student in the Signal Processing and Communications Group
    at the University of Cambridge, United Kingdom. From September 2014 to
    September 2017, he was post-doctoral researcher with the Laboratory for
    Information and Inference Systems at the \'Ecole Polytechnique F\'ed\'erale
    de Lausanne, Switzerland. Since January 2018, he has been an assistant
    professor in the Department of Computer Science and Department of Mathematics,
    National University of Singapore. His research interests are in
    the areas of information theory, machine learning, signal processing, and
    high-dimensional statistics. He received the Singapore National Research Foundation (NRF)
    fellowship, and the NUS Presidential Young Professorship.
\end{IEEEbiography}

\newpage
\onecolumn

\appendices 
\begin{center}
{\LARGE \bf Supplementary Material}

\medskip
{\LARGE Tight Regret Bounds for Noisy Optimization of a Brownian Motion}

\medskip
{\large Zexin Wang, Vincent Tan, and Jonathan Scarlett}

\end{center}
\setcounter{equation}{0}
\counterwithin*{equation}{section}
\renewcommand\theequation{\thesection.\arabic{equation}}
\setcounter{lemma}{0}
\counterwithin*{lemma}{section}
\renewcommand\thelemma{\thesection.\arabic{lemma}}

\section{Results Concerning Brownian Motion} \label{app:results}

\begin{lemma}[{Karatzas and Shreve~\cite[Problem 8.2]{Karatzas88}}] \label{RunningMaxDistributionLemma}
Let $W_x$ be a Brownian motion and $M_x = \max_{0 \le x^\prime \le x} W_{x^\prime}$. The probability density function of $M_x$ satisfies
\begin{align}
\mathbb{P}[M_x \in \mathrm{d}b] = \mathbb{P}[\lvert W_x \rvert \in \mathrm{d}b].
\end{align}
\end{lemma}

\begin{lemma} [{Grill {\em et al.}~\cite[Proof of Lemma 1]{grill2018optimistic}}] \label{BrownianBridgeLemma}
Given the values of a Brownian motion at the ends of an interval $[a,b]$ being $w_a$ and $w_b$, the cumulative distribution function of the running maximum of the Brownian motion in the interval $[a,b]$ is, for all  $y \ge \max(w_a, w_b)$,
\begin{align}
\mathbb{P}&\left[ \max_{x \in [a,b]} W_x > y \,\Big|\, W_a = w_a, W_b = w_b \right] = \exp\left[ - \frac{2(y-w_a)(y-w_b)}{b-a} \right].
\end{align}
\end{lemma}


\begin{lemma}[{Kallenberg~\cite[Equation (3.1)]{kallenberg2006foundations}}] \label{BrownianMeanderDistributionFunctionLemma}
Let $B$ be a Brownian motion with initial value $u$,\footnote{In \cite{kallenberg2006foundations} a possible drift is included, but for our purposes a drift of zero suffices.} the joint distribution function of it and its running minimum for any $y > v, u > v, s > 0$ is as follows:
\begin{align}
&\nonumber \mathbb{P}\left[ B(s) \in \mathrm{d}y, \min_{0 \le z \le s} B(z) > v \,\Big|\, B(0) = u \right] \\
&= \left[ \exp\left( - \frac{(y-u)^2}{2s} \right) - \exp\left( - \frac{(2v-y-u)^2}{2s} \right) \right] \frac{\mathrm{d}y}{\sqrt{2\pi s}}.
\end{align}
\end{lemma}

\begin{lemma}[{Denisov~\cite[Theorem 1]{denisov1983random}}]\label{BrownianMeanderIdenticalDistributionLemma}
Let $W$ be a standard Brownian motion with maximizer $\xM$ and maximum value $M = \max_{x \in D} W_x$.  If we define $B_2^{\mathrm{L}}$ and $B_2^{\mathrm{R}}$ as
\begin{equation}
 B_2^{\mathrm{L}}(x) \triangleq \frac{M - W_{\xM - x \cdot \xM}}{\sqrt{\xM}},\quad\mbox{and}\quad B_2^{\mathrm{R}}(x) \triangleq \frac{M - W_{\xM + x \cdot (1 - \xM)}}{\sqrt{1 - \xM}},
\end{equation}
then $B_2^{\mathrm{L}}$ and $B_2^{\mathrm{R}}$ are independent and identically distributed Brownian meanders (i.e., BM conditioned on being positive), and they are also independent of $\xM$.
\end{lemma}

\begin{lemma}[{Grill {\em et al.}~\cite[Proof of Lemma 1]{grill2018optimistic}}]\label{BrownianRunningMaxOOBLemma}
We have the following high-probability upper bound on the running maximum of a Brownian motion in the interval $[a,b]$:
\begin{align}
\mathbb{P}\left[ \sup_{x \in [a,b]} W_x > \max(W_a, W_b) + \eta_\delta(b-a) \right] \le [\delta(b-a)]^5,
\end{align}
with $\eta_{\delta}(\cdot)$ defined in \eqref{eqn:eta_alpha}.
\end{lemma}

\end{document}